\documentclass{article}

\usepackage{etoolbox}

\newbool{preprint}
\newbool{final}

\setbool{final}{true}

\ifbool{preprint}{
    \usepackage[preprint, nonatbib]{neurips_2026}
}{
    \ifbool{final}{
        \usepackage[main, final, nonatbib]{neurips_2026}
    }{
        \usepackage[nonatbib]{neurips_2026}
    }
}

\usepackage[utf8]{inputenc} %
\usepackage[T1]{fontenc}    %

\usepackage[colorlinks, citecolor=blue, linkcolor=red, urlcolor=black, pagebackref=true]{hyperref}
\renewcommand*\backref[1]{\ifx#1\relax \else (Cited on p. #1) \fi}

\usepackage{url}            %
\usepackage{booktabs}       %
\usepackage{amsfonts}       %
\usepackage{nicefrac}       %
\usepackage{microtype}      %
\usepackage{xcolor}         %
\usepackage{tcolorbox}
\usepackage{float}
\usepackage{subfig}

\usepackage{commands}

\usepackage[numbers, sort]{natbib}
\usepackage{amsmath,amsfonts,amssymb,mathtools}
\usepackage{amsthm}
\usepackage{cleveref}
\usepackage{bbold}
\usepackage{bm}
\usepackage{empheq}
\usepackage{enumitem}
\usepackage{pifont}
\usepackage{cases}
\usepackage{titletoc}
\usepackage{wrapfig}

\newtheorem{definition}{Definition} 
\newtheorem{theorem}{Theorem}
\newtheorem{proposition}{Proposition}

\newtheorem{example}{Example}
\newtheorem{remark}{Remark}
\newtheorem{lemma}{Lemma}
\newtheorem{corollary}{Corollary}
\newtheorem{assumption}{Assumption}

\DeclareMathOperator{\argmin}{argmin}

\title{Difference of Convex Programming in the Wasserstein Space with Applications to MMD Optimization}

\author{%
  Clément Bonet \\
  CMAP, CNRS, Ecole Polytechnique, IP Paris \\
  \texttt{clement.bonet.mapp@polytechnique.edu} \\
  \And
  Pierre-Cyril Aubin-Frankowski \\
  CERMICS, CNRS, ENPC, IP Paris \\
  \texttt{pierre-cyril.aubin@enpc.fr} \\
  \And
  Youssef Mroueh \\
  IBM Research \\
  \texttt{mroueh@us.ibm.com}
}

\begin{document}

\maketitle

\begin{abstract}
Optimizing functionals over the space of probability measures is now ubiquitous in machine learning. A widely used approach is to perform the optimization directly over the Wasserstein space, but many objective functionals of practical interest are non-convex along Wasserstein geodesics, making the analysis of standard first-order methods challenging. In this work, we study a class of objectives over the Wasserstein space that admit a difference-of-convex (DC) decomposition  and we lift the classical convex-concave procedure (CCCP) to this setting. Under smoothness and strong convexity assumptions on the convex components of the decomposition, we prove almost stationarity along the iterates of the resulting algorithm. Our main focus is on the Maximum Mean Discrepancy (MMD) and the Energy Distance (ED) functionals, for which we develop explicit Wasserstein DC decompositions, and establish local convergence of the scheme under mild assumptions. Empirically, we show that well-chosen DC decompositions yield faster and more stable convergence than Wasserstein gradient descent on these MMD objectives.

\end{abstract}

\section{Introduction}\label{sec:introduction}

Optimizing over the space of probability measures is an important problem in machine learning, which has received attention to solve problems ranging from variational inference \citep{blei2017variational, lambert2022variational, petit2025variational} to generative modeling \citep{deng2026generative, turan2026generative, cao2026gradient}, reinforcement learning \citep{zhang2018policy, pfau2025wasserstein}, optimization of neural networks \citep{mei2018mean, chizat2018global} or for modeling dynamics of cells \citep{bunne2022proximal, terpin2024learning, persiianov2026learning}. 
One prominent way to optimize over the space of probability measures is to equip it with the Wasserstein distance \citep{villani2009optimal, santambrogio2015optimal}, and to discretize the associated Wasserstein gradient flows \citep{jordan1998variational, wibisono2018sampling}. This allows to design many optimization algorithms as counterparts of their Euclidean version, such as gradient descent \citep{wibisono2018sampling}, proximal point and gradient algorithms \citep{jordan1998variational, salim2020wasserstein}, coordinate descent \citep{xu2026random} or mirror descent \citep{bonet2024mirror, sharrock2023learning}.

While these optimization methods have demonstrated good results in minimizing several objectives such as the Kullback-Leibler divergence (KL) \citep{wibisono2018sampling}, $f$-divergences \citep{gao2019deep,ansari2021refining, liu2024minimzing}, the Energy Distance \citep{hertrich2024generative} or the Sliced-Wasserstein distance \citep{liutkus2019sliced, duc2023nonparametric, bonet2025sliced}, most of them are tailored for convex functionals along (generalized) geodesics. However, many popular functionals are known to be non-convex. This is the case for instance of the squared Wasserstein distance itself \citep[Chapter 9]{ambrosio2008gradient}, the Sliced-Wasserstein distance \citep{bonnotte2013unidimensional}, the KL with non-log-concave target \citep{luu2024non}, the Energy distance \citep{hertrich2024wasserstein} or the Maximum Mean Discrepancy (MMD) \citep{arbel2019maximum}. An alternative to the convexity assumption is to consider Polyak-{\L}ojaziewicz inequalities \citep{blanchet2018family, liu2023polyak, zhu2025convergence}, however these are so far only known to hold in restrictive cases such as trivially for the squared Wasserstein distance, the KL \emph{w.r.t.}\ a measure satisfying the log-Sobolev inequality \cite[Chapter 21]{villani2009optimal}, the Sliced-Wasserstein distance over Gaussians \citep{thurin2026convergence}, the MMD with Coulomb kernels and smooth initializations \citep[Section 2]{boufadene2025global}. %

For the MMD specifically, to address its non-convexity in the Wasserstein space, \citet{arbel2019maximum} proposed injecting noise to the gradient. However, tuning the amount of noise remains delicate. \citet{gladin2024interaction} performed instead the optimization of MMD in the Wasserstein-MMD space, where MMD is convex. Their method improves performance, but it requires to change weights of the distribution, mixing two implicit steps in squared Wasserstein and MMD. More recently, \citet{belhadji2026weighted} tackled the same problem in the Wasserstein Fisher-Rao space through a fixed-point algorithm, also changing weights. We propose instead a new particle-based algorithm on the Wasserstein space to handle non-convex functionals.

A natural candidate for such an extension is the family of difference-of-convex (DC) methods on $\R^d$ \citep{le2018dc, pham2014recent}, designed for objectives that can be written as the difference of two convex functions. This class is large: on $\R^d$ it contains in particular all twice continuously differentiable functions \citep{yuille2001concave, HiriartUrruty1985}. Its most prominent algorithm is the Convex-Concave Procedure (CCCP) \citep{yuille2001concave}. DC methods have found many applications in machine learning, among which kernel selection \citep{argyriou2006dc}, clustering \citep{tao2014new}, dictionary learning \citep{vo2015dc}, optimal transport \citep{tran2021factored} and neural network training \citep{awasthi2024dc, askarizadeh2024convex}. However, except for a few works that studied DC algorithms on Riemannian manifolds \citep{souza2015proximal, weber2023global, bergmann2024difference, ferreira2026subdifferential}, and recently on the Wasserstein space \citep{luu2024non,luu2026dc}, optimization algorithms for DC functionals have mostly been studied on Euclidean spaces.

\paragraph{Contributions.} \looseness=-1 In this work, we focus on developing an optimization scheme on the Wasserstein space tailored for the non-convex objectives that can be written as a difference of convex objectives. To do so, we introduce the \emph{Wasserstein Convex-Concave Procedure} (WCCCP), and analyze its theoretical convergence, proving almost stationarity along its iterates. %
The closest work to ours is \citet{luu2024non} which used a Wasserstein Proximal Gradient algorithm \citep{salim2020wasserstein} to solve some DC problems, but which restricted their applications to functionals whose concave part is a potential energy, \emph{i.e.\ }linear. We argue that it is too restrictive to handle objectives such as MMD, which can be decomposed as a sum of two non-convex quadratic and linear terms. Thus, we deal with the more general case, where the concave part can be any differentiable functional over the Wasserstein space. 
Then, for several kernels, we show that with a well-chosen decomposition obtained by splitting the kernel, the scheme introduced can better optimize the MMD than Wasserstein Gradient Descent. In particular, we provide experiments on the Energy distance and on the MMD with Gaussian kernel.

\paragraph{Notation.} We denote by $\cPr$ the space of probability distributions with finite second moments, and by $\cPa$ its restriction to absolutely continuous measures with respect to the Lebesgue measure. Given $\mu,\nu\in\cPr$, we denote by $\W_2^2(\mu,\nu) = \inf_{\gamma\in\Pi(\mu,\nu)} \ \int \|x-y\|_2^2\ \mathrm{d}\gamma(x,y)$ the squared Wasserstein distance, where $\Pi(\mu,\nu)$ is the set of couplings between $\mu$ and $\nu$, and $\Pi_o(\mu,\nu)$ is its subset of optimal couplings. The metric space $(\cPr, \W_2)$ is called the Wasserstein space. For any $\mu\in\cPr$, we denote by $L^2(\mu)$ the Hilbert space of functions $f:\R^d\to\R^d$ such that $\int \|f(x)\|_2^2\ \mathrm{d}\mu(x) <\infty$ equipped with the norm $\|f\|_{L^2(\mu)}^2 = \int \|f(x)\|_2^2\ \mathrm{d}\mu(x)$ and with inner product $\langle \cdot,\cdot\rangle_{L^2(\mu)}$. Given $\T:\R^d\to\R^d$, $\T_\#\mu\in\cPr$ is the pushforward measure of $\mu$.

\section{Background} \label{sec:background}

\looseness=-1 We begin by recalling a few facts on optimization on the Wasserstein space. More precisely, we recall the notion of Wasserstein gradient, of total convexity on the Wasserstein space, and some classical optimization schemes on $(\cPr, \W_2)$. For more details about Wasserstein gradient flows, we refer to \emph{e.g.} \citep{ambrosio2008gradient, lanzetti2025first}. Then, we provide a brief introduction to the convex-concave procedure on Euclidean spaces.

\paragraph{Wasserstein gradient.} %

Let $\cF:\cPr\to\R$ be a functional. It admits a Wasserstein gradient $\gW\cF(\mu)\in L^2(\mu)$ at $\mu\in\cPr$ if for all $\nu\in\cPr$, $\gamma\in\Pi_o(\mu,\nu)$, the following first order Taylor expansion is satisfied \citep{bonnet2019pontryagin,lanzetti2025first}
\begin{equation}
    \cF(\nu) = \cF(\mu) + \int \langle \gW\cF(\mu)(x), y-x\rangle\ \mathrm{d}\gamma(x,y) + o\big(\W_2(\mu,\nu)\big).
\end{equation}
\looseness=-1 When it exists, the Wasserstein gradient may not be unique in $L^2(\mu)$ in general. Nonetheless, there is only one gradient living in the tangent space $T_\mu\cPr\subset L^2(\mu)$ which is a closed linear subspace of $L^2(\mu)$. Hence, by Hilbert's decomposition theorem, any gradient can be decomposed as a part in $T_\mu\cPr$ and an orthogonal part $\xi(\mu)$ which satisfies $\int \langle\xi(\mu),y-x\rangle\ \mathrm{d}\gamma(x,y)=0$, see \citep[Proposition 2.11]{lanzetti2025first}. Thus, without loss of generality, we always work with the unique $\gW\cF(\mu)\in T_\mu\cPr$ for a Wasserstein differentiable functional, using the shorthand W-differentiable for such functionals.

Classical functionals from $\cPr$ to $\R$ include potential energies $\cV(\mu)=\int\V\mathrm{d}\mu$ and interaction energies $\cW(\mu) = \frac12 \iint \W(x-y)\ \mathrm{d}\mu(x)\mathrm{d}\mu(y)$ for $\V,\W:\R^d\to\R$, with $\W$ symmetric. They are both differentiable provided $\V$ and $\W$ are differentiable and smooth enough \citep{lanzetti2025first}, and their Wasserstein gradients read respectively as $\gW\cV(\mu) = \nabla \V$ and as the convolution $\gW\cW(\mu) = \nabla \W * \mu$. In this work, we will mostly focus on functionals obtained as a sum of interaction and potential energies, as they include in particular the MMD \citep{arbel2019maximum}.

\paragraph{Convexity in the Wasserstein space.} 

Let $\mu\in\cPr$, $\T,\sS\in L^2(\mu)$. Given $\phi:L^2(\mu)\to \R$ convex and Gâteaux differentiable, the Bregman divergence on $L^2(\mu)$ between $\T$, $\sS$ is defined as \citep{frigyik2008functional}
\begin{equation} \label{eq:bregman_div_L2}
    \D_\phi(\T,\sS) = \phi(\T)-\phi(\sS)-\langle \nabla \phi(\sS), \T-\sS\rangle_{L^2(\mu)}.
\end{equation}
Given $\cF:\cPr\to\R$ W-differentiable, we can define the Bregman divergence in $L^2(\mu)$ associated to the lifted functional $\T\mapsto \cF(\T_\#\mu)$ as 
\begin{equation}
    \D_{\cF}^\mu(\T,\sS) = \cF(\T_\#\mu) - \cF(\sS_\#\mu) - \langle \gW\cF(\sS_\#\mu)\circ \sS, \T-\sS\rangle_{L^2(\mu)},
\end{equation}
using the chain rule for the gradient of $\Tilde{\cF}_\mu:\sS\mapsto \cF(\sS_\#\mu)$ which implies $\nabla\Tilde{\cF}_\mu(\sS)=\gW\cF(\sS_\#\mu)\circ\sS$, see \citep[Proposition 1]{bonet2024mirror}. Note that for $\cF(\mu)=\int \tfrac12 \|\cdot\|_2^2\ \mathrm{d}\mu$, this reduces to $\D_\cF^\mu(\T,\sS)=\tfrac12 \|\T-\sS\|_{L^2(\mu)}^2$. Let $\alpha \ge 0$, we say that $\cF$ is $\alpha$-totally convex \citep{cavagnari2023lagrangian, tanaka2023accelerated, parker2024some} if for all $\mu\in\cPr$, $\T,\sS\in L^2(\mu)$, $\D_{\cF}^\mu(\T,\sS)\ge \frac\alpha2 \|\T-\sS\|_{L^2(\mu)}^2$. Equivalently, it satisfies %
\begin{equation}
    \forall t\in [0,1],\ \cF\big(\big((1-t)\T+t\sS\big)_\#\mu\big)\le (1-t) \cF(\T_\#\mu) + t\cF(\sS_\#\mu) - \alpha \frac{t(1-t)}{2} \|\T-\sS\|_{L^2(\mu)}^2.
\end{equation}
If this result holds only for $\sS=\id$ and for $\T$ the gradient of a convex function, this corresponds to the less restrictive notion of strong convexity along geodesics \citep{ambrosio2008gradient}. If $V$ is $\alpha$-strongly convex, then $\cV$ is $\alpha$-totally convex, and if $W$ is convex, then $\cW$ is also totally convex \citep[Chapter 9]{ambrosio2008gradient}. Given $\beta\ge 0$, $\T,\sS\in L^2(\mu)$ and defining the linear interpolation $\T_t = (1-t)\T+t\sS$ for $t\in [0,1]$, $\cF$ is said to be $\beta$-smooth along the curve $t\mapsto (\T_t)_\#\mu$ if for all $t,s\in [0,1]$, $\D_\cF^\mu(\T_t,\T_s)\le \frac{\beta}{2} \|\T_t-\T_s\|_{L^2(\mu)}^2$.

\paragraph{Wasserstein Gradient Descent.}

\looseness=-1 Optimization on $L^2(\mu)$ and on $\cPr$ are very much intertwined in practice, see \emph{e.g.} \citep{bonet2024mirror, dumont2026learning}. For instance, the Wasserstein Gradient Descent (WGD) over a W-differentiable functional $\cF:\cPr\to\R$ which is defined for all $k\ge 0$, $\tau>0$, as $\mu_{k+1}=\big(\id-\tau\gW\cF(\mu_k)\big)_\#\mu_k$, can be written at each iteration in two steps: first solving an optimization problem on $L^2(\mu_k)$ to get a map $\T_{k+1}\in L^2(\mu_k)$, then pushing forward $\mu_k$ by $\T_{k+1}$, \emph{i.e.}
\begin{equation} \label{eq:wgd}
    \begin{cases}
        \T_{k+1} = \argmin_{\T\in L^2(\mu_k)}\ \tfrac{1}{2\tau} \|\T-\id\|_{L^2(\mu_k)}^2 + \langle \gW\cF(\mu_k), \T-\id\rangle_{L^2(\mu_k)} \\
        \mu_{k+1} = (\T_{k+1})_\#\mu_k.
    \end{cases}
\end{equation}
This scheme is known to converge if $\cF$ is smooth along $t\mapsto \big((1-t)\id + t\T_{k+1}\big)_\#\mu_k$ and convex along geodesics. Other first-order algorithms have been lifted from $\R^d$ to $\cPr$. For instance, replacing the squared $L^2$ distance in \eqref{eq:wgd} by a Bregman divergence \eqref{eq:bregman_div_L2} allows to lift the Mirror descent algorithm \citep{beck2003mirror, lu2018relatively} to $\cPr$ \citep{bonet2024mirror}.

\paragraph{The Convex-Concave Procedure.}

A function $\f:\R^d\to\R$ is DC if it can be written as the difference of two convex functions, \emph{i.e.} if there exists $\fp, \fm$ two convex functions such that $\f=\fp-\fm$. Every \(C^{1}\) function with Lipschitz gradient is DC~\cite[Section II]{HiriartUrruty1985}. To minimize such functions, a popular algorithm is the Convex-Concave Procedure (CCCP) \citep{yuille2001concave}. This algorithm amounts to linearizing the concave part around the current iterate $x_k\in\R^d$, which by convexity gives the lower bound,
\begin{equation}
    \forall x\in \R^d,\ \fm(x) \ge \fm(x_k) + \langle \nabla \fm(x_k), x-x_k\rangle,
\end{equation}
which entails an upper bound on $f=\fp-\fm$, and the majorization-minimization %
\begin{equation} \label{eq:cccp_step}
    x_{k+1}  = \argmin_x\ \fp(x) - \fm(x_k) - \langle \nabla \fm(x_k), x-x_k\rangle, \, \forall k\ge 0.
\end{equation}
\looseness=-1 When both $\fp$ and $\fm$ are differentiable, the iterates satisfy $\nabla \fp(x_{k+1}) = \nabla \fm(x_k)$ by the first order conditions. This algorithm belongs to the more general family of Difference-of-Convex algorithms (DCA) \citep{pham2014recent, le2018dc}, and is related to several optimization algorithms including Frank-Wolfe \citep{yurtsever2022cccp}, the Mirror and Bregman proximal descent \citep{faust2023bregman} or the Proximal gradient algorithm \citep{rotaru2025tight}. %

The first convergence analysis of CCCP focused on obtaining asymptotic convergence, showing that it converges towards a stationary point under some assumptions \citep{tao1997convex,lanckriet2009convergence}. Then, several works such as \citep{yurtsever2022cccp, abbaszadehpeivasti2024rate, faust2023bregman} derived non-asymptotic convergence rates. In particular, the algorithm was shown to converge in $O(1/k)$ in terms of the squared norm of the gradient. More recently, \citet{oikonomidis2025forward} provided an analysis of CCCP under a generalized convexity perspective, though in finite dimensions. Linear rates were also derived under Polyak-{\L}ojaziewicz inequalities adapted to DC functions \citep{abbaszadehpeivasti2024rate, faust2023bregman, oikonomidis2025forward, niu2026continuous}. From a complementary perspective, \citet{maskan2026revisiting} provided a convergence analysis by solving each step \eqref{eq:cccp_step} with Frank-Wolfe and analyzed the convergence for different DC decompositions in that setting.

\section{Wasserstein CCCP} \label{sec:wcccp}

In this section, we first introduce the Wasserstein Convex-Concave Procedure (WCCCP) to minimize difference-of-convex functions on the Wasserstein space as well as our assumptions. Then we provide a theoretical analysis in several settings, including the convex and non-convex ones. Finally, we discuss how we can implement these schemes in practice. All the proofs are deferred to Appendix \ref{appendix:proofs}.

\subsection{Convex-Concave Procedure in the Wasserstein Space}

We focus on the problem of minimizing $\bF:\cPr\to\R$ where $\bF$ can be decomposed as
\begin{equation}
    \bF(\mu)=\cFp(\mu)-\cFm(\mu),\quad \forall \mu\in\cPr,
\end{equation}
with $\cFp,\cFm:\cPr\to\R$ both totally convex. Moreover, we assume that $\cFm$ is W-differentiable, and will assume $\cFp$ W-differentiable on a case-by-case basis. In contrast to \citep{luu2024non}, we do not restrict $\cFm$ to be a potential energy.

Let $\mu_{}\in\cPr$. Since $\cFm$ is totally convex, for any $\T\in L^2(\mu_{})$, 
\begin{equation}
    \D_{\cFm}^{\mu_{}}(\T,\id) \ge 0 \iff \cFm(\T_\#\mu_{}) \ge \cFm(\mu_{}) + \langle \gW\cFm(\mu_{}), \T-\id\rangle_{L^2(\mu_{})}.
\end{equation}
Hence, we have the following upper bound on $\bF$:
\begin{equation} \label{eq:upper_bound_dc}
    \bF(\T_\#\mu_{}) = \cFp(\T_\#\mu_{}) - \cFm(\T_\#\mu_{}) \le \cFp(\T_\#\mu_{}) - \cFm(\mu_{}) - \langle \gW\cFm(\mu_{}), \T-\id\rangle_{L^2(\mu_{})}.
\end{equation}

\begin{tcolorbox}[colback=white]
    We define the Wasserstein Convex-Concave Procedure (WCCCP) as the majorization-minimization based on the upper bound in \eqref{eq:upper_bound_dc} at each iteration $k\ge 0$, \emph{i.e.} given $\mu_0 \in \cPr$,%
    \begin{empheq}[left=\empheqlbrace]{alignat=2} \label{eq:argmin_WDCA_maps}
        &\T_{k+1} = \argmin_{\T\in L^2(\mu_k)}\ \J(\T)\coloneqq\cFp(\T_\#\mu_k) 
                   - \langle \gW\cFm(\mu_k), \T-\id\rangle_{L^2(\mu_k)} \\
        &\mu_{k+1} = (\T_{k+1})_\#\mu_k. \notag 
    \end{empheq}
\end{tcolorbox}

\looseness=-1 This extends CCCP \citep{yuille2001concave} to the Wasserstein space, CCCP being recovered when $\cFp$ and $\cFm$ are potential energies. Restricting the optimization to measures of the form $\T_\#\mu_k$ for $\T \in L^2(\mu_k)$ is without loss of generality in two important cases: \emph{i)} as done in our experiments, for empirical target and initial distributions with the same number of particles; \emph{ii)} if $\mu_k\in\cPa$, as by Brenier's theorem \citep{brenier1991polar}, there always exists an OT map between $\mu_k$ and any $\nu\in\cPr$. While for greater generality we could instead optimize over couplings with first marginal $\mu_k$, the subproblems on $L^2(\mu_k)$ are more tractable and reflect practical implementations. 

We assume existence and uniqueness in \eqref{eq:argmin_WDCA_maps} for simplicity of exposition. Sufficient conditions are $\T\mapsto\cFp(\T_\#\mu_k) $ strictly convex, coercive and lower semicontinuous over $L^2(\mu_k) $, for all $k$. We now study the theoretical convergence of the WCCCP in several settings.

\subsection{Theoretical Analysis in the Non-Convex Case}\label{sec:theory_non-convex}

We first observe that in general, similarly to the analysis of \citep{faust2023bregman} in the Euclidean case, \eqref{eq:argmin_WDCA_maps} is equivalent to both a Mirror Descent and a Bregman Proximal Descent in the Wasserstein space on the objective $\bF$, with Bregman potential respectively $\cFp$ and $\cFm$, and with step size $\tau=1$. %

\begin{proposition} \label{prop:wdca_equivalent_md}
    Let $\mu_k\in\cPr$ for some $k\ge 0$. \eqref{eq:argmin_WDCA_maps} is equivalent to (Bregman Proximal Descent)
    \begin{equation} \label{eq:WDCA_breg_prox}
        \T_{k+1} = \argmin_{\T\in L^2(\mu_k)}\ \D_{\cFm}^{\mu_k}(\T,\id) + \bF(\T_\#\mu_k),
    \end{equation}
    and, if $\cFp$ is W-differentiable, to (Mirror Descent)
    \begin{equation} \label{eq:WDCA_md}
        \T_{k+1} = \argmin_{\T\in L^2(\mu_k)}\ \D_{\cFp}^{\mu_k}(\T,\id) + \langle \gW\bF(\mu_k), \T-\id\rangle_{L^2(\mu_k)}.
    \end{equation}
\end{proposition}

These equivalences allow us to provide a convergence rate in the case where $\bF$ satisfies convexity assumptions relative to $\cFp$ or $\cFm$, leveraging the convergence analysis of Mirror descent \citep{bonet2024mirror} and of Bregman Proximal Descent. We refer to Appendix \ref{appendix:convex_case} for these results and we focus only on the non-convex case here.

\paragraph{Descent Lemma.}
If both $\cFp,\cFm$ are W-differentiable, then we can take the first order conditions in \eqref{eq:argmin_WDCA_maps} which yield the equivalent update 
\begin{equation} \label{eq:foc_wdca}
    \gW\cFp(\mu_{k+1})\circ \T_{k+1} = \gW\cFm(\mu_k).
\end{equation} 
Leveraging this result, we obtain the following relation between $\cF(\mu_{k+1})$ and $\cF(\mu_k)$ involving Bregman divergences with Bregman potential $\cFp$ and $\cFm$.

\begin{proposition} \label{prop:descent_lemma}%
    We have for all $k\ge 0$
    \begin{equation} \label{eq:general_descent}
        \bF(\mu_{k+1})=\bF(\mu_k)-\D_{\cFm}^{\mu_k}(\T_{k+1}, \id)-\Ck
    \end{equation}
    where $\Ck \coloneqq \bF(\mu_k) - \min_{\T\in L^2(\mu_k)}\ \cFp(\T_\#\mu_k) - \cFm(\mu_k) - \langle \gW\cFm(\mu_k), \T-\id\rangle_{L^2(\mu_k)}$ and
    \begin{equation}
    \begin{aligned}
        \Ck = \cFp(\mu_k) - \cFp(\mu_{k+1}) - \langle \gW\cFm(\mu_k), \id - \T_{k+1}\rangle_{L^2(\mu_k)}.
    \end{aligned}
\end{equation}
    Assume that $\cFp$ is W-differentiable. Then, for all $k\ge 0$, $\Ck=\D_{\cFp}^{\mu_k}(\id, \T_{k+1})$, hence
    \begin{equation}\label{eq:iterates_difference_gap}
        \bF(\mu_{k+1}) = \bF(\mu_k) - \D_{\cFm}^{\mu_k}(\T_{k+1}, \id) - \D_{\cFp}^{\mu_k}(\id, \T_{k+1}).
    \end{equation}
\end{proposition}

The term $\Ck$, used also in the non-smooth and Euclidean case in \citep[eq.(15)]{abbaszadehpeivasti2024rate} or \citep{yurtsever2022cccp}, is merely a proxy for $ \D_{\cFp}^{\mu_k}(\id, \T_{k+1})$ when $\cFp$ is possibly not W-differentiable. Note that $\Ck\ge 0$, as taking $\T=\id$ to upper bound the min, the difference vanishes. Hence, \eqref{eq:general_descent} implies that $\cF$ is non-increasing along the WCCCP scheme.
The condition $\Ck=0$ can thus be used as a termination criterion as it implies that $\mu_k$ is a critical point of $\bF$, as shown in
\begin{proposition} \label{prop:critical_point}
    Assume that $\cFp$ is W-differentiable, then $\Ck=0$ implies that $\gW\bF(\mu_k)=0$.
\end{proposition}

\paragraph{Sublinear rates.} We now derive sublinear rates for the non-convex case. All the next results build on the next immediate result (obtained by telescoping), which provides a bound for any arbitrary sequence and corresponds actually to the more involved \citep[Theorem 5]{faust2023bregman},
    \begin{equation} \label{eq:base_nonconvex_WDCA}
        \min_{k\in \{0,\dots,K-1\}}\ \bF(\mu'_k) - \bF(\mu'_{k+1}) \le \frac{\bF(\mu'_0)-\bF(\mu'_K)}{K}, \quad  \forall (\mu'_k)_{k \ge 0}, \, K\ge 1.
    \end{equation}

In the next proposition, we show a sublinear convergence rate for the termination criterion $\Ck$, in accordance with results over $\R^d$ from \citep{yurtsever2022cccp, abbaszadehpeivasti2024rate, faust2023bregman}.  %

\begin{proposition} \label{prop:cv_stop_criteria}%
    Assume that $\bF$ is bounded from below. Then, for all $K\ge 1$,
    \begin{equation}
        0 \le \min_{k\in \{0,\dots,K-1\}} \Ck \le \frac{\bF(\mu_0) - \inf \bF}{K}.
    \end{equation}
\end{proposition}

Assuming that either $\cFp$ or $\cFm$ is strongly convex along iterates, we can also obtain a sublinear convergence result for the squared distance between iterates and the norm of the W-gradient, which guarantees that at least one iterate of WCCCP is almost a stationary point.

\begin{proposition} \label{prop:cv_convexity}
    Let $\ap,\am\ge 0$ such that $\ap+\am>0$. Assume $\cFp$ and $\cFm$ are respectively $\ap$ and $\am$ totally-convex. %
    Then for all $K\ge 1$,
    \begin{equation}\label{eq:cv_min_w_dist}
        \min_{0\le k\le K-1}\ \W_2^2(\mu_k, \mu_{k+1}) \le \min_{0\le k \le K-1}\ \|\T_{k+1} - \id\|_{L^2(\mu_k)}^2 
        \le \frac{2}{\ap+\am} \frac{\big(\bF(\mu_0)-\bF(\mu_K)\big)}{K}.
    \end{equation}
    Furthermore, if $\cFp$ is W-differentiable and $\gW\cFp$ satisfies the Lipschitz condition:
    \begin{equation} \label{eq:lipschitz_cond}
        \|\gW\cFp(\mu_{k+1})\circ \T_{k+1} - \gW\cFp(\mu_k)\|_{L^2(\mu_k)} \le L \|\T_{k+1} - \id\|_{L^2(\mu_k)} \, \forall k\ge 0,
    \end{equation}
    then
    \begin{equation}\label{eq:cv_min_w_grad}
        \min_{0\le k \le K-1}\ \|\gW\cF(\mu_k)\|_{L^2(\mu_k)}^2 \le \frac{2 L^2 }{\ap+\am} \frac{\big(\bF(\mu_0)-\bF(\mu_K)\big)}{K}.
    \end{equation}
\end{proposition}

\subsection{Computing WCCCP} \label{sec:computing}

To solve the WCCCP scheme \eqref{eq:argmin_WDCA_maps}, we need to minimize at each iteration $\J$ on $L^2(\mu_k)$. First, we observe that if $\cFp$ is convex along curves of the form $t\mapsto \big((1-t)\T+t\sS\big)_\#\mu_k$ for any $\sS,\T\in L^2(\mu_k)$, then $\J$ is convex on $L^2(\mu_k)$ as $\D_{\cFp}^{\mu_k}(\T,\sS)\ge 0$, see \Cref{sec:background}.%

If $\cFp$ is W-differentiable, then $\T\mapsto \cFp(\T_\#\mu_k)$ is Fréchet differentiable and taking the first order conditions in \eqref{eq:argmin_WDCA_maps}, we get the equivalent update \eqref{eq:foc_wdca}. 
If $\cFp(\mu)=\int \V\mathrm{d}\mu$ with $\V$ strictly convex, then $\T_{k+1}=\nabla \V^* \circ \gW\cFm(\mu_k)$ with $\V^*(x)=\sup_y \langle x,y\rangle - \V(y)$ the Legendre transform of $\V$. For more general functionals, the problem is implicit and, to the best of our knowledge, cannot be solved in closed-form. 

Nonetheless, since $\J$ is convex on $L^2(\mu_k)$, under the Lipschitz condition \eqref{eq:lipschitz_cond} (corresponding to $L$-smoothness of $\T\mapsto \cFp(\T_\#\mu_k)$ for all $k\ge 0$), we can perform a gradient descent on $L^2(\mu_k)$, which is of the form, for $0<\tau\le 1/L$ and $\Tilde \T_k^0 = \id$,
\begin{equation}
    \forall \ell\ge 0,\ \Tilde \T_k^{\ell+1} = \Tilde \T_k^\ell - \tau \big(\gW\cFp\big((\Tilde \T_k^\ell)_\#\mu_k\big)\circ \Tilde \T_k^\ell - \gW\cFm(\mu_k)\big).
    \label{eq:inner}
\end{equation}
Then, we approximate the solution of \eqref{eq:argmin_WDCA_maps} by $\Tilde \T_{k+1}=\Tilde \T^M_k$ and $\mu_{k+1} = (\Tilde \T_{k+1})_\#\mu_k$. If $\mu_k^\ell = (\Tilde \T_k^\ell)_\#\mu_k=\frac1n\sum_{i=1}^n \delta_{x_i^{k,\ell}}$, then the update translates to the particles as
\begin{equation} \label{eq:inner_gd}
    \forall k\ge 0, \ell\ge 0,\ x_i^{k,\ell+1} = x_i^{k,\ell} - \tau \big(\gW\cFp(\mu_k^\ell)(x_i^{k,\ell}) - \gW\cFm(\mu_k)(x_i^{k,0})\big).
\end{equation}
In practice, one can follow this procedure for any $\tau$ small enough, and we run the algorithm for $0< k \le K$ outer steps \eqref{eq:argmin_WDCA_maps}, and, at each $k$, $0<\ell \le M$ inner steps of \eqref{eq:inner_gd} to minimize $\J$. Alternatively one could rely on root-finding algorithms instead, such as Newton's method \citep{zolter2020tutorial}.

\subsection{Connection with Wasserstein Proximal Gradient Algorithm}\label{sec:w-prox_algo_luu}

\citet{luu2024non} proposed to solve the DC problem on $\cPr$ by using a Wasserstein Proximal Gradient Descent scheme \citep{salim2020wasserstein}, alternating between a gradient descent step on $-\cFm$ and a proximal (JKO) step on $\cFp$, \emph{i.e.} using
\begin{equation} \label{eq:wasserstein_proximal_dc}
    \begin{cases}
        \nu_{k+1} = \big(\id + \tau \gW\cFm(\mu_k)\big)_\#\mu_k \\
        \mu_{k+1} = \argmin_{\mu\in\cPr}\ \frac{1}{2\tau}\W_2^2(\mu, \nu_{k+1}) + \cFp(\mu).
    \end{cases}
\end{equation}
In \Cref{prop:scheme_luu}, we show that this is equivalent to minimizing an upper bound similar to that of WCCCP \eqref{eq:upper_bound_dc}, with an additional quadratic cost.
\begin{proposition} \label{prop:scheme_luu}
    Assume $\mu_k, \nu_k\in\cPa$ for all $k\ge 0$. Then \eqref{eq:wasserstein_proximal_dc} is equivalent to
    \begin{equation} \label{eq:wasserstein_prox_dc_equivalent}
        \begin{cases}
            \Tilde{\T}_{k+1} = \underset{\T\in L^2(\mu_k)}{\argmin}\ \cFp(\T_\#\mu_k) - \langle \gW\cFm(\mu_k), \T-\id\rangle_{L^2(\mu_k)} + \frac{1}{2\tau}\|\T-\id\|_{L^2(\mu_k)}^2 \\
            \mu_{k+1} = (\Tilde{\T}_{k+1})_\#\mu_k.
        \end{cases}
    \end{equation}
\end{proposition}

Hence, this scheme can be seen as a lifting of CCCP algorithms with additional quadratic terms \citep{sun2003proximal} which have been in particular used to deal with the stochastic setting \citep{nitanda2017stochastic, xu2019stochastic, chayti2025stochastic}. Note that it would also be possible to use instead a Bregman divergence on $L^2(\mu_k)$, and that it would be equivalent with a Bregman proximal gradient scheme, see \emph{e.g.} \citep[Appendix F]{bonet2024mirror}. 
Furthermore, the WCCCP algorithm can be seen as minimizing in \eqref{eq:WDCA_breg_prox} a regularization $\cF(\T_\#\mu_k)+\D_{\cFm}^{\mu_k}(\T,\id)$ of $\cF(\T_\#\mu_k)$, and \eqref{eq:wasserstein_prox_dc_equivalent} hence corresponds to replacing $\cF(\T_\#\mu_k)$ by $\cF(\T_\#\mu_k)+\frac{1}{2\tau}\|\T-\id\|_{L^2(\mu_k)}^2$ in our analysis. %

\looseness=-1 \citet{luu2024non} restrict their attention to functionals whose concave part is a potential energy, $\cFm(\mu)=\int \Vm\,\mathrm{d}\mu$ with $\Vm:\R^d\to\R$ convex. This already covers a large class of objectives, including the KL divergence with respect to a possibly non-log-concave target whose potential admits a DC decomposition $\V=\Vp-\Vm$ with $\Vp,\Vm$ respectively $\ap$- and $\am$-strongly convex. Our analysis of \Cref{sec:wcccp} also covers these objectives. In this case, $\cFm(\mu)=\int \Vm\,\mathrm{d}\mu$ and $\cFp(\mu)=\cH(\mu)+\int \Vp\,\mathrm{d}\mu+\mathrm{cst}$, where $\cH(\mu)=\int \log\rho\,\mathrm{d}\mu$ denotes the negative entropy of $\mathrm{d}\mu=\rho\,\mathrm{d}\mathrm{Leb}$ (and $\cH(\mu)=+\infty$ otherwise). While $\cFm$ is $\am$-totally convex, $\cFp$ is only $\ap$-convex along generalized geodesics, \emph{i.e.} along curves $t\mapsto\big((1-t)\T+t\sS\big)_\#\mu$ where $\T$ and $\sS$ are gradients of convex functions, hence OT maps \citep[Chapter 9]{ambrosio2008gradient}. This weaker notion is nevertheless sufficient: we show in \Cref{prop:cv_generalized_geodesics} that \eqref{eq:cv_min_w_dist} still holds for this functional.

\begin{proposition} \label{prop:cv_generalized_geodesics}
    Let $\ap \ge 0$ and $\am>0$, $\cG:\cP_2(\R^d)\to \R$, $\Vm:\R^d\to \R$ and for all $\mu\in\cP_2(\R^d)$, $\cFm(\mu) = \int \Vm\ \mathrm{d}\mu$ and $\cFp(\mu)=\cH(\mu) + \cG(\mu)$. Assume $\Vm$ is $\am$-strongly convex, and $\cFp$ is $\ap$-convex along generalized geodesics. Then, we also have almost stationarity of WCCCP as given in \eqref{eq:cv_min_w_dist}.
\end{proposition}

In the next section, we focus instead on the MMD, which can be decomposed as a DC function, whose concave part is a sum of an interaction and of a potential energy, and for which $\cFp$ and $\cFm$ are both totally convex.

\section{DC Decomposition for the Maximum Mean Discrepancy} \label{sec:mmd}

We now focus on finding a DC decomposition to the MMD \citep{gretton2012kernel}. Given $k:\R^d\times\R^d\to \R$ a kernel, the squared MMD is defined as
\begin{equation}
    \forall \mu,\nu\in\cPr,\ \mmd_k^2(\mu,\nu) = \iint k(x,y)\ \mathrm{d}(\mu-\nu)(x)\mathrm{d}(\mu-\nu)(y).
\end{equation}
It is well known that the squared MMD can be decomposed as a sum of an interaction energy and a potential energy \citep{arbel2019maximum}, \emph{i.e.}
\begin{equation} \label{eq:mmd}
    \bF(\mu) = \frac12 \mmd_k^2(\mu,\nu) = \frac12\iint k(x,y)\ \mathrm{d}\mu(x)\mathrm{d}\mu(y) + \int \V\ \mathrm{d}\mu + c(\nu),
\end{equation}
with $\V(x) = - \int k(x,y)\ \mathrm{d}\nu(y)$ and $c(\nu)=\frac12\iint k(x,y)\ \mathrm{d}\nu(x)\mathrm{d}\nu(y)$. The first term in \eqref{eq:mmd} is an interaction term and the second a potential. This objective is in general not (geodesically) convex, but only semi-convex \citep[Proposition 5]{arbel2019maximum}, \emph{i.e.} $\lambda$-totally convex with $\lambda\in\R$. Moreover, the performance of WGD to minimize it depends heavily on the kernel as observed in \emph{e.g.} \citep{arbel2019maximum, korba2021kernel, hertrich2024generative}. %

For an $L$-smooth kernel $k$, \emph{i.e.} satisfying $\|\nabla k(x,y) - \nabla k(x',y')\|_2^2 \le L^2\big(\|x-x'\|_2^2 + \|y-y'\|_2^2\big)$ for all $x,x',y,y'\in \R^d$, \citet[Appendix A.2]{luu2024non} proposed to use, for $\alpha \ge L$,
\begin{equation} \label{eq:decomposition_luu}
    \cFp(\mu) = \alpha \int \|x\|_2^2\ \mathrm{d}\mu(x) + \frac12\iint k(x,y)\ \mathrm{d}\mu(x)\mathrm{d}\mu(y) + c(\nu), \, \cFm(\mu) = \int \big(\alpha\|x\|_2^2 - \V(x)\big)\ \mathrm{d}\mu(x) 
\end{equation}
as a DC decomposition of \eqref{eq:mmd}. Instead, we propose to obtain a Difference-of-Convex function for this objective by decomposing the kernel itself. For this, we focus on translation-invariant kernels.

\paragraph{MMD with translation-invariant kernel.}

A large class of useful kernels are the translation-invariant ones, \emph{i.e.} those of the form $k(x,y)=\psi(x-y)$ for some symmetric function $\psi:\R^d\to \R$ \citep{sriperumbudur2011universality, muandet2017kernel}. This class includes, among others, the Riesz (a.k.a.\ negative distance) kernel with $\psi(z)=-\|z\|_2$, the Gaussian kernel with $\psi(z) = e^{-\|z\|_2^2/(2h)}$ %
or the inverse multiquadric kernel with $\psi(z)=(c^2 + \|z\|_2^2)^{-\alpha}$ and $\alpha >1$ \citep{muandet2017kernel}. Assuming that $\psi$ admits a DC decomposition $\psi=\psip-\psim$, we obtain the following DC decomposition of the MMD.

\newpage

\begin{proposition} \label{prop:DC_radial_MMD}
    Let $k$ be a translation-invariant kernel of the form $k(x,y)=\psi(x-y)$ for all $x,y\in\R^d$, with a $\psi$ admitting a DC decomposition $\psi=\psip-\psim$, $\psip,\psim$ being $\ap,\am$-strongly convex. Then \eqref{eq:mmd} admits the DC decomposition $\bF=\cFp-\cFm$ where for all $\mu\in\cPr$,
    \begin{equation} \label{eq:dc_decomposition_radial}
        \left\{
            \begin{array}{ll}
                 \cFp(\mu) = \tfrac12 \iint \psip(x-y)\ \mathrm{d}\mu(x)\mathrm{d}\mu(y) + \int \Vm\mathrm{d}\mu + c(\nu), & \Vm(\cdot) = \int \psim(\cdot-y)\ \mathrm{d}\nu(y), \\
                 \cFm(\mu) = \tfrac12 \iint \psim(x-y)\ \mathrm{d}\mu(x)\mathrm{d}\mu(y) + \int \Vp\mathrm{d}\mu, & \Vp(\cdot) = \int \psip(\cdot-y)\ \mathrm{d}\nu(y),
            \end{array}
        \right.
    \end{equation}
    and $\cFp,\cFm$ are respectively $\am$ and $\ap$ totally convex. Note that $\psip,\psim$ can always be chosen symmetric as, by symmetry of $\psi$, $2\psi(x)=\psi(x)+\psi(-x)=\psip(x)+\psip(-x)-\psim(x)-\psim(-x)$. They are also locally Lipschitz since they are convex and have full domain.
\end{proposition}

\begin{table}[t]
\caption{DC decompositions satisfying Prop.\ref{prop:cv_mmd_radial_kernels}, i.e., $k(x,y)=q_+(s)-q_{-}(s),$ where $s= \|x-y\|^2_2$. }
\label{tab:summary}
\tiny\centering
\renewcommand{\arraystretch}{1.2}
    \begin{tabular}{l l l l l l l l l}
        \toprule
        kernel & $\qp$ & $\qm$ & $\underline\lambda[\qp]$ & $\underline\lambda[\qm]$ & $\overline\Lambda[\qp]$ & $\overline\Lambda[\qm]$ & $\Omega$ & Ref  \\
        \midrule
        Smooth Riesz
        & $0$
        & $\sqrt{\varepsilon+s}$
        & $0$
        & $\dfrac{\varepsilon}{(\varepsilon+S_*)^{3/2}}$
        & $0$
        & $\varepsilon^{-1/2}$
        & compact
        & \Cref{lem:DCRiesz} \\[0.2em]
        
        Gauss-Jordan
        & $e^{-s/(2h)}+s/2h$
        & $s/2h$
        & $0$
        & $1/h$
        & $\dfrac{1+2e^{-3/2}}{h}$
        & $1/h$
        & $\R^d$
        &  \Cref{prop:DCGaussian_Jordan} \\[0.2em]
        
        Gauss-cosh/sinh
        & $\cosh(s/(2h))$
        & $\sinh(s/(2h))$
        & $0$
        & $1/h$
        & \eqref{eq:def_coeff_Gauss_Jordan+}
        & \eqref{eq:def_coeff_Gauss_Jordan-}
        & compact
        & \Cref{lem:DCGaussian} \\
        \bottomrule
    \end{tabular}
\end{table}

We now focus on the subclass of radial kernels, \emph{i.e.} those for which there exists $q:\R_+\to\R$ such that $\psi(z)=q(\|z\|_2^2)$. For such kernels, DC decompositions can be obtained via a Jordan decomposition%
\begin{equation} \label{eq:dc_decomposition_jordan}
    \begin{cases}
        q_+(x) = q(0) + (q'(0)+A)x + \int_{0}^{x} (x-t) \max\big(0, q''(t)\big) \, \mathrm{d}t, \, \\
        q_-(x) = Ax- \int_{0}^{x} (x-t) \min\big(0, q''(t)\big) \, \mathrm{d}t,
    \end{cases}
\end{equation}
whenever these integrals of the second derivative are computable, and taking $A\ge \max(0,-q'(0))$ to ensure that $q_+$ and $q_-$ are nondecreasing (hence $\psi_+$ and $\psi_-$ are convex). For the Gaussian kernel, $q(t)=e^{-t/(2h)}$ for $t\ge 0$, and \eqref{eq:dc_decomposition_jordan} gives a simple DC decomposition.

Alternatively, if $q$ is analytic, \emph{i.e.}\ $q(t)=\sum_{i\in \mathbb{N}} a_i t^i $ for all $t\ge 0$, then we can set $I^+\coloneqq \{i \, | \, a_i\ge 0\}$ and $I^-\coloneqq \{i \, | \, a_i\le 0\}$, $\qp(t)=\sum_{i\in  I^+} a_i t^i$, $\qm(t)=\sum_{i\in  I^-} |a_i| t^i$. For the Gaussian kernel, this gives a DC decomposition based on $\cosh$ and $\sinh$.
\begin{lemma} \label{lemma:dc_gaussian_jordan}
    Let $q:t\mapsto e^{-\alpha t}$ for $\alpha>0$. Then the DC decomposition obtained by \eqref{eq:dc_decomposition_jordan} is $\qp(t)=e^{-\alpha t} + \alpha t$ and $\qm(t)=\alpha t$ for all $t\ge 0$. Similarly the algebraic DC decomposition is $\qp(t)=\cosh(\alpha t)$ and $\qm(t)=\sinh(\alpha t)$ for all $t\ge 0$.
\end{lemma}

These two choices are explored for the Gaussian kernel (Gauss Jordan vs cosh/sinh) in \Cref{sec:xps}, and their forms and properties are detailed in \Cref{tab:summary}.

Let $\Omega=\R^d$ or a compact convex subset and $S_*=\sup_{x,y\in\Omega}\ \|x-y\|_2$. Define for $\psi(z)=q(\|z\|_2^2)$, 
\begin{equation*}
    \underline\lambda[q]\coloneqq\inf_{0\le s\le S_*}\min\bigl\{2q'(s),\,2q'(s)+4sq''(s)\bigr\}, \quad \overline\Lambda[q] \coloneqq \sup_{0\le s\le S_*}\ \max \bigl\{2q'(s), 2q'(s)+4s q''(s)\bigr\}.
\end{equation*}
Having $\underline\lambda[q_{\pm}]\ge 0$ and finite $\overline \Lambda[q_{\pm}] $, which are related to bounds on the Hessian of $\cF^\pm$, allows for a sufficient condition on a DC decomposition of $q$ and therefore of the MMD. Owing to this decomposition we can  apply \Cref{prop:DC_radial_MMD} and all the  convergence results of \Cref{sec:wcccp}. 
\begin{proposition} \label{prop:cv_mmd_radial_kernels}
    Let $q\in C^2([0, S_*])$. Assume there exists $q_{\pm}\in C^2([0, S_*])$ such that $q=\qp-\qm$. If $\underline \lambda[\qp],\underline\lambda[\qm]\ge 0$, then the translation-invariant kernel $k(x,y)=\psi(x-y)=q(\|x-y\|_2^2)$ satisfies \Cref{prop:DC_radial_MMD} with $\psi_{\pm}(z)=q_{\pm}(\|z\|_2^2)$, $\alpha^{\pm}=\underline\lambda[q_{\pm}]$. Moreover, if $\overline \Lambda [\qp],\overline \Lambda [\qm]< \infty$, for all $k\ge 0$,  the Lipschitz condition in \eqref{eq:lipschitz_cond} holds for $L=\sqrt{2}\cdot\overline\Lambda[\qp] + \overline\Lambda[\qm] <\infty$. Hence, if  $\underline \lambda[\qp]+\underline\lambda[\qm]>0$,
    \begin{enumerate}[leftmargin=2em, itemsep=0pt, topsep=0pt, parsep=0pt]
        \item for $S_{*}=\infty$, \emph{i.e.} $\Omega=\R^d$, WCCCP leads to an almost stationary measure, \emph{i.e.} \eqref{eq:cv_min_w_grad} holds;
        \item for $S_{*}<\infty$, \emph{i.e.} $\Omega$ compact and convex subset of $\R^d$, stationarity \eqref{eq:cv_min_w_grad} of WCCCP provided the additional condition that the iterates $(\mu_k)_k$ remain in $\Omega$.  
    \end{enumerate}   
\end{proposition}

\looseness=-1 We refer to \Cref{tab:summary} for decompositions of the Gaussian and (smoothed) Riesz kernels satisfying \Cref{prop:cv_mmd_radial_kernels}. 
For more discussion about kernels satisfying \Cref{prop:cv_mmd_radial_kernels}, we refer to Appendix \ref{appendix:theory_mmd}.
In Appendix \ref{appendix:cv_critical_pt}, we also provide sufficient conditions under which the WCCCP scheme converges locally towards a critical point of the  MMD in the compact (\Cref{prop:convcompact}) and non-compact cases (\Cref{prop:convnoncompact}).

\section{Numerical Experiments} \label{sec:xps}

\looseness=-1 We now apply the WCCCP algorithm on the Energy Distance and on the MMD with Gaussian kernel. We compare its performance with the Wasserstein Gradient Descent (WGD) and the Wasserstein Proximal Gradient \eqref{eq:wasserstein_proximal_dc} proposed in \citep{luu2024non} and denoted Forward-Backward (FB). For the Gaussian kernel, we compare WCCCP with the three different DC decompositions introduced in \Cref{sec:mmd}, namely the Jordan decomposition, the $\cosh$/$\sinh$ decomposition (\Cref{lemma:dc_gaussian_jordan}) and the one of \citep{luu2024non} recalled in \eqref{eq:decomposition_luu}. %
For more numerical experiments and implementation details, we refer to Appendix \ref{appendix:applications_mmd}\footnote{Code available at \url{https://github.com/clbonet/Wasserstein_Convex_Concave_Procedure}}.

\noindent \textbf{Energy distance.} The Energy distance (ED) \citep{sejdinovic2013equivalence} corresponds to \eqref{eq:mmd} induced by $\psi(z)=-\|z\|_2$, \emph{i.e.} $k(x,y)=-\|x-y\|_2$ and, for $c(\nu)=-\frac12 \iint \|x-y\|_2\ \mathrm{d}\nu(x)\mathrm{d}\nu(y)$,
\begin{equation}
    \frac12\ed(\mu,\nu) = -\frac12\iint \|x-y\|_2\ \mathrm{d}\mu(x)\mathrm{d}\mu(y) + \int \V\ \mathrm{d}\mu + c(\nu),\ \V(\cdot) = \int \|\cdot -y\|_2\ \mathrm{d}\nu(y).
\end{equation}
While not convex, its Wasserstein gradient flow has a good behavior \citep{chizat2026quantitative} and has demonstrated good results for different machine learning applications \citep{altekruger2023neural, hagemann2024posterior,hertrich2024generative, hertrich2024wasserstein,geuter2025ddeqs}. Nonetheless, it can be naturally decomposed as a DC functional using \Cref{prop:DC_radial_MMD} with $\psip=0$ and $\psim(z) = \|z\|_2$. Thus, we propose to apply the WCCCP algorithm to minimize it.%

\looseness=-1 On \Cref{fig:cv_mmd_riesz_shapes_spiral_cat,fig:cv_mmd_riesz_cifar}, we minimize the Energy distance with respect to $\nu_n$, an empirical distribution of $n=500$ samples drawn uniformly respectively over spiral or cat shapes, and from CIFAR10 \citep{krizhevsky2009learning}. In both cases, we start from an empirical distribution of $n$ samples of $\mu_0=\cN(0, I_d)$, and use the stepsize $\tau=1$ for WGD and FB, as WCCCP can be seen as Mirror Descent with $\tau=1$. This allows comparing the three schemes in the same setting, assuming access to an oracle for WCCCP and FB. We add in \Cref{fig:fair_cat} a comparison with different step sizes for WGD and the same computational budget for WCCCP, \emph{i.e.} with the same number of gradient evaluations. Results were averaged over 100 different source and target samples on \Cref{fig:cv_mmd_riesz_shapes_spiral_cat} and 5 on \Cref{fig:cv_mmd_riesz_cifar}.
For both FB and WCCCP, the inner optimization of $\J$ is performed via gradient descent, possibly with momentum. We display the number of outer iterations, focusing on the behavior in objective values on \Cref{fig:cv_mmd_riesz_shapes_spiral_cat} (FB and WCCCP having $M=50$ extra inner iterations). On \Cref{fig:cv_mmd_riesz_cifar} instead, to be fair on the computational time (1h30 for WGD, 1h15 for WCCCP on a Nvidia V100 GPU) for $x\in \R^d$ and $d\approx \text{3K}$, in total 200K iterations were performed for WGD against 40K outer iterations for WCCCP, and we show snapshots every 20K iterations for WGD, and every 4K iterations for WCCCP. Even with this rescaling, the convergence of WCCCP remains much faster.

\begin{figure}[t]
    \centering
    \begin{minipage}{0.53\textwidth}
        \centering
        \includegraphics[width=\linewidth]{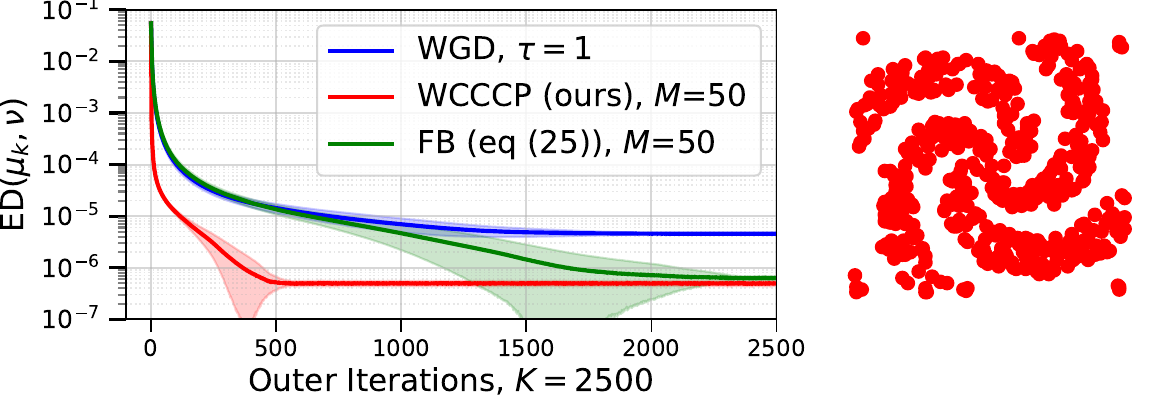} \\ \vspace{0.25em}
        \includegraphics[width=\linewidth]{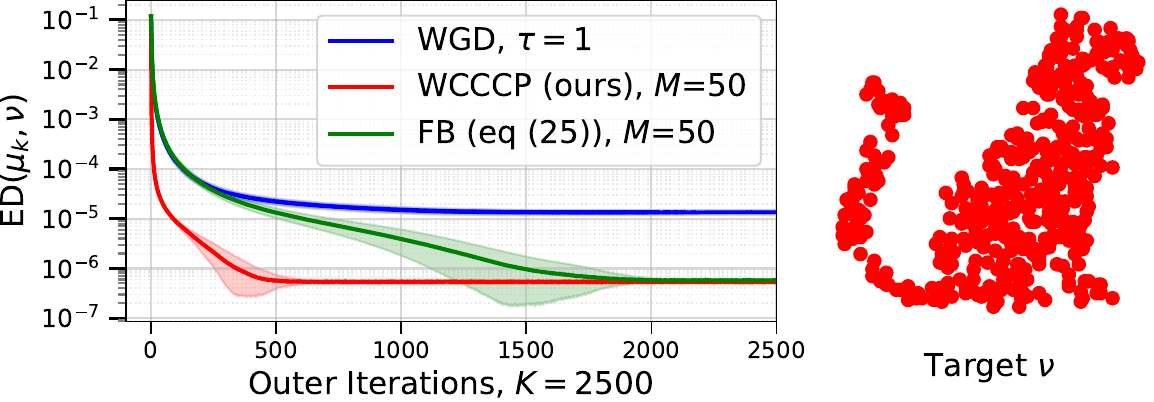}
        \caption{Convergence of WGD, WCCCP and FB to minimize $\cF(\mu)=\tfrac12\ed(\mu, \nu)$ with $\nu$ as uniform distribution over the spiral and cat shapes.} %
        \label{fig:cv_mmd_riesz_shapes_spiral_cat}
    \end{minipage}
    \hfill
    \begin{minipage}{0.44\textwidth}
        \centering
        \includegraphics[width=\linewidth]{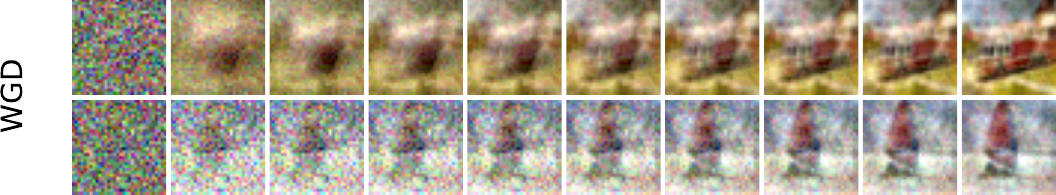} \\ \vspace{0.25em}
        \includegraphics[width=\linewidth]{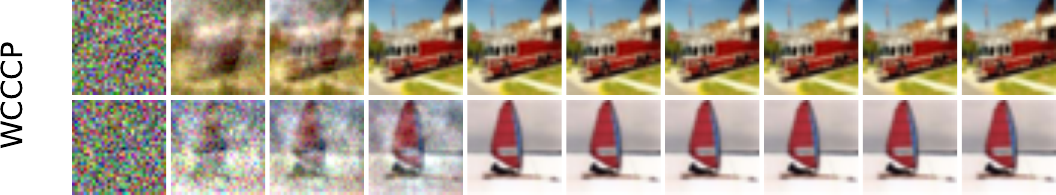} \\ \vspace{0.25em}
        \includegraphics[width=\linewidth]{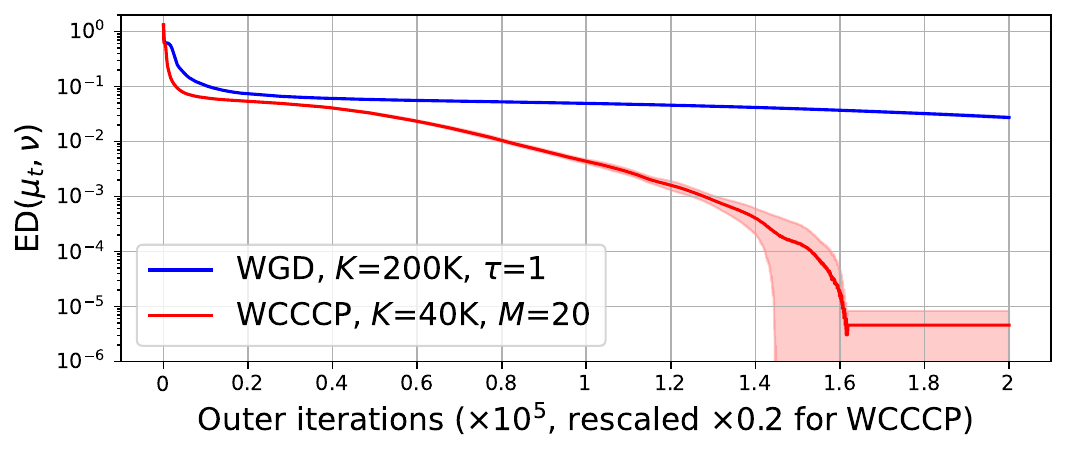}

        \caption{Convergence of WGD and WCCCP to minimize $\cF(\mu)=\tfrac12\ed(\mu,\nu)$ with $\nu$ sampled from CIFAR10.}
        
        \label{fig:cv_mmd_riesz_cifar}
    \end{minipage}
\end{figure}

\noindent \textbf{MMD with Gaussian kernel.} \looseness=-1 
Another natural translation-invariant kernel is the Gaussian kernel $k(x,y) = e^{-\|x-y\|_2^2/(2h)}$ for which $q(s)=e^{-s/(2h)}$. However, minimizing the MMD with this kernel has proven challenging, as its convergence depends a lot on the value of its bandwidth $h$ \citep{arbel2019maximum}.  
On the one hand, we can use the decomposition \eqref{eq:decomposition_luu} proposed by \citet{luu2024non} with choice $\alpha=\tfrac{2}{h}$ (see \Cref{example:decomposition_luu}). 
On the other hand, we can use DC decompositions based on decomposing the kernel as introduced in \Cref{prop:DC_radial_MMD}. We do so by choosing $\qp,\qm$ either with the Jordan decomposition \eqref{eq:dc_decomposition_jordan} or the algebraic one which involves $\cosh$ and $\sinh$ (see \Cref{lemma:dc_gaussian_jordan}).
Their properties are detailed in \Cref{tab:summary} and derived in Appendix \ref{appendix:theory_mmd}.

We compare on \Cref{fig:cccp_mmd_gaussian} the minimization of the squared MMD with Gaussian kernel and bandwidth $h=10$ between WGD and WCCCP with these three decompositions. %
We use $n=500$ particles initially sampled from $\mu_0 = \cN(5\mathbb{1}_2, I_2)$ and set the target as an empirical distribution $\nu_n = \frac1n\sum_{i=1}^n \delta_{y_i}$ with $y_i$ sampled independently from $\nu=\cN(0, \Sigma)$ with $\Sigma = \bigl(\begin{smallmatrix} 1 & 0.5 \\ 0.5 & 1 \end{smallmatrix}\bigr)$.
\looseness=-1 The results are averaged over 25 runs with different source and target samples. The inner problems of the WCCCP scheme with $\cosh/\sinh$ are solved using a gradient descent with momentum, with $m=0.9$ and step size $\tau=5\cdot 10^{-4}$ for $M=250$ iterations, while for the WCCCP scheme with Jordan decomposition, we use a gradient descent with $\tau=0.1$ and $M=500$ iterations. In particular, this decomposition is smoother and less prone to numerical instabilities than the $\cosh/\sinh$ one.
We observe that WGD and WCCCP with the decomposition \eqref{eq:decomposition_luu} get stuck in a local minimum where some samples drift away of the target distribution, while WCCCP with $\cosh/\sinh$ and Jordan decompositions converge much better. We note that the Jordan decomposition can still sometimes be stuck in local minima, which are nonetheless better than WGD. We hypothesize that this is due to the inexact solver used at each iteration of \eqref{eq:argmin_WDCA_maps}. On the other hand, all the samples of WCCCP with $\cosh/\sinh$ appear to eventually converge, instead of being sent away, but at a slower rate. We also compare the results with the FB algorithm in Appendix \ref{appendix:applications_mmd}, which also showcases results that depend heavily on the choice of DC decomposition.

\begin{figure}[t]
    \centering
    \includegraphics[width=\linewidth]{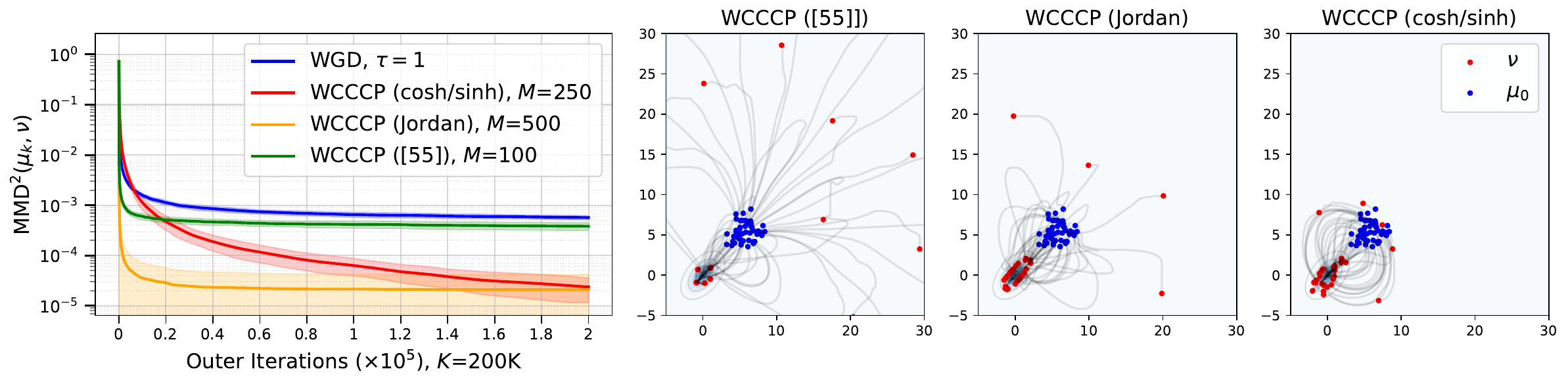}

    \caption{Optimization of $\bF(\mu)=\frac12 \mmd_k^2(\mu,\nu)$ for $\nu$ a Gaussian target and $k$ the Gaussian kernel. (\textbf{Left}) Evolution of the squared MMD along the flow. (\textbf{Right}) Trajectories of the particles over time. The initial particles are in blue and the final particles in red.}
    \label{fig:cccp_mmd_gaussian}
\end{figure}

This experiment exhibits the importance of the choice of the DC decomposition of the target functional, as different DC decompositions induce very different behaviors. We leave for future work providing guidelines on how to choose the DC decomposition to have better convergence guarantees for WCCCP, and include a preliminary experiment in Appendix \ref{xp:ablation_dc_decomposistion}.

\section{Conclusion}

In this work, we lifted the convex-concave procedure to the Wasserstein space, and analyzed its convergence in the convex and non-convex settings. Then, we used it to minimize the Maximum Mean Discrepancy, and showed improved performance over the Wasserstein Gradient Descent for the negative distance and Gaussian kernel. Nonetheless, the improved convergence of WCCCP strongly depends on the choice of the DC decomposition of the kernel. Future work will therefore focus on better understanding the impact of different DC decompositions on the performance, with the goal of designing more effective, automatic, and adaptive decomposition strategies \citep{ahmadi2018dc}. On the theoretical side, taking into account the inexact solvers used for the inner optimization problems would also be a natural avenue for future work.

\newpage

\begin{ack}

We thank the anonymous reviewers for their valuable comments. This work was granted access to the HPC resources of IDRIS under the allocation 2025-AD011015891R1 made by GENCI. CB's work was supported by the Ecole Polytechnique Foundation as part of its campaign ``Servir la Science'', and by the French National Research Agency (ANR) through the France 2030 program under the MacLeOD project (ANR-25-PEIA-0005) and the program "Hi! PARIS", grant number ANR-23-IACL-0005.

\end{ack}

\bibliographystyle{plainnat}
\bibliography{references}

\newpage
\appendix

\part*{Appendix}

The appendix is organized as follows. In Appendix \ref{appendix:limitations}, we discuss some limitations of the paper. In Appendix \ref{appendix:convex_case}, we detail the theoretical analysis of the WCCCP algorithm in the convex case, leveraging the mirror and Bregman proximal descent formulations. In Appendix \ref{sec:PL_for_w_DC}, we discuss the derivations of Polyak-{\L}ojaziewicz inequalities for DC functionals. In Appendix \ref{appendix:theory_mmd}, we provide the full theoretical analysis of DC decompositions of the MMD. In Appendix \ref{appendix:applications_mmd}, we include more details about the numerical experiments. Finally, in Appendix \ref{appendix:proofs}, we state all the proofs.

\startcontents[appendix]

\hypersetup{linkcolor=black}

\printcontents[appendix]{}{0}{\section*{Appendix Contents}}

\hypersetup{linkcolor=red}

\section{Limitations} \label{appendix:limitations}

Our empirical findings show that the improved convergence of WCCCP on MMD strongly depends on the choice of the DC decomposition of the kernel $k$. Our current strategy is restricted to a few simple decompositions of the kernels, without taking into account the geometry of the optimization problem. Hence, future work could focus on developing automatic and adaptive ways to find DC decompositions of the kernels better suited to the problem of minimizing the MMD, \emph{e.g.} based on algebraic decompositions and polynomial DC decompositions minimizing a suitable objective \citep{bomze2004undominated, ahmadi2018dc, niu2024difference}.

This work also focused on DC decompositions of MMD, which can be written as a sum of simple functionals, which are potential and interaction energies. For MMD with the ubiquitous radial kernels, we were thus able to derive DC decompositions based on decompositions of functions on $\R\to\R$. Future work could focus on the problem of deriving DC decompositions of more complex functionals.

On the theoretical part, as highlighted in \Cref{sec:introduction}, there are versions of Polyak-{\L}ojaziewicz (PL) inequalities for both the Wasserstein space \citep{blanchet2018family, liu2023polyak} and DC functions \citep{abbaszadehpeivasti2024rate, faust2023bregman, oikonomidis2025forward, niu2026continuous}. So far our framework does not include them despite preliminary research. We discuss reasons for this in \Cref{sec:PL_for_w_DC}.

We also did not take into account that we actually solve each inner scheme approximately, which leads to a gap between theory and practice.

\section{Theoretical Analysis of WCCCP in the Convex Case} \label{appendix:convex_case}

We focus in this section on the case where $\cF=\cFp-\cFm$ is also convex along a curve of interest which we will detail. For this, we will leverage \Cref{prop:wdca_equivalent_md}, where we showed that \eqref{eq:argmin_WDCA_maps} is equivalent to both a mirror descent and a Bregman proximal descent in the Wasserstein space. Consequently, on the one hand, we will inherit the results from \citep{bonet2024mirror} for mirror descent, and on the other hand we will derive novel convergence results for Bregman proximal descent.

\paragraph{Relative convexity and smoothness.}

First, we need to introduce the notions of relative convexity and smoothness. Let $\alpha,\beta \ge 0$. Following \citep{bonet2024mirror}, we say that $\cF$ is $\alpha$-convex relative to $\cG:\cPr\to\R$ along $t\mapsto \big((1-t)\T+t\sS\big)_\#\mu$ if $\D_\cF^\mu(\T,\sS)\ge \alpha\D_\cG^\mu(\T,\sS)$. Equivalently, we have that $\cF-\alpha\cG$ is convex along this curve, and
\begin{equation}
    \forall t\in [0,1],\ \cF(\mu_t)\le (1-t) \cF(\T_\#\mu) + t\cF(\sS_\#\mu) - \alpha t(1-t) \D_\cG^\mu(\T,\sS).
\end{equation}
Likewise, $\cF$ is $\beta$-smooth relative to $\cG$ along this curve if $\D_\cF^\mu(\T,\sS)\le \beta \D_\cG^\mu(\T,\sS)$. These notions enable lifting the notion of relative convexity and smoothness \citep{lu2018relatively, bauschke2017descent} to $\cPr$ for differentiable functionals.

If the convexity holds for all $\T,\sS\in L^2(\mu)$, $\mu\in\cPr$ and $\cG(\mu) = \int \tfrac12 \|\cdot\|_2^2\mathrm{d}\mu$, then we say that $\cF$ is $\alpha$-totally convex \citep{cavagnari2023lagrangian, tanaka2023accelerated, parker2024some}. Following \citep{ambrosio2008gradient}, if convexity only holds for $\T=\id$ and $\sS$ gradient of convex functions, it coincides with strong convexity along geodesics \citep{ambrosio2008gradient}. Convexity along generalized geodesics corresponds to both $\T$ and $\sS$ being the gradients of convex functions. Examples of totally convex functionals include potential and interaction energies, provided $\V,\W$ are convex,  lower semi-continuous and have a negative part with quadratic growth \citep[Section 9.3]{ambrosio2008gradient}. Note also that all three notions of convexity are equivalent for continuous functionals (for $d\ge 2$) \citep{cavagnari2023lagrangian, parker2024some}.

\paragraph{Bregman Wasserstein distance.}

Similarly to \citep{bonet2024mirror}, let us introduce the Bregman Optimal Transport problem $\W_\phi$ associated with a Wasserstein differentiable functional $\phi:\cPr\to\R$, defined for $\mu,\nu\in\cPr$ as
\begin{equation}
    \W_\phi(\nu,\mu) \coloneqq \inf_{\gamma\in\Pi(\nu,\mu)}\ \phi(\nu) - \phi(\mu) - \int \langle \gW\phi(\mu)(y), x-y\rangle\ \mathrm{d}\gamma(x,y). 
\end{equation}
By \citep[Proposition 15]{bonet2024mirror}, if $\mu\in\cPa$, then there exists $\T\in L^2(\mu)$ such that $\W_\phi(\nu,\mu) = \D_{\phi}^\mu(\T,\id)$, \emph{i.e.} this problem admits an optimal transport map.

\paragraph{Convergence results.}

We now provide a first convergence result for $\bF$ smooth and strongly convex relative to $\cFp$, relying on the mirror descent formulation \eqref{eq:WDCA_md} and the results of \citep{bonet2024mirror}. More precisely, $\bF$ needs to be smooth along iterates and convex along what would be the analog of geodesics on the space $(\cPr, \W_{\cFp})$.

\begin{proposition}
    Let $\bF:\cPr\to\R$ be a functional admitting a DC decomposition $\bF=\cFp-\cFm$ with both $\cFp,\cFm$ Wasserstein differentiable on $\cPr$. Let $\nu\in\cPr$, $0\le \alpha \le \beta \le 1$ and $(\T_k)_{k\ge 1}$, $(\mu_k)_{k\ge 0}$ iterates of \eqref{eq:argmin_WDCA_maps}. Assume that for all $k\ge 0$, $\mu_k\in\cPa$, and denote $\T^{\mu_k,\nu} = \argmin_{\T_\#\mu_k=\nu}\ \D_{\cFp}^{\mu_k}(\T,\id)$, which exists as $\mu_k\in\cPa$. Furthermore, assume that
    $\bF$ is $\beta$-smooth along $t\mapsto \big((1-t)\id + t\T_{k+1}\big)_\#\mu_k$ and $\bF$ is $\alpha$-convex relative to $\cFp$ along $t\mapsto \big((1-t)\id + t\T^{\mu_k,\nu}\big)_\#\mu_k$. Then, for all $k\ge 1$,
    \begin{equation}
        \bF(\mu_k) - \bF(\nu) \le \frac{\alpha}{(1-\alpha)^{-k} - 1} \W_{\cFp}(\nu, \mu_0) \le \frac{1-\alpha}{k}\W_{\cFp}(\nu,\mu_0).
    \end{equation}
    Moreover, if $\alpha>0$, taking $\nu=\mu^*$ the minimizer of $\bF$, we obtain a linear rate, \emph{i.e.} for all $k\ge 0$,
    \begin{equation}
        \W_{\cFp}(\mu^*,\mu_k) \le (1-\alpha)^k \W_{\cFp}(\mu^*, \mu_0).
    \end{equation}
\end{proposition}

\begin{proof}
    The assumptions imply that we can use \citep[Proposition 4]{bonet2024mirror} for the mirror descent formulation \eqref{eq:WDCA_md}.
\end{proof}

Assuming analog conditions as for the convergence of the JKO scheme adapted to the Bregman Proximal Gradient Descent \eqref{eq:WDCA_breg_prox}, we obtain the following linear rate convergence result.

\begin{proposition} \label{theorem:cv_dca_bpp_v3}
Let $\bF:\cPr\to\R$ a functional admitting a DC decomposition $\bF=\cFp-\cFm$ with both $\cFp,\cFm$ Wasserstein differentiable on $\cPr$. Let $\mu^*\in\cPr$ be the minimizer of $\bF$, $\alpha \ge 0$ and $(\T_k)_{k\ge 1}$, $(\mu_k)_{k\ge 0}$ iterates of \eqref{eq:argmin_WDCA_maps}. Assume that for all $k\ge 0$, $\mu_k\in\cPa$, and denote $\T_k^* = \argmin_{\T_\#\mu_k=\mu^*}\ \D_{\cFm}^{\mu_k}(\T,\id)$, which exists as $\mu_k\in\cPa$. Furthermore, assume that $\bF$ is $\alpha$-convex relative to $\cFm$ along $t\mapsto \big((1-t)\T_k^*+t\T_{k+1}\big)_\#\mu_k$.     Then, for all $k\ge 0$,
    \begin{equation}
        \W_{\cFm}(\mu^*, \mu_{k}) \le \left(\frac{1}{1+\alpha}\right)^{k} \W_{\cFm}(\mu^*, \mu_0), \quad \text{and} \quad   \bF(\mu_{k+1}) - \bF(\mu^*) \le \left(\frac{1}{1+\alpha}\right)^k \W_{\cFm}(\mu^*,\mu_0).
    \end{equation}

\end{proposition}

\begin{proof}
    See \Cref{proof:theorem_cv_dca_bpp}.
\end{proof}

The convergence results are given in the geometry induced by the Bregman divergence with potential given by $\cFm$.

\section{Enquiry on Polyak-{\L}ojaziewicz Inequality for DC Functionals}\label{sec:PL_for_w_DC}

An analog of the DC PL inequality of \citep[Definition 6.1]{oikonomidis2025forward}, writing it only using the Bregman divergences of $\cFp$ and $\cFm$, reads as follows

\begin{definition}\label{def:PL_oiko}
    We say that $\bF=\cFp-\cFm$ satisfies the Wasserstein DC PL inequality analog to \citep[Definition 6.1]{oikonomidis2025forward}, if there exists $\eta_1\ge 0$, $\eta_2\ge 0$ such that $\eta_1+\eta_2>0$ and for all $\mu\in\cP_2(\R^d)$, %
    \begin{equation} \label{eq:wdc_pl}
    \eta_1\big(\bF(\mu) - \inf \bF\big)+\eta_2\big(\bF(\Bar\T_\#\mu)-\inf \bF\big) \le \D_{\cFm}^{\mu}(\Bar\T, \id) + \D_{\cFp}^{\mu}(\id, \Bar\T)
    \end{equation}
        where $\Bar\T$ solves WCCCP, e.g.\ in its version \eqref{eq:WDCA_breg_prox}, giving $\Bar\T = \argmin_{\T\in L^2(\mu)}\ \D_{\cFm}^{\mu}(\T,\id) + \bF(\T_\#\mu)$.
\end{definition}

Under this ideal condition, we show that we have the following linear convergence rate, analogously to \citep[Lemma 1]{faust2023bregman} and \citep[Theorem 6.2]{oikonomidis2025forward}.

\begin{proposition} \label{prop:wasserstein_DC_PL}
    Assume $\bF$ satisfies the Wasserstein DC-PL inequalities \eqref{eq:wdc_pl}. Then, for all $k\ge 0$,
    \begin{equation}
        \bF(\mu_k) - \inf \bF \le \left(\frac{1-\eta_1}{1+\eta_2}\right)^k \big(\bF(\mu_0) - \inf\bF\big).
    \end{equation}
\end{proposition}
If $\eta_1 \ge 1$, then we can set $\eta_1=1$ and we have convergence in one step.
\begin{proof}
        Since $\mu_{k+1}=(\T_{k+1})_\#\mu_k$, \eqref{eq:wdc_pl} gives
        \begin{align*}
             \eta_1\big(\bF(\mu_k) - \inf \bF\big)+\eta_2\big(\bF(\mu_{k+1})-\inf \bF\big) &\le \D_{\cFm}^{\mu_k}(\T_{k+1}, \id) + \D_{\cFp}^{\mu_k}(\id, \T_{k+1})\\
             &\stackrel{ \eqref{eq:iterates_difference_gap} }{=} \bF(\mu_k)-\bF(\mu_{k+1}).
        \end{align*} 
    Rearranging, we obtain
    \begin{equation}
        \bF(\mu_{k+1}) - \inf\bF \le \frac{1-\eta_1}{1+\eta_2} \big(\bF(\mu_k)-\inf\bF\big) \le \left(\frac{1-\eta_1}{1+\eta_2}\right)^{k+1} \big(\bF(\mu_0)-\inf\bF\big).
    \end{equation}
\end{proof}
Although \citep[Lemma 6.3]{oikonomidis2025forward} in the Euclidean case shows that \citep[Definition 1]{faust2023bregman} implies their \citep[Definition 6.1]{oikonomidis2025forward}, the latter is difficult to check directly for a given functional and is more of a target inequality to establish. We thus want to argue that \citep[Definition 1]{faust2023bregman} is more promising to stand as DC PL inequality and has received a more developed discussion of when it holds, \emph{e.g.} under strong convexity. However this definition rests upon Fenchel duality on $\R^d$, for which there is no publicly available counterpart on the Wasserstein space at the time of this submission. Consequently it is unclear for now whether \Cref{def:PL_oiko} holds for functionals with strongly convex DC decompositions.

While \citep[Theorems 4 and 5]{luu2024non} do rest upon a {\L}ojaziewicz inequality, the latter is unrelated to Bregman divergences. \citet{luu2024non} achieve instead their linear rates using the extra Wasserstein regularization appearing in \eqref{eq:wasserstein_prox_dc_equivalent}. While we discussed the similarity of the two approaches in \Cref{sec:w-prox_algo_luu}, one cannot use their theory without the Wasserstein regularization, as otherwise their constant becomes vacuous.

\section{DC Theory for MMD} \label{appendix:theory_mmd}

Let $k$ be a translation-invariant kernel of the form $k(x,y)=\psi(x-y)$ for all $x,y\in\R^d$, with $\psi$ admitting a DC decomposition $\psi=\psip-\psim$. Then by \Cref{prop:DC_radial_MMD}, the squared MMD \eqref{eq:mmd} admits the DC decomposition $\bF=\cFp-\cFm$ where for all $\mu\in\cPr$,
\begin{equation} \label{eq:dc_decomposition_radial}
    \left\{
        \begin{array}{ll}
             \cFp(\mu) = \tfrac12 \iint \psip(x-y)\ \mathrm{d}\mu(x)\mathrm{d}\mu(y) + \int \Vm\mathrm{d}\mu + c(\nu), & \Vm(\cdot) = \int \psim(\cdot-y)\ \mathrm{d}\nu(y), \\
             \cFm(\mu) = \tfrac12 \iint \psim(x-y)\ \mathrm{d}\mu(x)\mathrm{d}\mu(y) + \int \Vp\mathrm{d}\mu, & \Vp(\cdot) = \int \psip(\cdot-y)\ \mathrm{d}\nu(y),
        \end{array}
    \right.
\end{equation}
and $\cFp,\cFm$ are both totally convex. Moreover, if $\psip$ and $\psim$ are respectively $\ap\ge 0$ and $\am\ge 0$ strongly convex, then $\cFp$ is $\am$-totally convex and $\cFm$ is $\ap$-totally convex as they inherit the strong convexity of the potential, while the interaction terms are only convex, see \citep[Section 9.3]{ambrosio2008gradient}.

We will now focus on radial kernels. Let $\Omega \subset \mathbb{R}^d$ be a nonempty compact convex set. Consider the geodesically convex set $\mathcal{P}_2(\Omega)$. 
Let
\begin{equation}
    S_* \coloneqq \sup_{x,y \in \Omega} \|x-y\|_2.
\end{equation}

We consider in what follows: 
\begin{equation}
    \psi^{\pm}(z) := q_{\pm}(\|z\|_2^2),\quad z \in \Omega -\Omega\coloneqq \{x-y \, | \, x,y\in\Omega\},
\end{equation}
where $q_{\pm}\in C^2([0,S_*])$.
The MMD functional $\mathcal{F}$ is therefore defined as a function of $q_{\pm}.$

\subsection{Strong Convexity of $\cFm$ and $\cFp$ for Radial Kernels}

We first study the Hessian of radial functions, in order to derive conditions for $\psi_{\pm}$ to be (strongly) convex and to be able to apply \Cref{prop:DC_radial_MMD}.

Let $q\in C^2([0,S_*])$ and define:
\begin{equation}
    \underline\lambda[q] \coloneqq \inf_{0\le s\le S_*}\min\bigl\{2q'(s),\,2q'(s)+4sq''(s)\bigr\}, \quad \overline\Lambda[q] \coloneqq\sup_{0\le s\le S_*}\max\bigl\{2q'(s),\,2q'(s)+4sq''(s)\bigr\}.
\end{equation}
We use the shorthands $\lambda_\pm:=\underline\lambda[q_\pm]$ and $\Lambda_\pm:=\overline\Lambda[q_\pm]$. First, we compute the Hessian of $\psi(z)=q(\|z\|_2^2)$ and bound its eigenvalues using $\underline\lambda[q]$ and $\overline\Lambda[q]$.

\begin{lemma}[Radial Hessian Bounds] \label{lem:radialhessian}
    Let $\psi(z)= q(\|z\|_2^2)$ on $z\in \Omega - \Omega$, we have
    \begin{equation}
         \nabla^2 \psi(z) = 2q'(\|z\|_2^2) I_d + 4 q''(\|z\|_2^2) zz^{\top}.
    \end{equation}
    Hence, we have the following bounds on $\nabla^2\psi$:
    \begin{equation}
        \underline\lambda[q] I_d \preceq  \nabla^2 \psi(z) \preceq \overline\Lambda[q] I_d,\quad z \in \Omega -\Omega.
    \end{equation}
    If $\underline \lambda(q)\geq 0,$ then $\psi$ is convex on $\Omega-\Omega$.
\end{lemma}

\begin{proof}
    See \Cref{proof:lem_radialhessian}.
\end{proof}

Now, building on the previous Lemma, we deduce sufficient conditions under which $\psi$ is convex.

\begin{lemma}[Sufficient Condition for Convexity] \label{lem:sufficient}
    Assume $q\in C^2([0,S_*])$ such that $q'(s)\geq 0 $ and $q''(s)\geq 0$ for all $s\in [0,S_*]$, then $z\mapsto \psi(z)= q(\|z\|_2^2)$ is convex in $\Omega - \Omega$, and $\overline \Lambda (q) \geq \underline \lambda(q) \geq 0$.
\end{lemma}
\begin{proof}[Proof of Lemma \ref{lem:sufficient}] 
    We always have $\overline \Lambda (q) \geq \underline \lambda(q) $ and under these conditions it is easy to see that $\underline \lambda(q) \geq 0$.
\end{proof}

This lemma provides sufficient conditions to get a DC decomposition of the form \eqref{eq:dc_decomposition_radial} by \Cref{prop:DC_radial_MMD}. 

\begin{corollary} [Strong Convexity of DC decomposition in MMD]  
    Let $k:\R^d\times\R^d\to\R$ be a radial kernel, \emph{i.e.} such that $k(x,y)=q(\|x-y\|_2^2)$, where $q$ admits a decomposition $q=\qp-\qm$, with $\qp$ and $\qm$ satisfying the conditions of \Cref{lem:sufficient}. Then, considering $\cFp,\cFm$ as defined in \eqref{eq:dc_decomposition_radial}, we have on $\mathcal{P}_2(\Omega)$, that
    \begin{enumerate}
        \item $\cFp$ is $\ap$-strongly totally convex with $\ap = \lambda_-$;
        \item $\cFm$ is $\am$-strongly totally convex with $\am = \lambda_+$. 
    \end{enumerate}
    \label{pro:StrongConvexMMD}
\end{corollary}

\begin{proof}
    We apply \Cref{prop:DC_radial_MMD}.
\end{proof}

Building on these two lemmas, we show that some usual kernels admit such decompositions. We begin by discussing several DC decompositions for the Gaussian kernel that we can use for MMD. The first one is based on a remark from \citep{luu2024non}, and the other ones are based on observing that it is a radial kernel with $q(t)=e^{-t/(2h)}$. Hence, we simply need to find a DC decomposition of $q=\qp-\qm$ which satisfies the assumptions in \Cref{lem:sufficient}. 

\paragraph{Gaussian kernel.} Recall that the Gaussian kernel is a radial kernel with $q(t)=e^{-\alpha t}$ with $\alpha> 0$ (often taken as $\alpha=1/(2h)$, $h$ being the bandwidth).

    \citet{luu2024non} observed in their Appendix A.2 that for $k$ differentiable and $L$-smooth, \emph{i.e.} satisfying $\|\nabla k(x,y) - \nabla k(x',y')\|_2^2 \le L^2\big(\|x-x'\|_2^2 + \|y-y'\|_2^2)$ for all $x,x',y,y'\in \R^d$, a DC decomposition of \eqref{eq:mmd} is given by 
\begin{equation} \label{eq:decomposition_luu_v2}
    \cFp(\mu) = \alpha \int \|\cdot\|_2^2\ \mathrm{d}\mu + \frac12\iint k(x,y)\ \mathrm{d}\mu(x)\mathrm{d}\mu(y) + c(\nu), \quad \cFm(\mu) = \int \big(\alpha \|x\|_2^2 - \V(x)\big)\ \mathrm{d}\mu(x) 
\end{equation}
for any $\alpha \ge L$. Actually we can specify in the next example the best choice of $\alpha$ for the Gaussian kernel.

\begin{example} \label{example:decomposition_luu}
    For the Gaussian kernel $k(x,y)=e^{-\|x-y\|_2^2/(2h)}$, let us find the best constant $L$. We have $\nabla_x k(x,y) = -\frac1h e^{-\|x-y\|_2^2/(2h)} (x-y)$ and $\nabla_x^2k(x,y) = \frac{1}{h^2}e^{-\|x-y\|_2^2/(2h)}(x-y)(x-y)^\top -\frac1h e^{-\|x-y\|_2^2/(2h)} I_d = \frac1h e^{-\|x-y\|_2^2/(2h)} \big(\frac1h(x-y)(x-y)^\top - I_d\big)$. Its eigenvalues are $\lambda_0 = -\frac1h e^{-\|x-y\|_2^2/(2h)}$ and $\lambda_1=(\frac1h\|x-y\|_2^2 - 1) \frac1h e^{-\|x-y\|_2^2/(2h)}$. Let $t=\frac1h \|x-y\|_2^2$, then the operator norm of the Hessian is $\|\nabla_x^2 k(x,y)\|_{\mathrm{op}} = \frac1h e^{-t/2} \max (1, |t-1|)$. The maximum in $t$ is obtained for $t=0$, thus $\|\nabla^2 k(x,y)\|_{\mathrm{op}} \le \frac2h$. Thus, we can use $\alpha=\frac2h$ in \eqref{eq:decomposition_luu_v2}.
\end{example}

\begin{figure}[t]
    \centering

    \hspace*{\fill}
    \subfloat[$q_+(t)=\cosh(t),\ q_-(t)=\sinh(t)$ \label{fig:cosh}]{\includegraphics[width=0.45\linewidth]{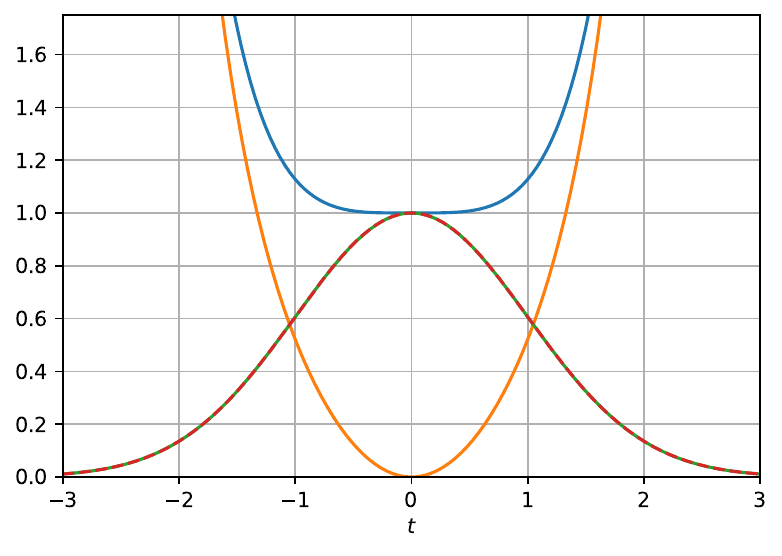}} \hfill
    \subfloat[$q_+(t)=e^{-\alpha t} + \alpha t,\ q_-(t)=\alpha t$ \label{fig:hessian_decomposition}]{\includegraphics[width=0.45\linewidth]{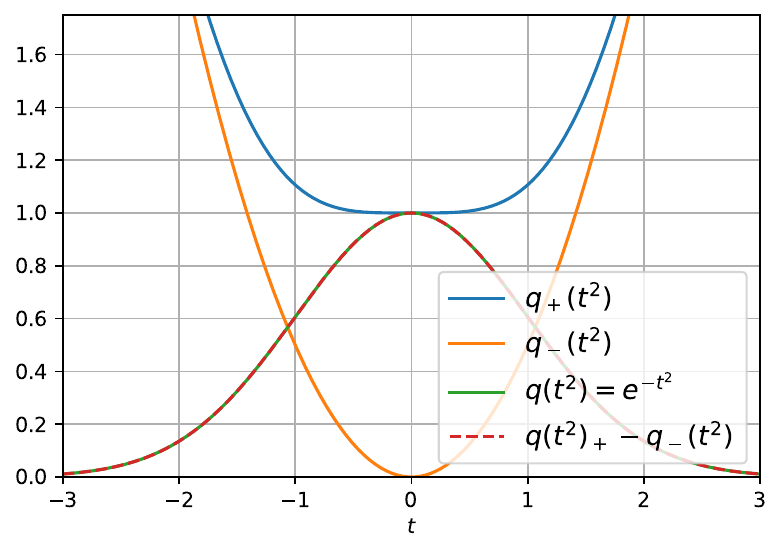}} \hfill
    \hspace*{\fill}

    \caption{Plot of two possible DC decompositions of $t\mapsto q(t^2) = e^{-t^2}$ that we use to get the DC decomposition of MMD with Gaussian kernel. On the left, we show the $\cosh/\sinh$ decomposition, and on the right, we show the decomposition based on the sign of the Hessian.}
    \label{fig:decomposition_q}
\end{figure}

 Another natural DC decomposition is based on the observation that $e^{-t}=\cosh(t) - \sinh(t)$. We show in the next lemma that this gives a valid decomposition for \Cref{prop:DC_radial_MMD}. We plot the decomposition on \Cref{fig:cosh}.

\begin{lemma}[DC decomposition of the Gaussian Kernel based on $\cosh/\sinh$]  \label{lem:DCGaussian}
    For $x,y \in \Omega$ and $\alpha >0$, a Gaussian kernel $k(x,y)=\exp(-\alpha \|x-y\|_2^2)$ admits a DC decomposition as follows: $k(x,y)= \psi^+(x-y) -\psi^-(x-y)$, where $\psi^{\pm}(z)=q_{\pm}(\|z\|_2^2)$ with
    \begin{equation}
        \qp(s) = \cosh(\alpha s),\quad \qm(s)=\sinh(\alpha s).
    \end{equation}
    For $s\in [0,S_*]$ the maximum and minimum eigenvalues of the Hessian satisfy
    \begin{equation}\label{eq:def_coeff_Gauss_Jordan+}
        \lambda_+ = 0, \quad \Lambda_+= 2\alpha \sinh(\alpha S_*)+ 4\alpha^2 S_* \cosh(\alpha S_*),
    \end{equation}
    and 
    \begin{equation}\label{eq:def_coeff_Gauss_Jordan-}
        \lambda_- = 2\alpha, \quad \Lambda_- = 2\alpha \cosh(\alpha S_*) + 4\alpha^2 S_* \sinh(\alpha S_*).
    \end{equation}
\end{lemma}

\begin{proof}
    See \Cref{proof:lem_DCGaussian}.
\end{proof}

Applying \Cref{prop:DC_radial_MMD}, we can deduce that $\cFp$ is $2\alpha$-totally convex, with $\alpha>0$, while $\cFm$ is only totally convex. Hence, we can apply \Cref{prop:cv_convexity} and obtain convergence of WCCCP towards a stationary point at a sublinear rate, provided the iterates remain in $\Omega$.

Next, we consider $q(t)=e^{-\alpha t}$ directly. Its derivative is $q'(t)=-\alpha e^{-\alpha t}$ and second derivative $q''(t)=\alpha^2 e^{-\alpha t}$. We note that it is convex but it does not satisfy \Cref{lem:sufficient} as $q'<0$. However, adding a linear term large enough, we can obtain a DC decomposition of $q$ satisfying the assumptions of \Cref{lem:sufficient}. For this, we apply \eqref{eq:dc_decomposition_jordan} with $A=\alpha$.
\begin{lemma}
    Let $q:t\mapsto e^{-\alpha t}$ for $\alpha>0$. Then the DC decomposition written in \eqref{eq:dc_decomposition_jordan} is $\qp(t)=e^{-\alpha t} + \alpha t$ and $\qm(t)=\alpha t$ for all $t\in \R$.
\end{lemma}
\begin{proof}
    See \Cref{proof:lemma_dc_gaussian_jordan}.
\end{proof}
Note that for this decomposition the definition of the bounds $\lambda$ do not need to be restricted to a compact set. The decomposition is plotted on \Cref{fig:hessian_decomposition}.

\begin{proposition}[DC decomposition of the Gaussian Kernel based on \eqref{eq:dc_decomposition_jordan}] \label{prop:DCGaussian_Jordan}
    For $x,y\in\R^d$, $\alpha>0$, a Gaussian kernel $k(x,y)=e^{-\alpha \|x-y\|_2^2}$ has a DC decomposition $k(x,y)=\qp(\|x-y\|_2^2) - \qm(\|x-y\|_2^2)$ with, for all $s\in [0, +\infty)$,
    \begin{equation}
        \qp(s) = e^{-\alpha s} + \alpha s,\quad \qm(s) = \alpha s.
    \end{equation}
    The maximum and minimum eigenvalues of the Hessian satisfy
    \begin{equation}
        \lambda_+ = 0, \quad \Lambda_+= 2\alpha (1+2e^{-\frac32}),
    \end{equation}
    and 
    \begin{equation}
        \lambda_- = 2\alpha, \quad \Lambda_- = 2\alpha.
    \end{equation}
\end{proposition}

\begin{proof}
    See \Cref{proof:prop_Dcgaussian_Jordan}.
\end{proof}

\paragraph{Smoothed Riesz Kernel.} We now deal with the negative distance (Riesz) kernel $k(x,y)=-\|x-y\|_2$. Since it is not differentiable in $x=y$, we instead focus on a smoothed version of it, defined as $k_\varepsilon(x,y) = -\sqrt{\varepsilon + \|x-y\|_2^2}$. In this case, we have $q(t) = -\sqrt{t + \varepsilon}$ for all $t\ge 0$.

\begin{lemma}[DC decomposition of the smoothed Riesz Kernel] \label{lem:DCRiesz}
    For $\varepsilon>0$, define $k_{\varepsilon}(x,y)= -\sqrt{\varepsilon + \|x-y\|_2^2}$. The smoothed Riesz kernel has a DC decomposition as follows:  $k_{\varepsilon}(x-y)= \psi^+(x-y) -\psi^-(x-y)$, where $\psi^{\pm}(z)=q_{\pm}(\|z\|_2^2)$ with
    \begin{equation}
        \qp(s) = 0, \quad \qm(s) = \sqrt{\varepsilon+s}.
    \end{equation}
    For $s\in [0,S_*]$ the maximum and minimum eigenvalues satisfy
    \begin{equation}
        \lambda_+ = 0, \quad \Lambda_+ =0, \quad \lambda_- = \frac{\varepsilon}{(\varepsilon+ S_*)^{3/2}}, \quad \Lambda_- =\frac{1}{\sqrt{\varepsilon}}.
    \end{equation}
\end{lemma}

\begin{proof}
    See \Cref{proof:lem_DCRiesz}.
\end{proof}

    \paragraph{Rational Quadratic Kernel.} We now turn to the rational quadratic kernel $k(x,y)= \psi(x-y),$ where   $\psi(z)=\frac{1}{(c^2 + \|z\|_2^2)^{\alpha}}, \alpha \geq 1$, $c>0$.
\begin{lemma}[Rational Quadratic Kernel DC decomposition based on \eqref{eq:dc_decomposition_jordan}] \label{lemma:DC_rational} 
    For $x,y\in\R^d$, $\alpha \geq 1$, a rational quadratic kernel  has a DC decomposition $k(x,y)=\qp(\|x-y\|_2^2) - \qm(\|x-y\|_2^2)$ with, for all $s\in [0, +\infty)$,  
    \begin{equation}
        \qp(s) = \frac{1}{(c^2 + s)^{\alpha}} +\alpha c^{-2(\alpha+1)} s, \quad \quad \qm(s) = \alpha c^{-2(\alpha+1)}  s.
    \end{equation}
   The maximum and minimum eigenvalues of the Hessian satisfy:
   \begin{equation}
        \lambda_+ = 0, \quad \Lambda_+ = f(s^*), \quad \lambda_- =2\alpha c^{-2(\alpha+1)},  \quad \Lambda_- = 2\alpha c^{-2(\alpha+1)},
    \end{equation}   
    where $f(s)= 2 q'_{+}(s) + 4 s q''_{+}(s)$ and $s^*=\frac{6c^2}{4(\alpha+2)-6}$.    
\end{lemma}
\begin{proof}   
    See \Cref{proof:lemma_DC_rational}.
\end{proof}

\subsection{Lipschitz Continuity of $\gW\mathcal{F}^+$ and Stationarity}

We now show that for a DC decomposition based on radial kernels as in \Cref{lem:sufficient}, the Wasserstein gradient of $\gW\cFp$ satisfies a Lipschitz condition, which corresponds to the assumption in \Cref{prop:cv_convexity}.

\begin{proposition}[Lipschitz Continuity of $\gW\cFp$]  \label{pro:LipschitzProjection}    
    Assume $\lambda_{\pm}\geq 0$ and $\Lambda_{\pm}<\infty$. Let $\sigma, \mu \in \mathcal{P}_{2}(\Omega)$ be two a.c measures, consider $\T$ such that $\sigma =\T_{\#} \mu$, then we have
    \begin{equation}
        \|  \gW\cFp(\sigma)\circ \T - \gW\cFp(\mu)\|_{L^2(\mu)} \leq L\|\T-\id\|_{L^2(\mu)},
    \end{equation}
    with  $L= \sqrt{2}\Lambda_{+} + \Lambda_-$.
\end{proposition}

\begin{proof}
    See \Cref{proof:pro_LipschitzProjection}.
\end{proof}

Using this Proposition, we can apply \Cref{prop:cv_convexity} to obtain a sublinear rate over the minimum of the Wasserstein gradient over the scheme.

\begin{theorem}[WCCCP with MMD leads to a Stationary Measure] \label{th:cv_stationary}
    Let $\nu \in \mathcal{P}_2(\Omega),$ where $\Omega$ is a compact and convex nonempty set in $\mathbb{R}^d$. Consider $(\mu_{k})_{k\ge 0}$ the WCCCP iterates \eqref{eq:argmin_WDCA_maps} for the MMD functional $\bF$ with a radial translation-invariant kernel that admits a DC decomposition: $k(x,y)= \psip(x-y) - \psim(x-y)$. Assume $\lambda_+,\lambda_- \ge 0$, $\lambda_+ + \lambda_- > 0$ and $0\le \Lambda_{\pm}<\infty$, and WCCCP iterates belong to $\cP_2(\Omega)$, then these iterates satisfy, for all $K\ge 1$,
    \begin{equation}
        \min_{0\le k\le K-1}\ \| \gW\cF(\mu_k)\|^2_{L^2(\mu_{k})} 
        \le \frac{2 (\sqrt{2}\Lambda_+ + \Lambda_-)^2}{\lambda_{+}+ \lambda_-} \frac{\big(\bF(\mu_0)-\bF(\mu_K)\big)}{K}.
    \end{equation}
\end{theorem}

\begin{proof}
    See \Cref{proof:th_cv_stationary}.
\end{proof}

\begin{remark} 
    We can get quantitative bounds for the Gaussian kernel with $\cosh/\sinh$ DC split using \Cref{lem:DCGaussian}, and with the Jordan decomposition using \Cref{prop:DCGaussian_Jordan}. Similarly, we also have these bounds for the smoothed Riesz kernel using \Cref{lem:DCRiesz} and for the rational quadratic kernel with \Cref{lemma:DC_rational}.
\end{remark}

\subsection{Critical points and Local Convergence} \label{appendix:cv_critical_pt}

In the following we analyze the conditions under which we obtain local convergence of WCCCP for the MMD functional with base space $\Omega$. All the statements are in the weak topology and assume a continuous kernel $k$.

First, we consider a condition on the initialization \emph{w.r.t.} the critical value gap as follows:
\begin{assumption} The set $\Omega$ is convex and compact and the following holds true:
\begin{enumerate}[leftmargin=2em]
    \item \textbf{\emph{Strong Convexity \& Lipschitz Continuity.}} $\cFp$ and $\cFm$ are totally strongly convex, $\gW\cFp$ is Lipschitz continuous, and $\mu \mapsto \|\gW\cF(\mu)\|_{L^2(\mu)}$ is lower semicontinuous;
    \item \textbf{\emph{Critical-value gap}.} Setting \[c_{*}\coloneqq \inf \{\cF(\mu) : \|\gW\bF(\mu) \|_{L^2(\mu)}= 0 , \mu \neq \nu\},\]
     the initialization $\mu_0$ satisfies:
    \[\cF(\mu_0)< c_{*}. \]
    \end{enumerate}
    \label{Assump:localIninit}
\end{assumption}

Under  Assumption \ref{Assump:localIninit}, we have the following local convergence result, interpreted as: if one starts with a near optimal initialization, by descent, there must be convergence. The condition $\cF(\mu_0)< c_{*}$ entails that $\nu$ is an isolated critical point.   
\begin{proposition}[Local Convergence of WCCCP for MMD under Critical Value Gap Assumption]
    Under Assumption \ref{Assump:localIninit}, the sequence produced by WCCCP converges to $\nu$, i.e., $\mu_{k} \to \nu$.
    \label{prop:convcompact}
\end{proposition}
    
\begin{proof}
    See \Cref{proof:prop_conv_compact}.

\end{proof}

The following proposition gives local linear convergence without a compactness assumption on the base space, which may be $\R^d$, under local quadratic growth, a local Wasserstein PL inequality, and a suitable initialization; compactness of the orbit closure follows from convergence.

\begin{proposition}[Local Linear Convergence of WCCCP under Quadratic Growth and a PL Inequality] 
    Let $(\mu_k)_{k\geq 0}$ be the exact WCCCP iterates \eqref{eq:argmin_WDCA_maps} for the MMD functional $\cF$ with target $\nu$. Assume $\cFp$ and $\cFm$ are W-differentiable and respectively $\lambda_+$- and $\lambda_-$-totally convex, with $\lambda_\pm\geq 0$ and $\lambda:=\lambda_++\lambda_->0$. Assume also that $\gW\cFp$ satisfies the Lipschitz condition \eqref{eq:lipschitz_cond} with constant $L>0$. Suppose there exist $r,c>0$ such that, writing
    \[ B_r(\nu)=\{\mu\in\cPr:\W_2(\mu,\nu)\leq r\},\]
    \begin{enumerate}[leftmargin=2em, label=(\roman*)]
        \item $\cF$ has local quadratic growth at $\nu$:
        \[\cF(\mu)\geq c\W_2^2(\mu,\nu),\qquad \mu\in B_r(\nu);\]
        \item the initialization satisfies
        \[\mu_0\in B_r(\nu),\qquad E_0:=\cF(\mu_0)<cr^2;\]
        \item $\cF$ satisfies a local Wasserstein PL inequality: there exists $m>0$ such that
        \begin{equation}\label{eq:local_wasserstein_pl_mmd}
            \|\gW\cF(\mu)\|_{L^2(\mu)}^2\geq 2m\cF(\mu),
            \qquad \mu\in B_r(\nu),\quad \cF(\mu)\leq E_0.
        \end{equation}
    \end{enumerate}
    Note that necessarily $m\le L^2/\lambda$; set $q:=1-\lambda m/L^2\in[0,1)$. Then all iterates remain in $B_r(\nu)$ and, for every $k\geq 0$,
    \[
    \cF(\mu_k)\leq q^k E_0,
    \qquad
    \W_2^2(\mu_k,\nu)\leq \frac{E_0}{c}\,q^k.
    \]
    In particular, $\mu_k\to\nu$ in $\W_2$.
    \label{prop:convnoncompact}
\end{proposition}

\begin{proof}
    See \Cref{proof:prop_convnoncompact}.

\end{proof}

\begin{remark}[Scope of the local PL assumption]
    Assumptions (i) and (iii) do not follow from the total convexity of $\cF^{\pm}$ or from the Lipschitz condition \eqref{eq:lipschitz_cond}, and whether they hold depends on the target.
    We leave establishing if they hold for the MMD kernels considered in the paper for future works. The argument above obtains the DC-PL estimate only along the iterates, where the Lipschitz condition \eqref{eq:lipschitz_cond} is assumed. 
\end{remark}

\section{Numerical Applications on MMD} \label{appendix:applications_mmd}

We now detail the experiments of \Cref{sec:xps}, as well as extra ones, starting with the experiments on the Energy Distance, and then focus on MMD with the Gaussian kernel. All the numerical applications are done on a Nvidia V100 GPU.

\subsection{Energy Distance}

We recall that the Energy distance \citep{sejdinovic2013equivalence} is of the form, for $\mu,\nu\in\cPr$,
\begin{equation}\label{eq:def_ED}
    \ed(\mu,\nu) = - \iint \|x-y\|_2\ \mathrm{d}(\mu-\nu)(x)\mathrm{d}(\mu-\nu)(y),
\end{equation}
which by \Cref{prop:DC_radial_MMD}, for a fixed target $\nu$, can be decomposed as $\tfrac12\ed(\mu,\nu) = \cFp(\mu) - \cFm(\mu)$ where
\begin{equation}
    \cFp(\mu) = \iint \|x-y\|_2\ \mathrm{d}\nu(y)\mathrm{d}\mu(x) + c(\nu), \quad \cFm(\mu) = \frac12 \iint \|x-y\|_2 \ \mathrm{d}\mu(x)\mathrm{d}\mu(y),
\end{equation}
with $c(\nu)=-\frac12 \iint \|x-y\|_2 \ \mathrm{d}\nu(x)\mathrm{d}\nu(y)$. The functional $\cFp$ is a potential energy with $V(x)=\int \|x-y\|_2\ \mathrm{d}\nu(y)$ and $\cFm$ an interaction energy, and both are totally convex. If we use the smoothed Riesz kernel $k(x,y)=-\sqrt{\varepsilon + \|x-y\|_2^2}$ and assume that $\mu,\nu\in\cP_2(\Omega)$ for $\Omega$ a compact convex set, then by \Cref{lem:DCRiesz} and \Cref{pro:StrongConvexMMD}, $\cFp$ is also $\alpha$-totally convex with $\alpha=\varepsilon/(\varepsilon + S_*)^{\frac32}$ with $S_*=\sup_{x,y\in\Omega}\ \|x-y\|_2$. In practice, we mostly use the non-smooth Riesz kernel $k(x,y)=-\|x-y\|_2$ as it works well \citep{hertrich2024generative}. Nonetheless, some smoothed versions based on convolutions have been also shown to have more favorable theoretical properties \citep{rux2026smoothed}.

\begin{figure}[t]
    \centering
    \includegraphics[width=\linewidth]{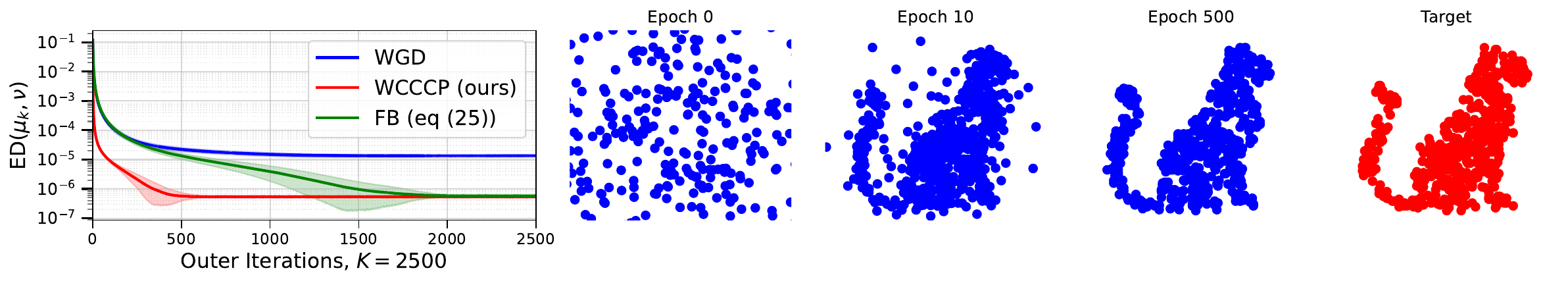}
    \includegraphics[width=\linewidth]{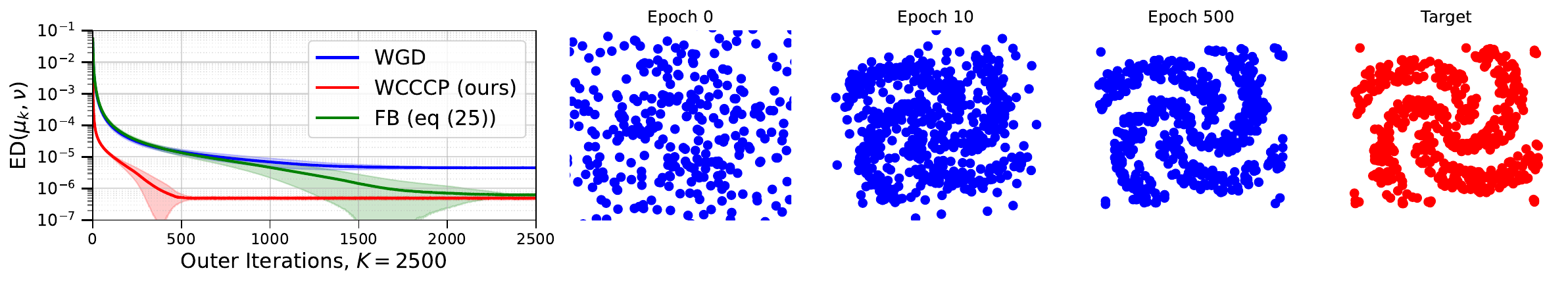}
    \includegraphics[width=\linewidth]{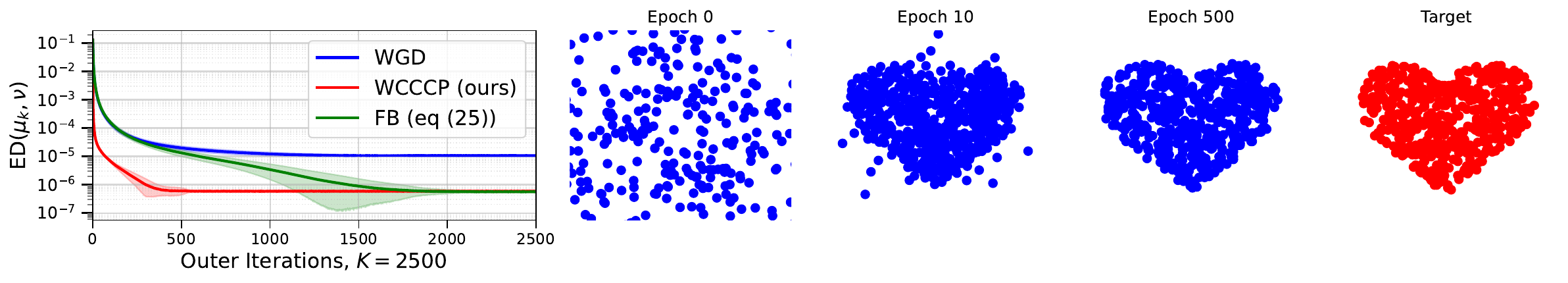}
    \includegraphics[width=\linewidth]{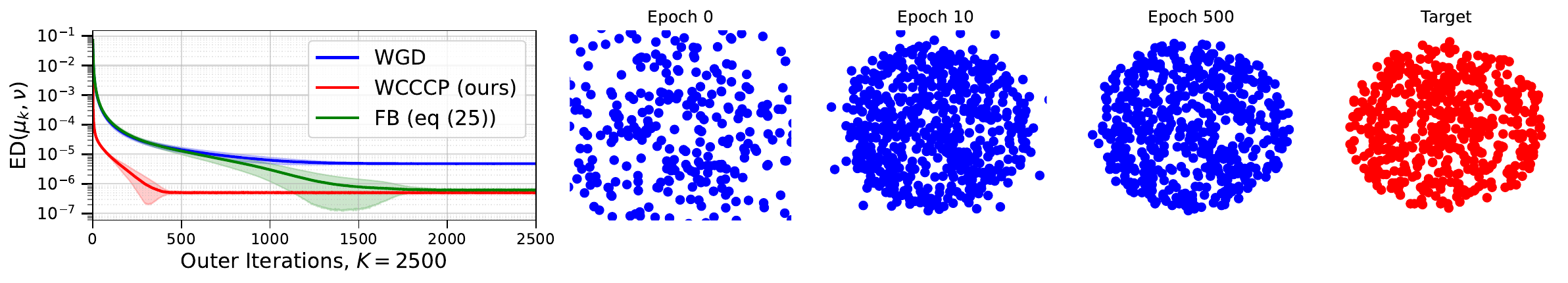}

    \caption{Convergence of the Wasserstein Gradient Descent (WGD), Forward-Backward (FB) and Wasserstein Convex-Concave Procedure (WCCCP) on the squared MMD with kernel $k(x,y)=-\|x-y\|_2$ (\textbf{Left}) and particles of WCCCP along the scheme (\textbf{Right}).}
    \label{fig:cv_mmd_riesz_shapes}

    \vspace{-2em}
\end{figure}

\begin{figure}[t]
    \centering
    \includegraphics[width=\linewidth]{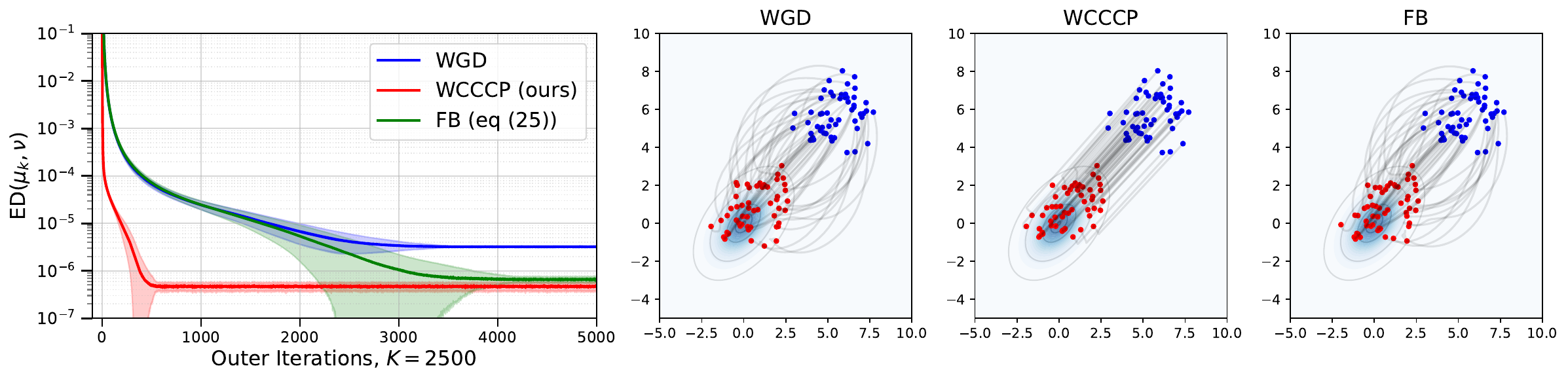}
    \includegraphics[width=\linewidth]{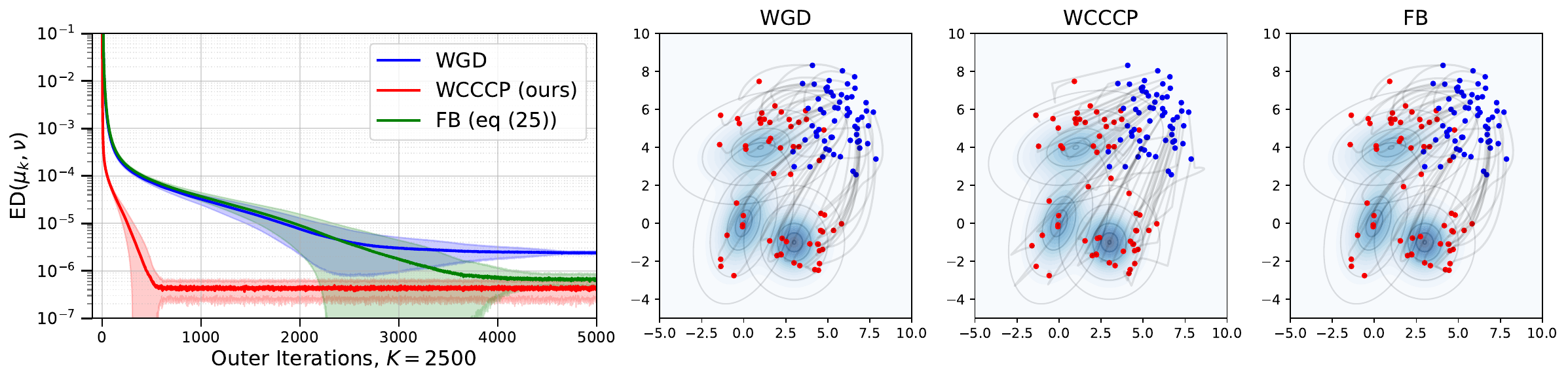}

    \caption{Optimization of $\bF(\mu)=\frac12 \ed(\mu,\nu)$ for $\nu$ a Gaussian target (\textbf{Top}) and a Gaussian mixture (\textbf{Bottom}). (\textbf{Left}) Evolution of the squared Energy distance along the flow. (\textbf{Right}) Trajectories of the particles over time. The initial particles are in blue and the final particles in red.}
    \label{fig:cccp_riesz_gaussian}
\end{figure}

\paragraph{Experiments on shapes.}

For the experiment of \Cref{fig:cv_mmd_riesz_shapes_spiral_cat}, we compare the Wasserstein gradient descent (WGD) \eqref{eq:wgd}, the Forward-Backward (FB) scheme \eqref{eq:wasserstein_proximal_dc} from \citep{luu2024non}, and our scheme WCCCP \eqref{eq:argmin_WDCA_maps}. We use a spatial discretization as described in \Cref{sec:computing}, \emph{i.e.} we first sample $n=500$ particles from $\mu_0=\cN(0,I_2)$, and $n$ uniform independent particles from the target shape to obtain the target distribution $\nu_n$. Then, at each iteration, we compute the map $\T_{k+1}$ and apply it to move each of the particles. For WGD and FB, we use as step size $\tau=1$. For WCCCP and FB, we solve each inner optimization problem with a gradient descent with momentum $m=0.9$ for $M=50$ iterations and step size $0.1$. We report on \Cref{fig:cv_mmd_riesz_shapes_spiral_cat} the values of the objective $\ed(\mu_k, \nu_n)$ for the cat and spiral shapes depending on the number of outer iterations (hence WGD is much faster but still converges to a local minimum and does not improve further). We performed the experiment 100 times, and report the average values with standard deviation to get confidence intervals. We add on \Cref{fig:cv_mmd_riesz_shapes} the results for the heart and disk shapes, as well as particles for the WCCCP at iterations 0, 10 and 500.

\begin{wrapfigure}{r}{0.5\linewidth}
    \vspace{-1em}
    \centering
    \includegraphics[width=\linewidth]{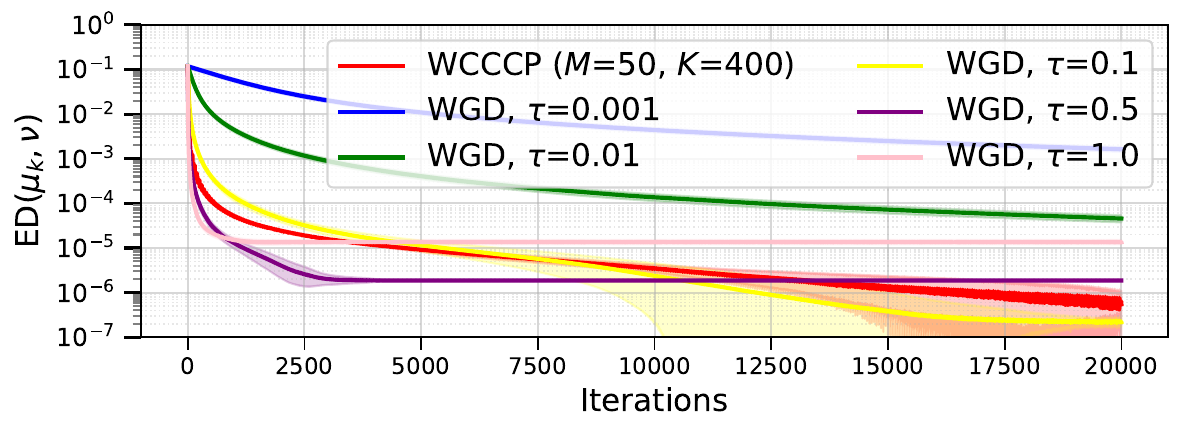}
    \caption{Convergence of WGD (with different step size and $K$=20K) and WCCCP with the same computational budget, \emph{i.e.} $M=50$ and $K=400$ (each iteration thus corresponding to one inner step). The target is the cat shape.}
    \label{fig:fair_cat}
    \vspace{-1.5em}
\end{wrapfigure}

On \Cref{fig:cccp_riesz_gaussian}, we performed the same experiment with target samples from a Gaussian $\nu=\cN(0,\Sigma)$ for $\Sigma = \bigl(\begin{smallmatrix} 1 & 0.5 \\ 0.5 & 1 \end{smallmatrix}\bigr)$ (top), and with target samples from a mixture of 3 Gaussian with equal weights, means $m_1=(0,0)$, $m_2=(3,-1)$ and $m_3=(1,4)$, and covariances $\Sigma_1=\bigl(\begin{smallmatrix} 1 & 0.5 \\ 0.5 & 1\end{smallmatrix}\bigr)$, $\Sigma_2=I_2$ and $\Sigma_3=\bigl(\begin{smallmatrix} 3 & 0.5 \\ 0.5 & 1\end{smallmatrix}\bigr)$. For both, we use the same initial distribution $\mu_0=\cN(5,I_2)$ and the same hyperparameters as the shape experiment.

\looseness=-1 We also report on \Cref{fig:fair_cat} the convergence for the cat target uniform distribution, with the exact same number of iterations between WCCCP and WGD with different step sizes, \emph{i.e.} using $K=20$K iterations of WGD and $K=400$ outer iterations of WCCCP and $M=50$ inner iterations to solve each sub-problem. We observe that for $\tau=0.1$, WGD converges better but slower than for $\tau=1$. With the same computational budget for WCCCP with $M=50$, $K=400$, we observe that the algorithm converges in similar computational time.

\paragraph{Experiments on images.}

On \Cref{fig:cv_mmd_riesz_cifar}, we performed the same experiment with target samples from the CIFAR10 dataset, whose images are of size $3\times 32 \times 32$. More precisely, we sampled 50 points by class, and hence also worked with $n=500$ particles. We started the flows from $\mu_0=\cN(0,I_d)$. We compared WGD with stepsize $\tau=1$ and WCCCP, using a gradient descent to solve the inner optimization scheme with $\tau=1$. We ran WGD for $K=$200K iterations and WCCCP for $K=$40K iterations, with $M=20$ iterations to solve each of the subproblems \eqref{eq:argmin_WDCA_maps}. We choose to use a different number of iterations to be fair in comparing the two methods, as WGD solves each inner problem in closed-form, and is thus less computationally expensive. With a Nvidia v100 GPU, WCCCP took about 1h15 while WGD took 1h30. On \Cref{fig:cv_mmd_riesz_cifar}, we plot samples along the scheme of WGD every 20K iterations, and samples along the scheme of WCCCP every 4K iterations. Below, we plot the value of the loss across iterations, where we rescaled the abscissa for WCCCP to match the abscissa of WGD as done on the image samples. We observe overall a much faster convergence of WCCCP compared to WGD on this high dimensional challenging dataset. On \Cref{fig:cv_mmd_riesz_cifar_full}, we add more samples along the schemes of WGD and WCCCP, every 2K iterations for WCCCP and 10K iterations for WGD.

We also performed the experiment on MNIST and report the results on \Cref{fig:cv_mmd_riesz_cifar_full}. Here, the samples are reported every 1K iterations for WCCCP and 5K iterations for WGD. The schemes took about 18 minutes to run for both WGD and WCCCP.

We show on \Cref{fig:cv_ED_MNIST} the evolution of the Energy distance along the scheme, averaged over 5 runs with different samples of the source and of the target. We observe that on the MNIST experiment, both schemes converge, but WCCCP converges faster (even though the iterations are rescaled). On the CIFAR10 experiment, even after 200K iterations, WGD is very far from converging, whereas WCCCP has already converged. Hence, WCCCP can be promising to accelerate convergence in high dimensions (here $d=3\times 32 \times 32=3072$).

\begin{figure}[t]
        \centering
        \includegraphics[width=0.48\linewidth]{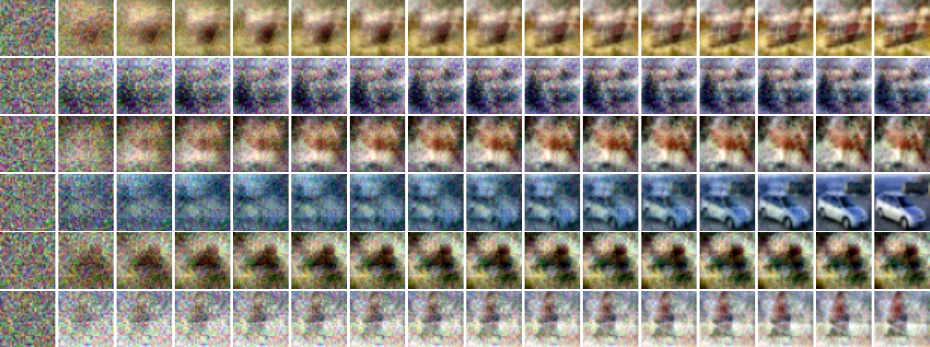} \hspace{0.2em}
        \includegraphics[width=0.48\linewidth]{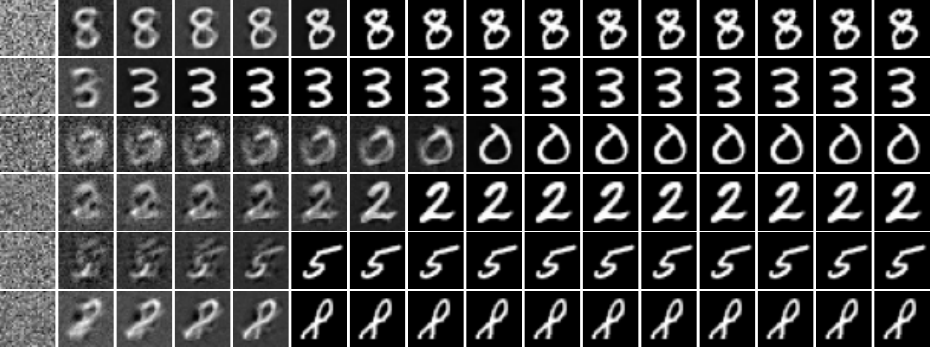}
        \\ \vspace{0.5em}
        \includegraphics[width=0.48\linewidth]{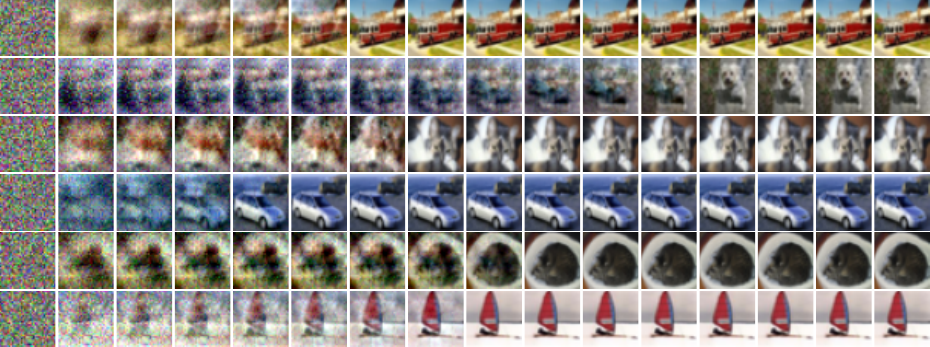} \hspace{0.2em}
        \includegraphics[width=0.48\linewidth]{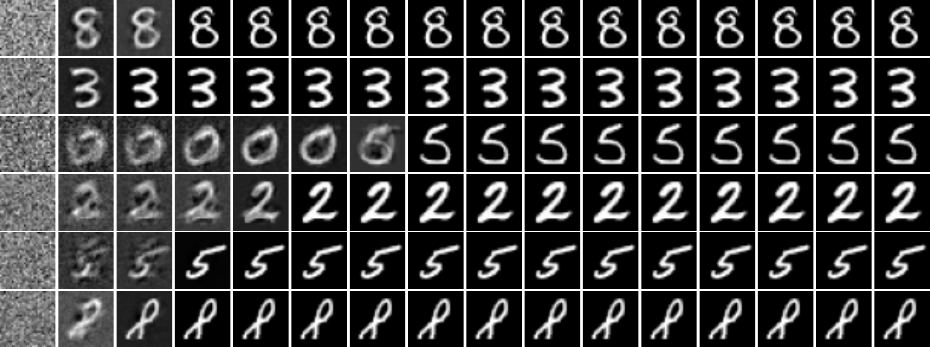}

        \caption{Samples along the scheme of WGD (\textbf{Top}) and WCCCP (\textbf{Bottom}) with $\cF(\mu)=\frac12\ed(\mu,\nu)$ as objective. On CIFAR10, we plot samples every 2K iterations for WCCCP and every 10K iterations for WGD. On MNIST, we plot samples every 1K iterations for WCCCP and every 5K iterations for WGD.}
        \label{fig:cv_mmd_riesz_cifar_full}
\end{figure}

\begin{figure}[t]
    \centering
    \includegraphics[width=0.45\linewidth]{Figures/cv_mmd_riesz/Convergence_MMD_CIFAR10.pdf}
    \includegraphics[width=0.45\linewidth]{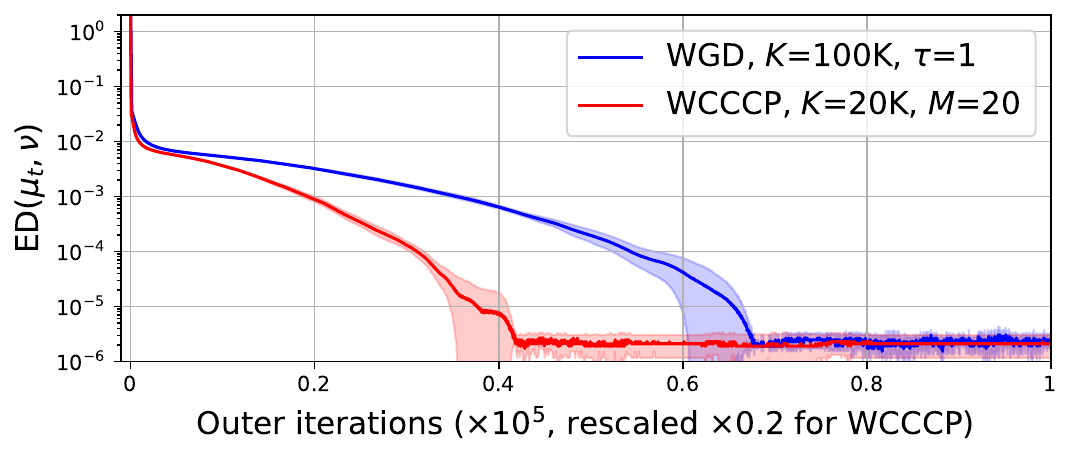}
    \caption{Evolution of $\cF(\mu)=\tfrac12\ed(\mu,\nu)$ along the WCCCP and WGD schemes for $\nu$ composed of samples of CIFAR10 (\textbf{Left}) and of MNIST (\textbf{Right}). The results are averaged over 5 different runs with different samples of the source and target distribution.}
    \label{fig:cv_ED_MNIST}
\end{figure}

\subsection{MMD with Gaussian Kernel}\label{sec:appendix_mmd_gaussian}

We now focus on MMD with Gaussian kernel $k(x,y)=e^{-\|x-y\|_2^2/(2h)}$ for a bandwidth $h>0$.

\begin{figure}[t]
    \centering
    \begin{minipage}{\linewidth}
        \centering
        \includegraphics[width=\linewidth]{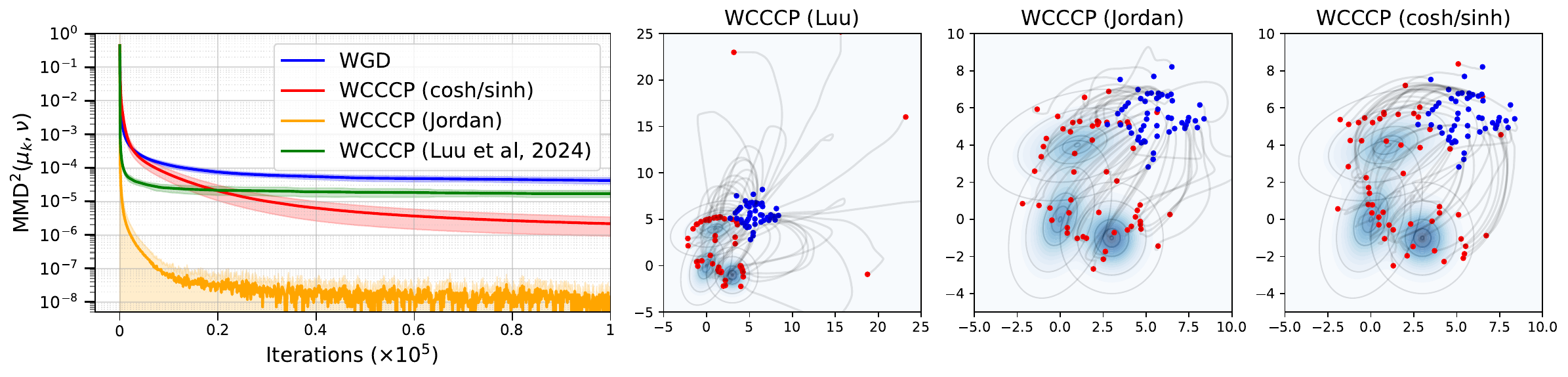}
    
        \caption{Optimization of $\bF(\mu)=\frac12 \mmd_k^2(\mu,\nu)$ for $\nu$ a Gaussian mixture target and $k$ the Gaussian kernel. (\textbf{Left}) Evolution of the squared MMD along the flow. (\textbf{Right}) Trajectories of the particles over times. The initial particles are in blue and the final particles in red.}
        \label{fig:cccp_mmd_mixture}
    \end{minipage}

    \vspace{0.3cm}

    \begin{minipage}{\linewidth}
        \centering
        \includegraphics[width=0.4\linewidth]{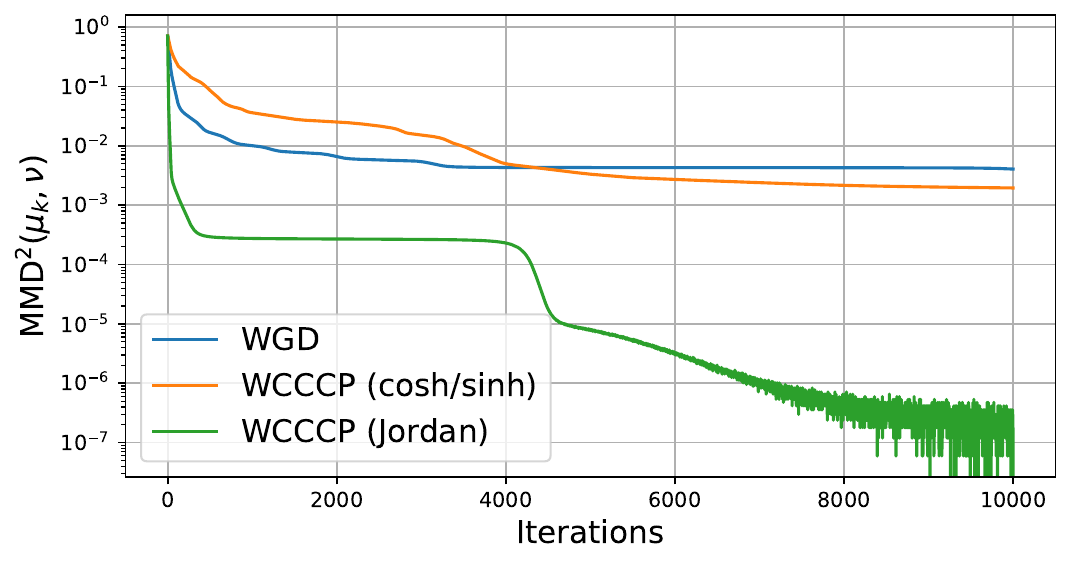}
        \includegraphics[width=0.4\linewidth]{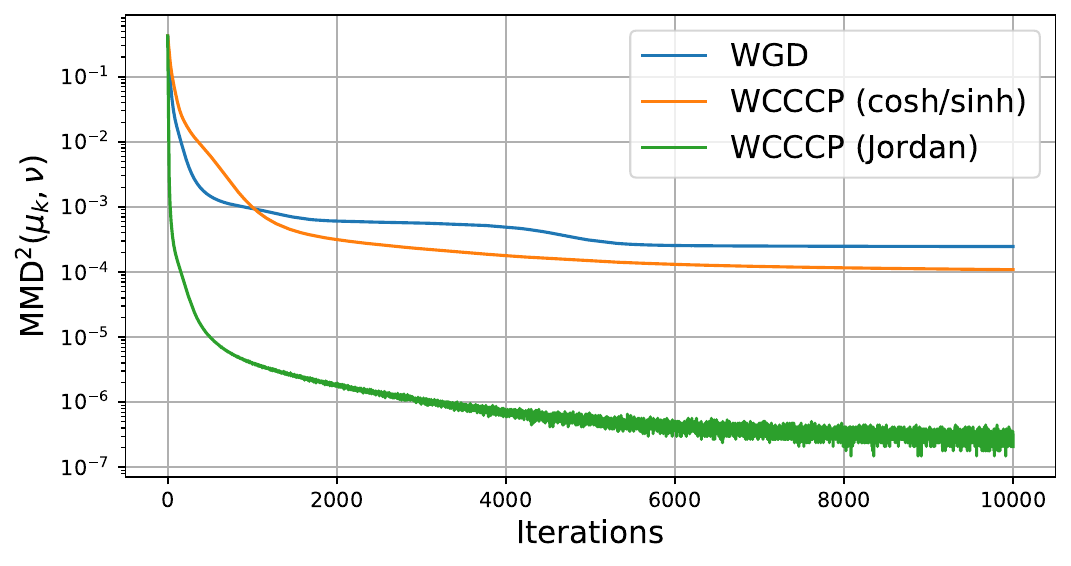}
        \caption{Loss for one run for MMD with Gaussian kernel, and Gaussian target (\textbf{Left}) and Gaussian mixture target (\textbf{Right}).}
        \label{fig:loss_mmd_gaussian}
    \end{minipage}

    \vspace{0.3cm}

    \begin{minipage}{\linewidth}
        \centering
        \includegraphics[width=0.4\linewidth]{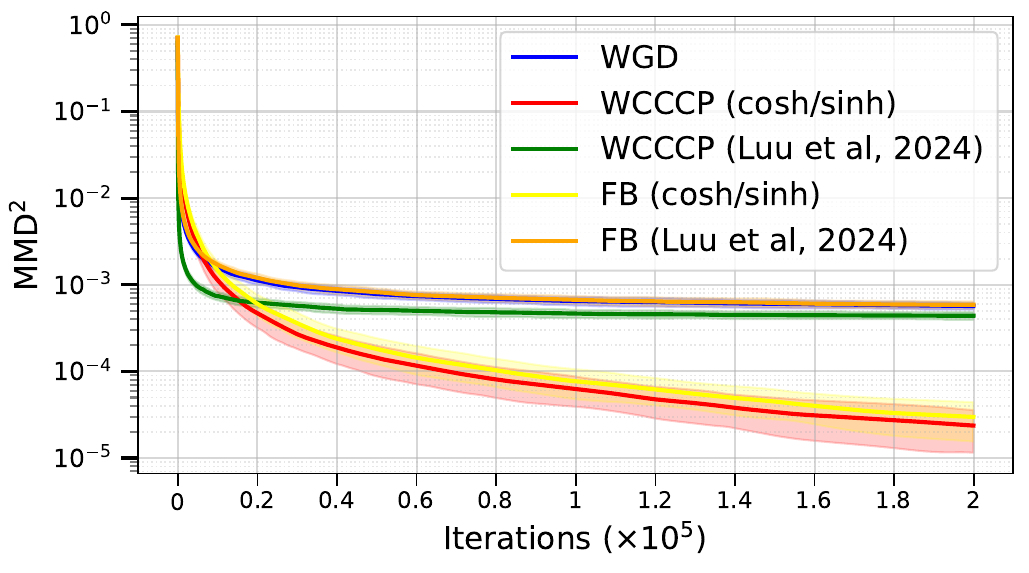}
        \includegraphics[width=0.4\linewidth]{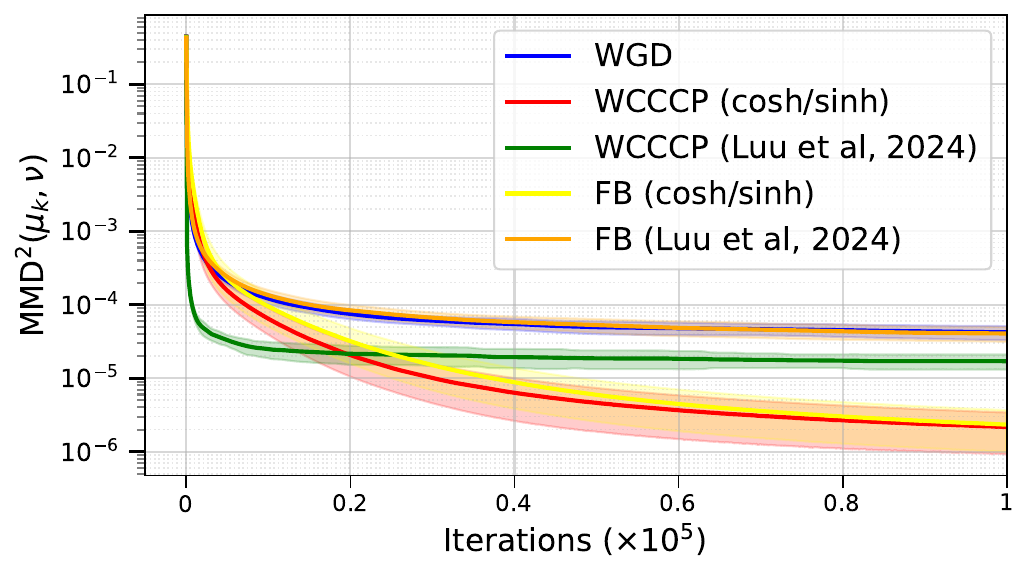}
    
        \caption{Evolution of $\bF(\mu)=\frac12\mmd^2_k(\mu,\nu)$ for $\nu$ a Gaussian target (\textbf{Left}) and a Gaussian mixture target (\textbf{Right}) over iterations, comparing WGD, WCCCP and FB with the $\cosh/\sinh$ decomposition and the DC decomposition \eqref{eq:decomposition_luu_v2}. We observe that both FB and WCCCP work well with the $\cosh/\sinh$ decomposition.}
        \label{fig:convergence_Gaussian_kernel_fb}
    \end{minipage}

\end{figure}

\paragraph{Decomposition based on the radial kernel.}

By \Cref{pro:StrongConvexMMD}, we can have a DC decomposition of the squared MMD if we find $\qp,\qm$ satisfying $\qp',\qm',\qp'',\qm''\ge 0$ and $q=\qp-\qm$. We discuss here two natural decompositions of $q:t\mapsto e^{-t/(2h)}$. 

The first one is based on the remark that $e^{-z} = \cosh(z) - \sinh(z)$ since $\cosh(z) = \big(e^z+e^{-z})/2$ and $\sinh(z)=\big(e^z - e^{-z})/2$, and hence $\qp(t)=\cosh\big(t/(2h)\big)$, $\qm(t)=\sinh\big(t/(2h)\big)$. This is based on the algebraic decomposition described in \Cref{sec:mmd} as $\cosh(t)=\sum_{k \ \mathrm{even}} \frac{t^{k}}{k!}$ and $\sinh(t)=\sum_{k\ \mathrm{odd}} \frac{t^k}{k!}$. This decomposition is valid to apply \Cref{prop:DC_radial_MMD} as shown in \Cref{lem:DCGaussian}. We show this decomposition on \Cref{fig:cosh}. We note that both $\qp$ and $\qm$ tend to take very large values, which might be prone to numerical instabilities, in particular for small bandwidths.

The second one is based on decomposing the second derivative, and taking the convex part as the function having as second derivative the non-negative part, and $\qm$ as the function whose second derivative is the negative part of $q''$. This can be done using the Jordan decomposition \eqref{eq:dc_decomposition_jordan}. The exact formula gives $\qp(s)=e^{-\alpha s} + \alpha s$ and $\qm(s)=\alpha s$. This also satisfies \Cref{prop:DC_radial_MMD} as shown in \Cref{prop:DCGaussian_Jordan}.
We plot on \Cref{fig:hessian_decomposition} the resulting DC decomposition. We observe here that this decomposition does not blow up, and should be numerically more stable.

\paragraph{Experiments.}

The Wasserstein gradient descent on $\cF(\mu)=\frac12 \mmd_k^2(\mu,\nu)$ with $k(x,y)=e^{-\|x-y\|_2^2/(2h)}$ is known to heavily depend on the choice of the bandwidth \citep{arbel2019maximum, hertrich2024generative}, and might not converge well in practice. In particular, on simple examples such as a single Gaussian distribution, \citet{arbel2019maximum, gladin2024interaction} observed that some particles get stuck further from the mode, and never converge. In our experiments, we compare several DC decompositions in the WCCCP algorithm and test for which decompositions we obtain a better convergence or not.

For this, we choose the same setting as \citep{gladin2024interaction}. We first take $n=500$ samples from $\mu_0=\cN(5, I_2)$ to get the initial distribution. Then, we take $n=500$ samples from the target distribution $\nu$. We experiment with 2 choices for $\nu$. The first one is a Gaussian $\nu=\cN(0,\Sigma)$ with $\Sigma = \bigl(\begin{smallmatrix} 1 & 0.5 \\ 0.5 & 1 \end{smallmatrix}\bigr)$.
The second is a mixture of three Gaussians with uniform weights, means $m_1=(0, 0)$, $m_2=(3, -1)$, $m_3=(1, 4)$ and covariances $\Sigma_1 = \bigl(\begin{smallmatrix} 1 & 0.5 \\ 0.5 & 1\end{smallmatrix}\bigr)$, $\Sigma_2=I_2$, $\Sigma_3=\bigl(\begin{smallmatrix} 3 & 0.5 \\ 0.5 & 1 \end{smallmatrix}\bigr)$.
Moreover, the bandwidth is fixed to $h=10$.

In each experiment, we use for WGD and FB a stepsize $\tau=1$. For FB and WCCCP, we optimize the inner problem with a gradient descent with momentum $m=0.9$, step size $\tau=5\cdot 10^{-4}$ and for 250 iterations. 

On \Cref{fig:cccp_mmd_gaussian}, we show the evolution of the loss over 200K iterations of the algorithms, comparing WGD with WCCCP for the DC decomposition \eqref{eq:decomposition_luu_v2}, and the DC decomposition of the radial kernel based on $\cosh$ and $\sinh$ as presented in \Cref{lem:DCGaussian} as well as the decomposition based on the Jordan decomposition presented in \Cref{prop:DCGaussian_Jordan}. We show the evolution of the objective over the iterations and average the results over 100 different set of initial and target samples.

We observe that the two radial decompositions perform better than WGD and the DC decomposition \eqref{eq:decomposition_luu_v2}. The $\cosh/\sinh$ decomposition seems to converge more consistently than the one based on the Jordan decomposition, even though it takes longer to converge. Indeed, the Jordan decomposition converges much faster, but also shows a higher variance, and sometimes, particles get stuck away from the target mode, which does not seem to happen with the $\cosh/\sinh$ decomposition.

We also show the result on \Cref{fig:cccp_mmd_mixture} of the same experiment on the Gaussian mixture target. On this target, the Jordan decomposition outperforms all the others. On \Cref{fig:loss_mmd_gaussian}, we add the loss for one run and $n=50$. We observe that WCCCP with Jordan decomposition can plateau during the training, and seems to have several convergence regimes for the Gaussian target. Moreover, it is much faster in its first iterations than the $\cosh/\sinh$ decomposition.

\begin{figure}[t]
    \centering
    \begin{minipage}{\linewidth}
        \centering
        \includegraphics[width=\linewidth]{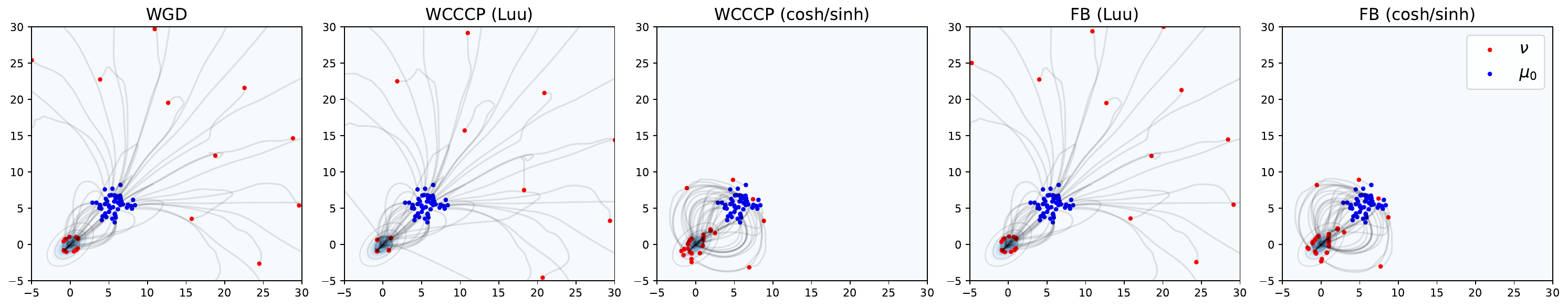}
        \caption{Evolution of the particles over the scheme for the Gaussian target.}
        \label{fig:particle_Gaussian_full}
    \end{minipage}

    \vspace{0.3cm}

    \begin{minipage}{\linewidth}
        \centering
        \includegraphics[width=\linewidth]{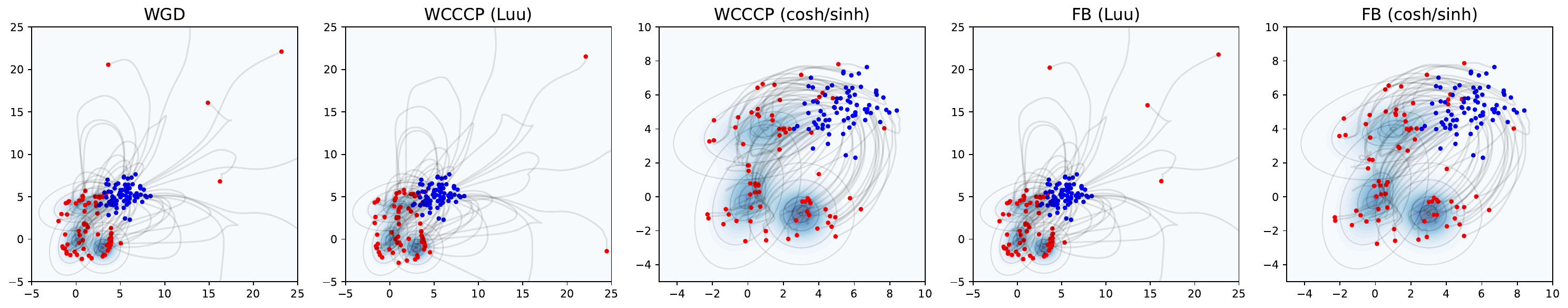}
        \caption{Evolution of the particles over the scheme for the Gaussian mixture target.}
        \label{fig:particle_mixture_full}
    \end{minipage}

    \vspace{-1.5em}
\end{figure}

On \Cref{fig:convergence_Gaussian_kernel_fb}, we plot the evolution of the loss for the Gaussian and Gaussian mixture targets, comparing WGD with WCCCP and the Forward-Backward (FB) algorithm studied in \citep{luu2024non}. We compare WCCCP and FB with two DC decompositions: the one based on the DC decomposition of the radial kernel with $\cosh/\sinh$ as presented in \Cref{lem:DCGaussian} and the one based on the DC decomposition \eqref{eq:decomposition_luu_v2}. We observe that both WCCCP and FB performance heavily depend on the choice of the DC decomposition. In particular, the $\cosh/\sinh$ decomposition seems to work well for both algorithms, while the iterates of the decomposition \eqref{eq:decomposition_luu_v2} get stuck in a local minimum, like WGD. We add on \Cref{fig:particle_Gaussian_full} and \Cref{fig:particle_mixture_full} the evolutions of particles over the flows of WGD, and WCCCP and FB for both decompositions.

\paragraph{Ablation on the DC Decomposition} \label{xp:ablation_dc_decomposistion}

\begin{wrapfigure}{r}{0.4\linewidth}
    \vspace{-1em}
    \centering
    \includegraphics[width=\linewidth]{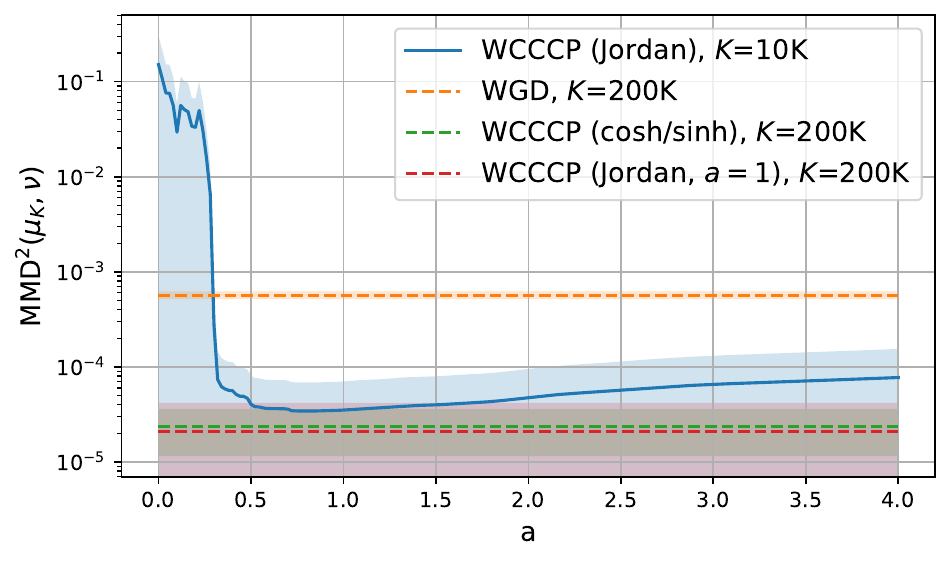}
    \caption{Evolution of the loss at convergence for $\qp^a(s)=e^{-s/2h} + a s/(2h)$, $\qm^a(s) = a s / (2h)$.}
    \label{fig:ablation_a}
    \vspace{-1em}
\end{wrapfigure}

The experiments exhibit the importance of the choice of the DC decomposition of the target functional, as different DC decompositions induce very different behaviors. We investigate how the performance of the WCCCP method evolves when we change the Jordan decomposition. To do so, we define $\qp^a(s) = e^{-s/(2h)} + a \cdot s / (2h)$ and $\qm^a(s) = a \cdot s / (2h)$ for $a>0$. For $a=1$, this recovers the Jordan decomposition. We show in \Cref{fig:ablation_a} the value of the squared MMD objective after running WCCCP for $a\in)0,4]$ with the decomposition $q=\qp^a-\qm^a$, for 10K iterations, and averaged for each $a$ over 5 different initializations.
For $a<1$, note that the decomposition is not valid as $\qp$ is not non-decreasing. We can run nonetheless WCCCP. We observe that the loss at convergence is pretty stable for $a\in [0.5, 1]$ as there is no clear minimum, but the results seem to get worse as $a$ increases. In the regime where $\qp^a$ is far from convex, \emph{i.e.} $a<0.5$, we observe very bad convergence.

\section{Proofs} \label{appendix:proofs}

\subsection{Proof of \Cref{prop:wdca_equivalent_md}} \label{proof:prop_wdca_equivalent_md}

Let $k\ge 0$, $\mu_k\in\cPr$. On one hand, we have for any $\T\in L^2(\mu_k)$, setting $\J$ as the r.h.s.\ of \eqref{eq:upper_bound_dc},
\begin{equation}
    \begin{aligned}
        \J(\T) &\coloneqq \cFp(\T_\#\mu_k) - \cFm(\mu_k) - \langle \gW\cFm(\mu_k), \T-\id\rangle_{L^2(\mu_k)} \\
        &= \cFp(\T_\#\mu_k) - \cFp(\mu_k) - \langle \gW\cFp(\mu_k), \T-\id\rangle_{L^2(\mu_k)} \\
        &\qquad\qquad\qquad +\cFp(\mu_k) + \langle \gW\cFp(\mu_k), \T-\id\rangle_{L^2(\mu_k)} \\
        &\quad - \cFm(\mu_k) - \langle \gW\cFm(\mu_k), \T-\id\rangle_{L^2(\mu_k)} \\
        &= \D_{\cFp}^{\mu_k}(\T,\id) + \cFp(\mu_k)-\cFm(\mu_k) + \langle \gW\cFp(\mu_k)-\gW\cFm(\mu_k), \T-\id\rangle_{L^2(\mu_k)} \\
        &= \D_{\cFp}^{\mu_k}(\T,\id) + \bF(\mu_k) + \langle \gW\bF(\mu_k), \T-\id\rangle_{L^2(\mu_k)}.
    \end{aligned}
\end{equation}
Hence, \eqref{eq:argmin_WDCA_maps} is equivalent to a mirror descent on $\bF$ with geometry induced by the Bregman divergence with Bregman potential $\cFp$.

On the other hand, we also have, for any $\T\in L^2(\mu_k)$,
\begin{equation}
    \begin{aligned}
        \J(\T) &= \cFp(\T_\#\mu_k) - \cFm(\mu_k) - \langle \gW\cFm(\mu_k), \T-\id\rangle_{L^2(\mu_k)} \\
        &= \cFp(\T_\#\mu_k) - \cFm(\T_\#\mu_k) + \cFm(\T_\#\mu_k) - \cFm(\mu_k) - \langle \gW\cFm(\mu_k), \T-\id\rangle_{L^2(\mu_k)} \\
        &= \bF(\T_\#\mu_k) + \D_{\cFm}^{\mu_k}(\T,\id).
    \end{aligned}
\end{equation}
Hence, \eqref{eq:argmin_WDCA_maps} is also equivalent with a Bregman proximal descent on $\bF$ with geometry induced by the Bregman divergence with Bregman potential $\cFm$.

\subsection{Proof of \Cref{prop:descent_lemma}} \label{proof:prop_descent_lemma}

\begin{equation}
    \begin{aligned}
        \bF(\mu_{k+1})&=\bF(\mu_k)+\cFp(\mu_{k+1})-\cFp(\mu_{k})-\cFm(\mu_{k+1})+\cFm(\mu_{k}) \\
        &=\bF(\mu_k)-\D_{\cFm}^{\mu_k}(\T_{k+1}, \id)+\cFp(\mu_{k+1})-\cFp(\mu_{k}) - \langle \gW\cFm(\mu_k), \T_{k+1}-\id\rangle_{L^2(\mu_k)} \\
        &= \bF(\mu_k) - \D_{\cFm}^{\mu_k}(\T_{k+1},\id) - \Ck.
    \end{aligned}
\end{equation}
If $\cFp$ is W-differentiable, observe that by the first order conditions in \eqref{eq:argmin_WDCA_maps}, $\gW\cFp(\mu_{k+1})\circ\T_{k+1} = \gW\cFm(\mu_k)$, and thus $\Ck=\D_{\cFp}^{\mu_k}(\id,\T_{k+1})$.

\subsection{Proof of \Cref{prop:critical_point}} \label{proof:prop_critical_point}

If  $\Ck=0$, then $\id\in \argmin_{\T\in L^2(\mu_k)}\ \cFp(\T_\#\mu_k) - \cFm(\mu_k) - \langle \gW\cFm(\mu_k), \T-\id\rangle_{L^2(\mu_k)}$. Hence the first order condition gives $0 = \gW\cFp(\id_\#\mu_{k})\circ \id - \gW\cFm(\mu_k)=\gW\cF(\mu_k)$.

\subsection{Proof of \Cref{prop:cv_stop_criteria}} \label{proof:prop_cv_stop_criteria}  

We just use a telescopic sum on \eqref{eq:general_descent}, discarding $\D_{\cFm}^{\mu_k}(\T_{k+1}, \id)$ (nonnegative by convexity of $\cFm$), and then use the observation \eqref{eq:base_nonconvex_WDCA}.

\subsection{Proof of \Cref{prop:cv_convexity}} \label{proof:prop_cv_convexity} 

By \eqref{eq:general_descent}, for all $k\in \{0,\dots,K-1\}$,
\begin{equation} \label{eq:recall_generali_descent}
    \cF(\mu_k)-\cF(\mu_{k+1}) = \D_{\cFm}^{\mu_k}(\T_{k+1},\id) + \Ck.
\end{equation}

$\cFp$ and $\cFm$ are respectively $\ap$ and $\am$-convex along iterates with respect to $\mu\mapsto \int \tfrac12 \|\cdot\|_2^2\ \mathrm{d}\mu$. As $\cFm$ is W-differentiable, we have $\D_{\cFm}^{\mu_k}(\T_{k+1},\id) \ge \frac{\am}{2}\|\T_{k+1}-\id\|_{L^2(\mu_k)}^2$. For all $t\in [0,1]$, define $\rho_t = \big((1-t)\T_{k+1}+t\id\big)_\#\mu_k$ and 
\begin{equation}
    \begin{aligned}
        g(t) &:= \J\big((1-t)\T_{k+1} + t\id\big) - \cFm(\mu_k) \\
        &= \cFp(\rho_t) - \cFm(\mu_k) - \langle \gW\cFm(\mu_k), (1-t)(\T_{k+1}-\id)\rangle_{L^2(\mu_k)}.
    \end{aligned}
\end{equation}
By $\ap$-total convexity of $\cFp$, for all $t\in [0,1]$,
\begin{equation} \label{eq:cvx_g}
    \begin{aligned}
        g(t) &= \cFp(\rho_t) - \cFm(\mu_k) - \langle \gW\cFm(\mu_k), (1-t)(\T_{k+1}-\id)\rangle_{L^2(\mu_k)} \\
        &\le (1-t)\cFp(\rho_0) + t\cFp(\rho_1) - \ap \frac{t(1-t)}{2} \|\T_{k+1}-\id\|_{L^2(\mu_k)}^2 \\ &\quad- (1-t)\cFm(\mu_k) - t\cFm(\mu_k) - (1-t) \langle \gW\cFm(\mu_k), \T_{k+1}-\id\rangle_{L^2(\mu_k)} \\
        &= (1-t) g(0) + t g(1) - \ap \frac{t(1-t)}{2} \|\T_{k+1}-\id\|_{L^2(\mu_k)}^2.
    \end{aligned}
\end{equation}
By definition of $\T_{k+1}$ in \eqref{eq:argmin_WDCA_maps}, $g(0)=\J(\T_{k+1}) - \cFm(\mu_k) \le g(t)$ for all $t\in [0,1]$. Hence, 
\begin{multline}
        g(0) \le (1-t) g(0) + t g(1) - \ap \frac{t(1-t)}{2} \|\T_{k+1}-\id\|_{L^2(\mu_k)}^2 \\ 
        \iff \ap \frac{(1-t)}{2} \|\T_{k+1}-\id\|_{L^2(\mu_k)}^2 \le g(1) - g(0) = \J(\id) - \J(\T_{k+1}) = \Ck.
\end{multline}
Taking the limit $t\to 0$, we get $\Ck\ge \frac{\ap}{2}\|\T_{k+1}-\id\|_{L^2(\mu_k)}^2$.

Hence, by \eqref{eq:recall_generali_descent},
\begin{equation}  \label{eq:69}
    \cF(\mu_k)-\cF(\mu_{k+1})\ge \frac{\ap+\am}{2}\|\T_{k+1}-\id\|_{L^2(\mu_k)}^2,
\end{equation}
and by \eqref{eq:base_nonconvex_WDCA},
\begin{equation}
    \begin{aligned}
        \min_{k\in\{0,\dots,K-1\}}\ \|\T_{k+1}-\id\|_{L^2(\mu_k)}^2 &\le \frac{2}{\ap + \am} \min_{k\in\{0,\dots,K-1\}}\ \bF(\mu_k)-\bF(\mu_{k+1}) \\
        &\le \frac{2}{\ap+\am} \cdot \frac{\bF(\mu_0)-\bF(\mu_K)}{K}.
    \end{aligned}
\end{equation}

Under the additional assumption that $\cFp$ is W-differentiable and $\|\gW\cFp(\mu_{k+1})\circ\T_{k+1}-\gW\cFp(\mu_k)\|_{L^2(\mu_k)} \le L \|\T_{k+1}-\id\|_{L^2(\mu_k)}$, then by the first order condition \eqref{eq:foc_wdca},
\begin{equation} \label{eq:71}
    \begin{aligned}
        \|\gW\bF(\mu_k)\|_{L^2(\mu_k)} &= \|\gW\cFp(\mu_k) - \gW\cFm(\mu_{k})\|_{L^2(\mu_k)} \\
        &= \|\gW\cFp(\mu_k) - \gW\cFp(\mu_{k+1})\circ\T_{k+1}\|_{L^2(\mu_k)} \\
        &\le L\|\T_{k+1}-\id\|_{L^2(\mu_k)}.
    \end{aligned}
\end{equation}
Hence, combining this with the previous result,
\begin{equation}
    \min_{k\in \{0,\dots,K-1\}} \ \|\gW\cF(\mu_k)\|_{L^2(\mu_k)}^2 \le \frac{2L^2}{\ap+\am} \cdot \frac{\bF(\mu_0)-\bF(\mu_K)}{K}.
\end{equation}

\subsection{Proof of \Cref{prop:scheme_luu}} \label{proof:prop_scheme_luu}

Let $\sS_{k+1}\coloneqq\id + \tau \gW\cFm(\mu_k)$. First, notice that since $\nu_{k+1}\in\cPa$ by assumption, then \eqref{eq:wasserstein_proximal_dc} is equivalent to
\begin{equation}
    \left\{
    \begin{array}{ll}
        \sS_{k+1} = \argmin_{\sS\in L^2(\mu_k)}\ \frac{1}{2\tau} \|\sS-\id\|_{L^2(\mu_k)}^2 - \langle \gW\cFm(\mu_k), \sS-\id\rangle_{L^2(\mu_k)}, & \nu_{k+1} = (\sS_{k+1})_\#\mu_k, \\
        \T_{k+1} = \argmin_{\T\in L^2(\nu_{k+1})}\ \frac{1}{2\tau}\|\T-\id\|_{L^2(\nu_{k+1})}^2 + \cFp(\T_\#\nu_{k+1}), & \mu_{k+1} = (\T_{k+1})_\#\nu_{k+1}.
    \end{array}
    \right.
\end{equation}
The first equation is obtained by the first order condition, and the second by observing that $\T_{k+1}$ is necessarily an OT map between $\nu_{k+1}$ and $\mu_{k+1}$. Indeed, as $\nu_{k+1}\in\cPa$, by Brenier's theorem, there exists a unique OT map between $\nu_{k+1}$ and $\mu_{k+1}$, and if $\T_{k+1}$ is not such that $\|\T_{k+1}-\id\|_{L^2(\nu_{k+1})}^2 = \W_2^2(\mu_{k+1}, \nu_{k+1})$, then we can replace $\T_{k+1}$ by the OT map $\T_{\nu_{k+1}}^{\mu_{k+1}}$, which will not change $\cFp\big((\T_{\nu_{k+1}}^{\mu_{k+1}})_\#\nu_{k+1}\big)=\cFp(\mu_{k+1})$ but will give a better transport cost.

Let $J(\T)=\cFp(\T_\#\mu_k) - \langle \gW\cFm(\mu_k), \T-\id\rangle_{L^2(\mu_k)} + \frac{1}{2\tau}\|\T-\id\|_{L^2(\mu_k)}^2$. The gradient of $J$ gives
\begin{equation}
    \begin{aligned}
        \nabla J(\T) &= \gW\cFp(\T_\#\mu_k)\circ \T - \gW\cFm(\mu_k) + \frac{1}{\tau}(\T-\id) \\
        &= \gW\cFp(\T_\#\mu_k) \circ \T + \frac{1}{\tau} \big(\T - (\id + \tau \gW\cFm(\mu_k))\big) \\
        &= \gW\cFp(\T_\#\mu_k)\circ \T + \frac{1}{\tau} \big(\T - \sS_{k+1}\big).
    \end{aligned}
\end{equation}
Taking the first order conditions, 
\begin{equation} \label{eq:foc_tilde}
    \nabla J(\Tilde{\T}_{k+1}) = 0 \iff \Tilde{\T}_{k+1} = \argmin_{\Tilde\T\in L^2(\mu_k)}\ \cFp(\Tilde\T_\#\mu_k) + \frac{1}{2\tau}\|\Tilde\T-\sS_{k+1}\|_{L^2(\mu_k)}^2.
\end{equation}

Now let us show that $\Tilde{\T}_{k+1} = \T_{k+1}\circ\sS_{k+1}$, \emph{i.e.}
\begin{equation} \label{eq:ineq_min}
    \min_{\Tilde\T\in L^2(\mu_k)}\ \cFp(\Tilde\T_\#\mu_k) + \frac{1}{2\tau}\|\Tilde\T-\sS_{k+1}\|_{L^2(\mu_k)}^2 = \min_{\T\in L^2(\nu_{k+1})}\ \cFp(\T_\#\nu_{k+1}) + \frac{1}{2\tau}\|\T-\id\|_{L^2(\nu_{k+1})}^2.
\end{equation}

On one hand, notice by a change of variables, since $\nu_{k+1}=(\sS_{k+1})_\#\mu_k$, that for all $\T\in L^2(\nu_{k+1})$,
\begin{equation}
    \frac{1}{2\tau}\|\T-\id\|_{L^2(\nu_{k+1})}^2 + \cFp(\T_\#\nu_{k+1}) = \frac{1}{2\tau} \|\T\circ \sS_{k+1} - \sS_{k+1}\|_{L^2(\mu_k)}^2 + \cFp\big((\T\circ \sS_{k+1})_\#\mu_k\big).
\end{equation}
Since $\{\T\circ \sS_{k+1},\ \T\in L^2(\nu_{k+1})\}\subset L^2(\mu_k)$, 
\begin{equation}
    \begin{aligned}
        &\min_{\T\in L^2(\nu_{k+1})} \ \frac{1}{2\tau}\|\T-\id\|_{L^2(\nu_{k+1})}^2 + \cFp(\T_\#\nu_{k+1})\\
        &= \min_{\T\in L^2(\nu_{k+1})} \  \frac{1}{2\tau} \|\T\circ \sS_{k+1} - \sS_{k+1}\|_{L^2(\mu_k)}^2 + \cFp\big((\T\circ \sS_{k+1})_\#\mu_k\big) \\
        &\ge \min_{\T\in L^2(\mu_k)}\ \frac{1}{2\tau}\|\T-\sS_{k+1}\|_{L^2(\mu_k)}^2 + \cFp(\T_\#\mu_k) \\
        &\stackrel{\eqref{eq:foc_tilde}}{=} \frac{1}{2\tau}\|\Tilde{\T}_{k+1} - \sS_{k+1}\|_{L^2(\mu_k)}^2 + \cFp\big((\Tilde{\T}_{k+1})_\#\mu_k\big).
    \end{aligned}
\end{equation}

Now, let us suppose by contradiction that it is a strict inequality, \emph{i.e.}
\begin{equation}
    \min_{\T\in L^2(\nu_{k+1})} \ \frac{1}{2\tau}\|\T-\id\|_{L^2(\nu_{k+1})}^2 + \cFp(\T_\#\nu_{k+1}) > \frac{1}{2\tau}\|\Tilde{\T}_{k+1} - \sS_{k+1}\|_{L^2(\mu_k)}^2 + \cFp\big((\Tilde{\T}_{k+1})_\#\mu_k\big).
\end{equation}
Define $\eta = (\Tilde{\T}_{k+1})_\#\mu_k$. Since $\nu_{k+1}\in\cPa$, there exists an OT map $\T^\eta\in L^2(\nu_{k+1})$ such that $\W_2^2(\eta, \nu_{k+1}) = \|\T^\eta - \id\|_{L^2(\nu_{k+1})}^2$. Notice that $(\Tilde{\T}_{k+1}, \sS_{k+1})_\#\mu_k \in \Pi(\eta, \nu_{k+1})$, hence $\|\Tilde{\T}_{k+1} - \sS_{k+1}\|_{L^2(\mu_k)}^2 \ge \W_2^2(\eta, \nu_{k+1})$ and 
\begin{equation}
    \begin{aligned}
        \min_{\T\in L^2(\nu_{k+1})} \ \frac{1}{2\tau}\|\T-\id\|_{L^2(\nu_{k+1})}^2 + \cFp(\T_\#\nu_{k+1}) &> \frac{1}{2\tau}\|\Tilde{\T}_{k+1} - \sS_{k+1}\|_{L^2(\mu_k)}^2 + \cFp\big((\Tilde{\T}_{k+1})_\#\mu_k\big) \\
        &\ge \frac{1}{2\tau}\|\T^\eta-\id\|_{L^2(\nu_{k+1}}^2 + \cFp((\T^\eta)_\#\nu_{k+1}).
    \end{aligned}
\end{equation}
However, $\T^\eta \in L^2(\nu_{k+1})$ and thus this is a contradiction. Therefore, we necessarily have an equality in \eqref{eq:ineq_min} and $\Tilde{\T}_{k+1} = \T_{k+1}\circ \sS_{k+1}$.

\subsection{Proof of \Cref{prop:cv_generalized_geodesics}} \label{proof:prop_cv_generalized_geodesics}

Similarly as in \Cref{proof:prop_cv_convexity}, since $\Vm$ is $\am$-strongly convex, $\cFm$ is $\am$-totally convex and we have $\D_{\cFm}^{\mu_k}(\T_{k+1}, \id) \ge \frac{\am}{2}\|\T_{k+1}-\id\|_{L^2(\mu_k)}^2$. Let us show that $\Ck \ge \frac{\ap}{2}\|\T_{k+1}-\id\|_{L^2(\mu_k)}^2$.

First, notice by \eqref{eq:WDCA_breg_prox} that $\T_{k+1}$ is an OT map for the Bregman OT problem $\W_{\cFm}$, \emph{i.e.} $\W_{\cFm}(\mu_{k+1},\mu_k)=\D_{\cFm}^{\mu_k}(\T_{k+1}, \id)$ (see \Cref{appendix:convex_case} for the definition of $\W_{\cFm}$). Indeed, $\mu_k\in\cPa$ as there is $\cH$ in the optimized functional. Hence, there exist transport maps from $\mu_k$ to $\mu_{k+1}$. If $\T_{k+1}$ is not the optimal one, we can find another one $\T$ such that $\T_\#\mu_k=\mu_{k+1}$ which leaves $\cF(\T_\#\mu_k)=\cF((\T_{k+1})_\#\mu_k)$ unchanged and will give a better cost $\D_{\cFm}^{\mu_k}(\T,\id) < \D_{\cFm}^{\mu_k}(\T_{k+1}, \id)$, which contradicts that $\T_{k+1}$ is optimal.

As $\T_{k+1}$ is an OT map for $\W_{\cFm}$, by \citep[Proposition 15]{bonet2024mirror}, there exists a convex function $u_k:\R^d\to \R$ such that $\T_{k+1} = \nabla u_k \circ \gW\cFm(\mu_k) = \nabla u_k \circ \nabla \Vm$ $\mu_k$-a.e. Define for all $t\in [0,1]$, $\rho_t = \big((1-t)\T_{k+1} + t\id\big)_\#\mu_k$ an interpolating curve between $\rho_0=\mu_{k+1}$ and $\rho_1=\mu_k$. Since $\Vm$ is $\am>0$-strongly convex, $\nabla \Vm$ is invertible and $(\nabla\Vm)^{-1}=\nabla (\Vm)^*$. Hence, the inverse is the gradient of a convex function. Thus, $\rho_t = \big((1-t)\nabla u_k + t \nabla(\Vm)^*\big)_\#(\nabla\Vm)_\#\mu_k$ is a generalized geodesic with base measure $(\nabla\Vm)_\#\mu_k$. Define now for all $t\in [0,1]$
\begin{equation}
    \begin{aligned}
        g(t) &:= \J\big((1-t)\T_{k+1} + t\id\big) - \cFm(\mu_k) \\
        &= \cFp(\rho_t) - \cFm(\mu_k) - \langle \gW\cFm(\mu_k), (1-t)(\T_{k+1}-\id)\rangle_{L^2(\mu_k)}.
    \end{aligned}
\end{equation}
Then, by convexity of $\cFp$ along generalized geodesics, we have for all $t\in [0,1]$,
\begin{equation}
    \begin{aligned}
        g(t) &= \cFp(\rho_t) - \cFm(\mu_k) - \langle \gW\cFm(\mu_k), (1-t)(\T_{k+1}-\id)\rangle_{L^2(\mu_k)}  \\
        &\le (1-t) g(0) + t g(1) - \ap \frac{t(1-t)}{2}\|\T_{k+1}-\id\|_{L^2(\mu_k)}^2.
    \end{aligned}
\end{equation}
We recover \eqref{eq:cvx_g} and can apply the same proof as in \Cref{proof:prop_cv_convexity} starting from there to find \eqref{eq:cv_min_w_dist}.

\subsection{Proof of \Cref{prop:DC_radial_MMD}} \label{proof:prop_DC_radial_MMD} 

We start with two preliminary lemmas.

\begin{lemma} \label{lemma:potential_int_cvx}
    Let $\nu\in\cPr$ and $\alpha \ge 0$. Let $\psi:\R^d\to\R$ be a $\alpha$-strongly convex function that is $\nu$-integrable. Then $\V:x\mapsto \int \psi(x-y)\ \mathrm{d}\nu(y)$ is $\alpha$-strongly convex. 
\end{lemma}

\begin{proof}
    Let $x_0,x_1\in\R^d$, $t\in [0,1]$, then
    \begin{equation}
        \begin{aligned}
            \V\big((1-t)x_0+tx_1\big) &= \int \psi\big((1-t)x_0+tx_1 - y\big)\ \mathrm{d}\nu(y) \\
            &= \int \psi\big((1-t)(x_0-y) + t(x_1-y)\big)\ \mathrm{d}\nu(y) \\
            &\le (1-t) \int \psi(x_0-y)\ \mathrm{d}\nu(y) + t \int \psi(x_1-y)\ \mathrm{d}\nu(y) - \frac{\alpha t(1-t)}{2} \|x_0 - x_1\|_2^2 \\
            &= (1-t) \V(x_0) + t\V(x_1) - \frac{\alpha t(1-t)}{2} \|x_0 - x_1\|_2^2.
        \end{aligned}
    \end{equation}
\end{proof}

\begin{lemma} \label{lemma:quadratic_growth}
    Let $\nu\in\cPr$. Let $\psi:\R^d\to \R$ be $\nu$-integrable and define $\V:x\mapsto \int \psi(x-y)\ \mathrm{d}\nu(y)$. Assume there exists $a,b\in\R$ such that $\psi(x) \ge -a-b\|x\|_2^2$ for all $x\in\R^d$. Then, there exists $a',b'\in \R$ such that $\V(x) \ge -a'-b'\|x\|_2^2$.
\end{lemma}

\begin{proof}
    Let $x\in \R^d$, then
    \begin{equation}
        \begin{aligned}
            \V(x) = \int \psi(x-y)\ \mathrm{d}\nu(y) &\ge -a -b \int \|x-y\|_2^2\ \mathrm{d}\nu(y). %
        \end{aligned}
    \end{equation}
    
    If $b\ge 0$, we can use that $\int \|x-y\|_2^2\ \mathrm{d}\nu(y) \le 2 \|x\|_2^2 + 2 \int \|y\|_2^2\ \mathrm{d}\nu(y)$, and thus we have the result for $a'=a + 2 b \int \|y\|_2^2\ \mathrm{d}\nu(y)$ and $b'=2b$. If $b<0$, then $\V(x) \ge -a$ and we can use $a'=a$, $b'=0$.
\end{proof}
We now move to proving \Cref{prop:DC_radial_MMD}. First, let us focus on the potential term of MMD, \emph{i.e.}
\begin{equation}
    \cV(\mu) = \int \V\ \mathrm{d}\mu, \quad \V(\cdot) = -\int k(\cdot, y)\ \mathrm{d}\nu(y).
\end{equation}
Since for all $x,y\in \R^d$, $k(x,y) = \psi(x-y) = \psip(x-y) - \psim(x-y)$, we can rewrite $\V$ as, for all $x\in \R^d$,
\begin{equation}
    \begin{aligned}
        \V(x) &= -\int \psi(x-y)\ \mathrm{d}\nu(y) \\
        &= - \int \big(\psip(x-y) - \psim(x-y)\big)\ \mathrm{d}\nu(y) \\
        &= \int \psim(x-y)\ \mathrm{d}\nu(y) - \int \psip(x-y)\ \mathrm{d}\nu(y) \\
        &= \Vm(x) - \Vp(x).
    \end{aligned}
\end{equation}
Moreover, as $\psim$ and $\psip$ are respectively $\am$ and $\ap$-strongly convex and locally Lipschitz, $\Vm$ and $\Vp$ are respectively $\am$ and $\ap$-strongly convex by \Cref{lemma:potential_int_cvx} and continuous. 

We now show $\psim$ and $\psip$ have a negative with at most quadratic growth. Since they have full domain, they have nonempty subdifferentials everywhere, in particular in $0$. Take for instance $p\in\partial \psim(0)$, then $\psim(x) \ge \psim(0)+\langle p, x \rangle$. If $p=0$, then $\psim$ is lower bounded. Assume $p\neq 0$, then
\begin{equation}
    \psim(x) \ge \psim(0)+\langle p, x \rangle \ge \psim(0)-\|p\|_2 \|x\|_2 \ge  \psim(0) -\|p\|_2(1+\|x\|_2^2).
\end{equation}
The same reasoning applies to $\psip$. Consequently there exists $a^+, a^-, b^+, b^-\in \R$ such that $\psip(\cdot) \ge -a^+-b^+\|\cdot\|_2^2$ and $\psim(\cdot) \ge -a^--b^-\|\cdot\|_2^2$. Consequently, $\Vm$ and $\Vp$ have a negative part with a quadratic growth using \Cref{lemma:quadratic_growth}. 

Hence, by \citep[Proposition 9.3.2]{ambrosio2008gradient}, $\cV$ can be decomposed as a difference of two strongly totally convex potential energies $\cV = \cVm - \cVp$ with $\cVm=\int\Vm\mathrm{d}\mu$ $\am$-totally convex, $\cVp=\int\Vp\mathrm{d}\mu$ $\ap$-totally convex.

Similarly, the interaction energy term $\cW(\mu) = \frac12 \iint k(x,y)\ \mathrm{d}\mu(x)\mathrm{d}\mu(y)$ can be decomposed as a difference of totally convex interaction energies $\cW=\cWp-\cWm$ (by \citep[Proposition 9.3.5]{ambrosio2008gradient}) with 
\begin{equation}
    \cWp(\mu) = \frac12 \iint \psip(x-y)\ \mathrm{d}\mu(x)\mathrm{d}\mu(y),\quad \cWm(\mu) = \frac12 \iint \psim(x-y)\ \mathrm{d}\mu(x)\mathrm{d}\mu(y).
\end{equation}

Defining $\cFp(\mu) = \cWp + \cVm + c$ and $\cFm(\mu) = \cWm + \cVp$, we have for all $\mu\in\cPr$,
\begin{equation}
    \bF(\mu) = \cFp(\mu) - \cFm(\mu),
\end{equation}
where $\cFp$ is $\am$-totally convex and $\cFm$ is $\ap$-totally convex, as the sum of a convex term (the interactions) and of a strongly-convex term (the potentials).

\subsection{Proof of \Cref{lemma:dc_gaussian_jordan}} \label{proof:lemma_dc_gaussian_jordan}

Concerning the Jordan decomposition, we have for all $t\ge 0$, $q(t)=e^{-\alpha t}$, hence $q'(t)=-\alpha e^{-\alpha t}$ and $q''(t)=\alpha^2 e^{-\alpha t} \ge 0$. Setting $A=\max(0, -q'(0)) = \alpha$,
\begin{equation}
    \qm(t) = \alpha t - \int_0^t (t-s) \min\big(0, q''(s)\big)\ \mathrm{d}s = \alpha t, 
\end{equation}
and
\begin{equation}
    \qp(t)=q(t) + \qm(t) = e^{-\alpha t} + \alpha t.
\end{equation}

Concerning the algebraic decomposition, for all $t\ge 0$,
\begin{equation}
    \begin{aligned}
        q(t) = e^{-\alpha t} = \sum_{k\ge 0} \frac{(-1)^k (\alpha t)^k}{k!} &= \sum_{k\ge 0, \mathrm{ even}} \frac{(\alpha t)^k}{k!} - \sum_{k\ge 0, \mathrm{ odd}} \frac{(\alpha t)^k}{k!} \\
        &= \sum_{\ell\ge 0} \frac{(\alpha t)^{2\ell}}{(2\ell)!} - \sum_{\ell\ge 0} \frac{(\alpha t)^{2\ell +1}}{(2\ell+1)!} \\
        &= \cosh(\alpha t) - \sinh(\alpha t) \\
        &= \qp(t) - \qm(t).
    \end{aligned}
\end{equation}

\subsection{Proof of \Cref{prop:cv_mmd_radial_kernels}} \label{proof:prop_cv_mmd_radial_kernels}

The assumption $\underline\lambda[\qp],\underline\lambda[\qm]\ge 0$ allows %
to deduce that $\psip:z\mapsto \qp(\|z\|_2^2)$ and $\psim:z\mapsto \qm(\|z\|_2^2)$ are $\underline\lambda[\qp]\ge0$ and $\underline\lambda[\qm]\ge 0$ strongly convex. Hence, we can apply \Cref{prop:DC_radial_MMD} and obtain that $\cFp$ is $\underline\lambda[\qm]$-totally convex, and $\cFm$ is $\underline\lambda[\qp]$-totally convex.

Then applying \Cref{pro:LipschitzProjection}, we obtain the second result.

\subsection{Proof of \Cref{theorem:cv_dca_bpp_v3}} \label{proof:theorem_cv_dca_bpp}

We recall that for all $k\ge 0$,
\begin{equation}
    \T_{k+1} = \argmin_{\T\in L^2(\mu_k)}\ \bF(\T_\#\mu_k) + \D_{\cFm}^{\mu_k}(\T,\id).
\end{equation}
Taking the first order conditions, we obtain
\begin{multline}\label{eq:first-order_cnd}
    \gW\bF(\mu_{k+1})\circ \T_{k+1} + \gW\cFm(\mu_{k+1})\circ\T_{k+1} - \gW\cFm(\mu_k) = 0\\
    \iff \gW\bF(\mu_{k+1})\circ \T_{k+1} = -\big(\gW\cFm(\mu_{k+1})\circ\T_{k+1} - \gW\cFm(\mu_k)\big).
\end{multline}

Notice that $\T_{k+1}=\argmin_{\T_\#\mu_k=\mu_{k+1}}\ \D_{\cFm}^{\mu_k}(\T,\id)$, and define $\T_k^* = \argmin_{\T_\#\mu_k=\mu^*}\ \D_{\cFm}^{\mu_k}(\T,\id)$. Both exist as by assumption, $\mu_k\in\cPa$ for all $k\ge 0$.

By hypothesis, $\bF$ is $\alpha$-convex relative to $\cFm$ along $t\mapsto \big((1-t)\T_k^*+t\T_{k+1}\big)_\#\mu_k$, thus we have
\begin{multline}
    \D_{\bF}^{\mu_k}(\T_k^*, \T_{k+1}) \ge \alpha \D_{\cFm}^{\mu_k}(\T_k^*, \T_{k+1}) \\ \iff \bF(\mu^*) - \bF(\mu_{k+1}) - \langle \gW\bF(\mu_{k+1})\circ \T_{k+1}, \T_k^* - \T_{k+1}\rangle_{L^2(\mu_k)} \ge \alpha \D_{\cFm}^{\mu_k}(\T_k^*, \T_{k+1}).
\end{multline}
By definition of $\mu^*$, $\bF(\mu^*)-\bF(\mu_{k+1}) \le 0$. Using the first order conditions, we get the inequality
\begin{equation} \label{ineq:612}
    \langle \gW\cFm(\mu_{k+1})\circ \T_{k+1}-\gW\cFm(\mu_k), \T_k^* - \T_{k+1}\rangle_{L^2(\mu_k)} \ge \alpha \D_{\cFm}^{\mu_k}(\T_k^*, \T_{k+1}).
\end{equation}
By the 3-point equality (see \citep[Lemma 28]{bonet2024mirror}) applied with $\T:=\T_{k+1}$, $\sS:=\T_k^*$ and $\mathrm{U}:=\id$,
\begin{multline} \label{eq:app_three_pt_ineq}
    \langle \gW\cFm(\mu_{k+1})\circ \T_{k+1}-\gW\cFm(\mu_k), \T_k^* - \T_{k+1}\rangle_{L^2(\mu_k)} \\ = \D_{\cFm}^{\mu_k}(\T_k^*,\id) - \D_{\cFm}^{\mu_k}(\T_k^*,\T_{k+1}) - \D_{\cFm}^{\mu_k}(\T_{k+1}, \id).
\end{multline}
Plugging \eqref{eq:app_three_pt_ineq} into \eqref{ineq:612}, we get
\begin{equation}
    \begin{aligned}
        \D_{\cFm}^{\mu_k}(\T_k^*,\id) - \D_{\cFm}^{\mu_k}(\T_{k+1},\id) \ge (\alpha + 1) \D_{\cFm}^{\mu_k}(\T_k^*, \T_{k+1}).
    \end{aligned}
\end{equation}
Using that $\cFm$ is convex along $t\mapsto \big((1-t)\id + t \T_{k+1})_\#\mu_k$, we get that $\D_{\cFm}^{\mu_k}(\T_{k+1},\id) \ge 0$. Using the definition of $\T_k^*$, we have $\D_{\cFm}^{\mu_k}(\T_k^*,\id)=\W_{\cFm}(\mu^*,\mu_k)$ and since $\gamma=(\T_k^*, \T_{k+1})_\#\mu_k \in\Pi(\mu^*, \mu_{k+1})$, we also have $\W_{\cFm}(\mu^*,\mu_{k+1}) \le \D_{\cFm}^{\mu_k}(\T_k^*,\T_{k+1})$. Thus, we obtain by induction
\begin{equation}
    \W_{\cFm}(\mu^*,\mu_{k+1}) \le \D_{\cFm}^{\mu_k}(\T_k^*, \T_{k+1}) \le \frac{1}{1+\alpha}\W_{\cFm}(\mu^*,\mu_k) \le \left(\frac{1}{1+\alpha}\right)^{k+1} \W_{\cFm}(\mu^*,\mu_0).
\end{equation}

Moreover, by definition of $\T_{k+1}$,
\begin{equation}
    \bF(\mu_{k+1}) + \D_{\cFm}^{\mu_k}(\T_{k+1}, \id) \le \bF(\mu^*) + \D_{\cFm}^{\mu_k}(\T_k^*, \id).
\end{equation}
Hence,
\begin{equation}
    \bF(\mu_{k+1}) - \bF(\mu^*) \le \W_{\cFm}(\mu^*,\mu_k) \le \left(\frac{1}{1+\alpha}\right)^k \W_{\cFm}(\mu^*,\mu_0).
\end{equation}

\subsection{Proof of \Cref{lem:radialhessian}} \label{proof:lem_radialhessian}

Let $s=\|z\|_2^2$ and recall that $\psi(z)=q(\|z\|_2^2)$. Differentiating, we have $\nabla \psi(z)= 2q'(s)z$, hence differentiating a second time we obtain
\begin{equation}
    \nabla^2 \psi(z)= 2 q'(s)I_d + 4q''(s) zz^{\top}.
\end{equation}
Fix $z$ and let $e\perp z$ and $\|e\|_2=1$ then $zz^{\top}e=0$, and 
\begin{equation}
    \nabla^2\psi(z)e= 2q'(s)e,
\end{equation}
then $e$ is a tangential eigenvector, with eigenvalue $2q'(s)$. We can find $d-1$ vectors that are linearly independent and $\perp z$, hence $2q'(s)$ is an eigenvalue with multiplicity $d-1$. Now consider in the radial direction $v = \frac{z}{\|z\|_2}$ (when $z\neq 0$), we have: 
\begin{equation}
    \begin{aligned}
        \nabla^2\psi(z) v &=  \left(2 q'(s)I_d + 4q''(s) zz^{\top}\right) v\\
        &= 2 q'(s)v + 4q''(s) zz^{\top} \frac{z}{\|z\|_2} \\
        & = 2 q'(s)v + 4 q''(s) \|z\|_2 z\\
        & =2 q'(s)v + 4q''(s)  \|z\|_2^2 v\\
        & = \big(2q'(s) + 4q''(s)s\big)v.
    \end{aligned}
\end{equation}
Hence $v$ is an eigenvector of $\nabla^2\psi(z)$ with eigenvalue $\left(2q'(s) + 4q''(s)s\right)$ of multiplicity 1. 
Hence we have identified the full spectrum of $\nabla^2\psi(z)$, and the smallest eigenvalue is
\begin{equation}
    \min\bigl\{2q'(s),\,2q'(s)+4sq''(s)\bigr\},
\end{equation}
and the largest eigenvalue is
\begin{equation}
    \max\bigl\{2q'(s),\,2q'(s)+4sq''(s)\bigr\}.
\end{equation}

\subsection{Proof of \Cref{lem:DCGaussian}} \label{proof:lem_DCGaussian}

Recall that $k(x,y) = \psi(x-y)$, where $\psi(x-y) = \exp(-\alpha \|x-y\|_2^2)$. Let $q(s) = \exp(-\alpha s)$,  we have $q(s)=\qp(s)-\qm(s)$ where 
\begin{equation}
    \qp(s) = \cosh(\alpha s), \quad \qm(s)=\sinh(\alpha s).
\end{equation}
Taking the derivatives, we have
\begin{equation}
    \qp'(s)= \alpha \sinh(\alpha s), \quad \qp''(s)= \alpha^2 \cosh(\alpha s),
\end{equation}
and 
\begin{equation}
    \qm'(s)= \alpha \cosh(\alpha s), \quad \qm''(s)= \alpha^2 \sinh(\alpha s).
\end{equation}
It is easy to see that $q'_{\pm}(s)\geq 0 $ and $q''_{\pm}(s)\geq 0 $ for $s\geq 0$. Therefore $z\mapsto \psi_{\pm}(z)= q_{\pm}(\|z\|_2^2)$ are convex on $\Omega - \Omega$ by \Cref{lem:sufficient}, and we have a DC decomposition of the Gaussian kernel.

On $[0,S_*]$ we have: 
\begin{equation}
    \lambda_+ = \inf_{s \in [0,S_*]} \min \big(2\alpha \sinh(\alpha s), 2\alpha \sinh(\alpha s)+ 4\alpha^2 s \cosh(\alpha s) \big)= \inf_{s\in [0,S_*]}  2\alpha \sinh(\alpha s) = 0,
\end{equation}
and 
\begin{equation}
    \begin{aligned}
         \Lambda_+ &= \sup_{s\in [0,S_*]} \max\left( 2\alpha \sinh(\alpha s), 2\alpha \sinh(\alpha s)+ 4\alpha^2 s \cosh(\alpha s) \right)\\
         &= \sup_{s\in [0,S_*]} 2\alpha \sinh(\alpha s)+ 4\alpha^2 s \cosh(\alpha s) \\
         &= 2\alpha \sinh(\alpha S_*) + 4\alpha^2\cosh(\alpha S_*).
    \end{aligned}
\end{equation}
Similarly we have:
\begin{equation}
    \lambda_{-}= 2\alpha,  \quad \Lambda_{-} = 2\alpha \cosh(\alpha S_*) + 4\alpha^2\sinh(\alpha S_*).
\end{equation}

\subsection{Proof of \Cref{prop:DCGaussian_Jordan}} \label{proof:prop_Dcgaussian_Jordan}

Recall that $k(x,y)=\psi(x-y)$, where $\psi(x-y)=e^{-\alpha \|x-y\|_2^2}$. Let $q(s) = e^{-\alpha s}$, then $q(s)=\qp(s)-\qm(s)$ where for all $s\ge 0$,
\begin{equation}
    \qp(s) = e^{-\alpha s} + \alpha s,\quad \qm(s)=\alpha s.
\end{equation}
Taking the derivatives, we get
\begin{equation}
    \qp'(s) = \alpha(1- e^{-\alpha s}),\quad \qp''(s) = \alpha^2 e^{-\alpha s},
\end{equation}
and
\begin{equation}
    \qm'(s) = \alpha,\quad \qm''(s) = 0.
\end{equation}
Hence, for $s\ge 0$, $q_{\pm}'(s) \ge 0$ and thus $z\mapsto \psi_{\pm}(\|z\|_2^2)$ are convex by \Cref{lem:sufficient}.

Moreover, we get the following minimal and maximal eigenvalues,
\begin{equation}
    \lambda_+ = \inf_{s\in [0,+\infty)}\ \min \big\{ 2\alpha (1-e^{-\alpha s}), 2\alpha (1-e^{-\alpha s}) + 4s\alpha^2 e^{-\alpha s}\big\} = \inf_{s\in [0,+\infty)}\ 2\alpha (1-e^{-\alpha s}) = 0,
\end{equation}
and
\begin{equation}
    \begin{aligned}
        \Lambda_+ &= \sup_{s\in [0,+\infty)} \max \big\{ 2\alpha (1-e^{-\alpha s}), 2\alpha (1-e^{-\alpha s}) + 4s\alpha^2 e^{-\alpha s}\big\} \\
        &= \sup_{s\in [0,+\infty)}\ 2\alpha (1-e^{-\alpha s}) + 4s\alpha^2 e^{-\alpha s} = 2\alpha + 2\alpha(2s\alpha - 1) e^{-\alpha s}.
    \end{aligned}
\end{equation}
Let $f(s)=2\alpha + 2\alpha(2s\alpha - 1) e^{-\alpha s}$, its derivative gives $f'(s)=4\alpha^2 e^{-\alpha s} - 2\alpha^2(2s\alpha -1) e^{-\alpha s} = 2\alpha^2 e^{-\alpha s} (3-2s\alpha) = 0 \iff s= \frac{3}{2\alpha}$. Moreover, $f''(s)=-2\alpha^3 e^{-\alpha s}(3-2s\alpha) -4\alpha^3 e^{-s\alpha} = -2\alpha^3 e^{-s\alpha} (5-2s\alpha) \le 0 \iff s\le \frac{5}{2\alpha}$. Hence $s=\frac{3}{2\alpha}$ is the maximizer, and
\begin{equation}
    \begin{aligned}
        \Lambda_+ &= 2\alpha (1+2 e^{-\frac32}).
    \end{aligned}
\end{equation}
Similarly, $\lambda_-=\Lambda_-=2\alpha$.

\subsection{Proof of \Cref{lem:DCRiesz}} \label{proof:lem_DCRiesz}

The derivatives of $\qm$ give, for all $s\in \R$,
\begin{equation}
    \qm'(s) = \frac{1}{2 \sqrt{\varepsilon + s}}, \quad \qm''(s)= - \frac{1}{4(\varepsilon+s)^{3/2}}.
\end{equation}
For the minimum eigenvalue we have
\begin{equation}
    \begin{aligned}
        \min\bigl\{2\qm'(s),\,2\qm'(s)+4s\qm''(s)\bigr\} &= 2\qm'(s)+4s\qm''(s) \\
        &= \frac{1}{\sqrt{\varepsilon+ s}} -  \frac{s}{(s+\varepsilon)^{\frac{3}{2}}} \\
        & = \frac{\varepsilon}{(s+ \varepsilon)^{3/2}}.
    \end{aligned}
\end{equation}
Therefore for $s\in [0,S_*]$, we have
\begin{equation}
    \lambda_- =\frac{\varepsilon}{(S_*+ \varepsilon)^{3/2}}.
\end{equation}
For the maximum eigenvalue we have
\begin{equation}
    \max \bigl\{2\qm'(s),\,2\qm'(s)+4s\qm''(s)\bigr\} = 2\qm'(s) = \frac{1}{\sqrt{\varepsilon+s}} .
\end{equation}
Therefore for $s\in [0,S_*]$, we have
\begin{equation}
    \Lambda_{-} = \frac{1}{\sqrt{\varepsilon}}.
\end{equation}

\subsection{Proof of \Cref{lemma:DC_rational}} \label{proof:lemma_DC_rational}

$\psi(z)=\frac{1}{(c^2 + \|z\|_2^2)^{\alpha}},$ this corresponds to $q(s)=\frac{1}{(c^2 + s)^{\alpha}}, \alpha \geq 1.$
We have for $s\geq 0$,
\begin{equation}
    q'(s)= -\alpha \left(c^2 + s \right)^{-\alpha -1}, \quad q''(s)= \alpha(\alpha+1)\left(c^2 +s\right)^{-\alpha -2}.
\end{equation}
Hence we have $q'(s)<0$ and $q''(s)>0$, and thus $q(s)$ is convex on $s \geq 0$, but $\psi$ is not convex since $q'(s)<0$. Set $A= \max\big(0, -q'(0)\big)=\alpha c^{-2(\alpha+1)},$ hence
\begin{equation}
    \qm(s)= A s - \int_{0}^s(s-t) \min\big(0,q''(t)\big) \ \mathrm{d}t = \alpha c^{-2(\alpha+1)} s,
\end{equation}
and 
\begin{equation}
    \qp(s) = q(s) + \qm(s) = \frac{1}{(c^2 + s)^{\alpha}} +\alpha c^{-2(\alpha+1)} s.
\end{equation}
Now turning to minimum of the Hessian we have
\begin{equation}
    \underline\lambda[\qp] \coloneqq \inf_{s\geq 0}\min\bigl\{2\qp'(s),\,2\qp'(s)+4s\qp''(s)\bigr\}, \quad  \underline\lambda[\qm] \coloneqq \inf_{s\geq 0}\min\bigl\{2\qm'(s),\,2\qm'(s)+4s\qm''(s)\bigr\}.
\end{equation}
By construction $\qp'(s)\geq 0$ and $\qp''(s)\geq 0$, and hence $\qp'$ is non decreasing and $\underline\lambda[\qp] = \inf_{s\geq 0} 2\qp'(s) = 2\qp'(0)=-A+A=0.$ 
On the other hand $\qm'(s)= \alpha c^{-2(\alpha+1)}$ and $\qm''(s)=0$ and hence $\underline\lambda[\qm] =2\alpha c^{-2(\alpha+1)}.$ 
For the maximum of the Hessian we have as well,
\begin{equation}
    \overline\Lambda[\qp] =\sup_{s \geq 0}\max\bigl\{2\qp'(s),\,2\qp'(s)+4s\qp''(s)\bigr\} = \sup_{s\geq 0 }\ 2\qp'(s)+4s\qp''(s).
\end{equation}
Let us note for $s\geq 0$,
\begin{equation}
    f(s)= 2 \qp'(s) + 4 s \qp''(s), \quad f'(s)= 6 \qp''(s) + 4s  \qp'''(s).
\end{equation}
We have
\begin{equation}
    \begin{aligned}
        f'(s) &= 6\alpha(\alpha+1)(c^2+s)^{-\alpha-2}  - 4 \alpha(\alpha+1)(\alpha+2) s(c^2+s)^{-\alpha-3}  \\
        &=\alpha(\alpha+1)(c^2+s)^{-\alpha-2} \left( 6 - 4 (\alpha+2)\frac{s}{c^2+s} \right). 
    \end{aligned}
\end{equation}
Let $s^*$ such that
\begin{equation}
    6- 4 (\alpha+2)\frac{s^*}{c^2+s^*}=0,
\end{equation}
equivalently
\begin{equation}
    s^* = \frac{6c^2}{4(\alpha+2)-6} \geq 0.
\end{equation}
We have $f'(s)\geq 0$ for $s\in [0,s^*]$ and $f'(s)\leq 0$ for $s\in[s^*,+\infty[$, hence $f(s^*)$  is the global sup.  The result for $\overline \Lambda [\qm]$ is immediate.

\subsection{Proof of \Cref{pro:LipschitzProjection}} \label{proof:pro_LipschitzProjection}

Recall that based on the DC decomposition of \Cref{prop:DC_radial_MMD}, $\cFp = \cWp + \cVm + c$ where    
\begin{equation}
    \cWp(\mu) = \frac12 \iint \psip(x-y)\ \mathrm{d}\mu(x)\mathrm{d}\mu(y),\quad \cVm(\mu)=\int\Vm(x)\mathrm{d}\mu(x), 
\end{equation}
and
\begin{equation}
    \Vm(x)= \int \psim(x-y)\ \mathrm{d}\nu(y).        
\end{equation}
As $\cWp$ is an interaction energy, and $\cVm$ a potential energy, their Wasserstein gradients at $\mu\in\cP_2(\Omega)$ read as, for all $x\in\R^d$,
\begin{equation}
    \begin{aligned}
        &\gW \cWp(\mu)(x) = \nabla(\psip * \mu)(x) \\
        &\gW \cVm(\mu)(x) = \nabla \Vm(x) = \nabla (\psim * \nu)(x).
    \end{aligned}
\end{equation}
Hence, the Wasserstein gradient of $\cFp$ at $\mu$ is, for all $x\in\R^d$,
\begin{equation}
    \begin{aligned}
        \gW \cFp(\mu)(x) &=\gW \cWp(\mu)(x)+  \gW \cVm(\mu)(x)\\
        &= \nabla(\psip * \mu)(x) + \nabla (\psim * \nu)(x).
    \end{aligned}
\end{equation}

Let $\T\in L^2(\mu)$ such that $\T_\#\mu=\sigma$. Define $a, b:\R^d\to\R^d$ as
\begin{equation}
    a(x) \coloneqq \nabla(\psip * \mu)(x) - \nabla(\psip * \sigma)\big(\T(x)\big), \quad b(x) \coloneqq \nabla (\psim * \nu)(x) - \nabla (\psim * \nu)\big(\T(x)\big).
\end{equation}
Then, we have the following relation between the Wasserstein gradient of $\cFp$ at $\mu$ and $\sigma$:
\begin{equation}
    \gW\cFp(\mu)(x) - \gW\cFp(\sigma)(\T(x)) = a(x )+ b(x).
\end{equation}

\textbf{Bounding the term $\|a\|_{L^2(\mu)}$.} Looking at the term $a$ we have: 
\begin{equation}
    \begin{aligned}
        a(x) &= \int \nabla \psip(x-y) \ \mathrm{d}\mu(y)- \int \nabla \psip\big(\T(x)-z\big)\ \mathrm{d}\sigma(z)\\
        &= \int \left(\nabla \psip(x-y) - \nabla \psip\big(\T(x)-\T(y)\big) \right)\ \mathrm{d}\mu(y),
    \end{aligned}
\end{equation}
since $\T_{\#}\mu = \sigma$. $\nabla \psi_{+}$ is $\Lambda_+$-Lipschitz by assumption on $\Omega -\Omega$ and hence we have: 
\begin{equation}
     \|a(x)\|_2 \leq \Lambda_{+} \int \|(x- y) -( \T(x)-\T(y))\|_2\ \mathrm{d}\mu(y).
\end{equation}
Now, Jensen inequality and rearranging terms we obtain:
\begin{equation}
    \|a(x)\|_2^2 \leq \Lambda^2_{+} \int \|(x- \T(x)) -( y-\T(y))\|_2^2\ \mathrm{d}\mu(y)   
\end{equation}
Integrating on $x$ \emph{w.r.t} $\mu$,
\begin{equation} \label{eq:terma}
    \int \|a(x)\|_2^2\ \mathrm{d}\mu(x) \leq \Lambda^2_{+} \int \|(x- \T(x)) -( y-\T(y))\|_2^2\ \mathrm{d}\mu(x) \mathrm{d}\mu(y)  .
\end{equation}

Let us now provide an intermediary lemma before bounding $a$.
\begin{lemma} \label{lem:interlem}
    Let $H\in L^2(\mu)$, we have:
    \begin{equation}
        \int \| H(x)-H(y)\|_2^2\ \mathrm{d}\mu(x)\mathrm{d}\mu(y)= 2 \int \| H(x)\|_2^2\ \mathrm{d}\mu(x) - 2\left\| \int H(x)\  \mathrm{d}\mu(x) \right\|_2^2    
    \end{equation}
\end{lemma}

\begin{proof}[Proof of Lemma \ref{lem:interlem}] 
    Note that:
    \begin{equation}
        \begin{aligned}
            \int \| H(x)-H(y)\|_2^2\ \mathrm{d}\mu(x)\mathrm{d}\mu(y) &= \int (\| H(x)\|^2 + \| H(y)\|^2 - 2 \langle H(x), H(y)\rangle)\ \mathrm{d}\mu(x)\mathrm{d}\mu(y) \\
            &= 2 \int \| H(x)\|_2^2\ \mathrm{d}\mu(x) - 2\left\| \int H(x)\ \mathrm{d}\mu(x) \right\|_2^2.  
        \end{aligned}
    \end{equation}
\end{proof}
Using Lemma \ref{lem:interlem} in Equation \eqref{eq:terma} for $H(x)= x-\T(x)$ we obtain: 
\begin{equation}
    \begin{aligned}
         \int \|a(x)\|_2^2\ \mathrm{d}\mu(x) &\leq 2\Lambda^2_{+} \int \|x-\T(x) \|_2^2\ \mathrm{d}\mu(x) - 2 \Lambda^2_{+} \left\| \int (x-\T(x))\ \mathrm{d}\mu(x)\right\|_2^2\\
        &\leq 2\Lambda_+^2 \|\T-\id\|^2_{L^2(\mu)}.
    \end{aligned}
\end{equation}
Therefore we have: 
\begin{equation} \label{eq:Boundnorma}
    \|a\|_{L^2(\mu)} \leq \sqrt{2} \Lambda_{+} \|\T-\id\|_{L^2(\mu)}.
\end{equation}

\textbf{Bounding the term $\|b\|_{L^2(\mu)}$.} Note that $\nabla \Vm$ is $\Lambda_-$-Lipschitz since $\nabla^2 \psim \preceq \Lambda_{-}I_{d}$ and hence for all $x\in \Omega$,
\begin{equation}
    \nabla^2 \Vm(x) = \int \nabla^2 \psim(x-y)\ \mathrm{d}\nu(y) \preceq \Lambda_{-} I_{d}.
\end{equation}
Hence 
\begin{equation}
    \|b(x)\|_2  = \| \nabla \Vm(x) -\nabla \Vm(\T(x) )\|_2\leq \Lambda_{-} \|x-\T(x)\|_2.
\end{equation}
Integrating on $x$ and by Jensen inequality, we have finally
 \begin{equation} \label{eq:termb}
    \|b\|_{L^2(\mu)} \leq \Lambda_- \|\T-\id\|_{L^2(\mu)}.
\end{equation}

\textbf{Final Bound.} Combining equations \eqref{eq:terma} and \eqref{eq:termb} we obtain:
\begin{equation}
    \begin{aligned}
     \| \gW\cFp(\mu) - \gW\cFp(\sigma)\circ \T \|_{L^2(\mu)} &= \|a(x )+ b(x)\|_{L^2(\mu)} \\
     &\leq \|a\|_{L^2(\mu)} + \|b\|_{L^2(\mu)} \\
    &\leq \left( \sqrt{2}\Lambda_+ + \Lambda_- \right)\|\T-\id\|_{L^2(\mu)}.
    \end{aligned}
 \end{equation}

 \subsection{Proof of \Cref{th:cv_stationary}} \label{proof:th_cv_stationary}

  In Wasserstein CCCP we have $\gW\cFp(\mu_{k+1})\circ\T_{k+1} = \gW\cFm(\mu_k)$ and hence
 \begin{equation}
     \begin{aligned}
        \gW\cF(\mu_k) &= \gW\cFp(\mu_k)-\gW\cFm(\mu_k)\\
        &=\gW\cFp(\mu_k) - \gW\cFp(\mu_{k+1})\circ\T_{k+1}.
    \end{aligned}
 \end{equation}
 Hence we have
 \begin{equation}
    \begin{aligned}
         \| \gW\cF(\mu_k)\|_{L^2(\mu_{k})}& = \|\gW\cFp(\mu_k) - \gW\cFp(\mu_{k+1})\circ\T_{k+1} \|_{L^2(\mu_{k})} \\
         &\leq L \|\T_{k+1}-\id\|_{L^2(\mu_k)}, 
         \label{eq:lipch}
     \end{aligned}
 \end{equation}
 where we used \Cref{pro:LipschitzProjection} and $L=\sqrt{2}\Lambda_+ + \Lambda_-$. Since $\lambda_{\pm}\ge 0$ and $\lambda_+ + \lambda_- > 0$, we can apply \Cref{prop:cv_convexity} with $\alpha_{\pm}=\lambda_{\pm}$. Taking minimum on both sides in inequality \eqref{eq:lipch}, and applying \Cref{prop:cv_convexity} we have finally
\begin{equation}
    \begin{aligned}
        \min_{0\le k\le K-1}\ \| \gW\cF(\mu_k)\|^2_{L^2(\mu_{k})} &\leq L^2 \min_{0\le k\le K-1}\ \|\T_{k+1}-\id\|^2_{L^2(\mu_k)} %
        \\ &
        \le \frac{2 L^2}{\ap+\am} \frac{\big(\bF(\mu_0)-\bF(\mu_K)\big)}{K}.
    \end{aligned}
\end{equation}
In particular, $\min_{0\le k\le K-1}\|\gW\bF(\mu_k)\|_{L^2(\mu_k)}\to 0$ as $K\to\infty$.

\subsection{Proof of \Cref{prop:convcompact}} \label{proof:prop_conv_compact}

Let $(\mu_{k})_k$ be the sequence produced by WCCCP.
The  compactness of $\Omega$ implies the compactness of $\mathcal{P}_{2}(\Omega)$, hence we can extract a subsequence $\mu_{k_{j}} \to \mu^*$ in $\W_2$. 

Using \Cref{th:cv_stationary} (satisfied under the first point in Assumption \ref{Assump:localIninit}), $\| \gW\mathcal{F}(\mu_{k_j})\|_{L^2(\mu_{k_j})} \to 0$. Indeed, by \eqref{eq:69} and \eqref{eq:71} which are obtained while proving \Cref{th:cv_stationary}, we have the following inequalities
\begin{equation}
    \begin{cases}
        \frac{\lambda_+ + \lambda_-}{2}\|\T_{k+1}-\id\|_{L^2(\mu_k)}^2 \le \cF(\mu_k)-\cF(\mu_{k+1}), \\
        \|\gW\cF(\mu_k)\|_{L^2(\mu_k)} \le L\|\T_{k+1}-\id\|_{L^2(\mu_k)}.
    \end{cases}
\end{equation}
Then, since $\cF$ is bounded below, telescoping gives
\begin{equation}
    \frac{\lambda_+ + \lambda_-}{2}\sum_{k=0}^{K-1}\|\T_{k+1}-\id\|_{L^2(\mu_k)}^2 \le \cF(\mu_0) - \cF(\mu_K) \le \cF(\mu_0) - \inf\cF.
\end{equation}
Letting $K\to \infty$ yields 
\begin{equation}
    \sum_{k=0}^{\infty}\|\T_{k+1}-\id\|_{L^2(\mu_k)}^2 \le \frac{2(\cF(\mu_0)-\inf\cF)}{\lambda_+ + \lambda_-} < \infty.
\end{equation}
Therefore, $\|\T_{k+1}-\id\|_{L^2(\mu_k)} \to 0$ and $\|\gW\cF(\mu_k)\|_{L^2(\mu_k)}\to 0$ by \eqref{eq:71}. Thus, the norm of the gradient vanishes along the full sequence, and in particular along the subsequence.

By lower semicontinuity of $\mu \mapsto \| \gW\mathcal{F}(\mu)\|_{L^2(\mu)}$, we have $\| \gW\mathcal{F}(\mu^*)\|_{L^2(\mu^*)}= 0$ and $\mu^*$ is then a critical point. 
Since under our assumptions WCCCP guarantees descent of values and under Assumption  \ref{Assump:localIninit} we have $\cF(\mu_0)<c_*$, we obtain: \[\mathcal{F}(\mu^*) \leq \mathcal{F}(\mu_0)< c_*.\] 
By the definition of $c_*$, the limit $\mu^*$ must be equal to $\nu$ ($\mu^*=\nu$).  We finally conclude that every subsequence $\mu_{k_{j}}$ converges to $\nu$, and hence $\mu_{k} \to \nu$.

\subsection{Proof of \Cref{prop:convnoncompact}} \label{proof:prop_convnoncompact}

Local quadratic growth and the initialization assumption give $\W_2(\mu_0,\nu)<r$. To prove that the iterates remain in the ball, define the majorizing surrogate
\[
Q_k(\T):=\cFp(\T_\#\mu_k)-\cFm(\mu_k)
-\langle\gW\cFm(\mu_k),\T-\id\rangle_{L^2(\mu_k)}.
\]
By \eqref{eq:upper_bound_dc}, total convexity of $\cFp$, and optimality of $\T_{k+1}$, the interpolations $\T_{k,t}:=(1-t)\id+t\T_{k+1}$ and $\mu_{k,t}:=(\T_{k,t})_\#\mu_k$ satisfy
\[
\cF(\mu_{k,t})\leq Q_k(\T_{k,t})
\leq (1-t)Q_k(\id)+tQ_k(\T_{k+1})
\leq\cF(\mu_k),\qquad t\in[0,1].
\]
The path is continuous in $\W_2$, since
\[
\W_2(\mu_{k,t},\mu_{k,u})\leq |t-u|\,\|\T_{k+1}-\id\|_{L^2(\mu_k)}.
\]
If $\W_2(\mu_k,\nu)<r$ and $\cF(\mu_k)\leq E_0$, this path cannot reach the boundary of $B_r(\nu)$, where quadratic growth would give $\cF\geq cr^2>E_0$. Induction therefore yields $\W_2(\mu_k,\nu)<r$ and $\cF(\mu_k)\leq E_0$ for all $k$, so the local PL inequality applies at every iterate.

The descent and gradient estimates established in the proof of \Cref{prop:cv_convexity}, together with \eqref{eq:local_wasserstein_pl_mmd}, give
\[
\begin{aligned}
\cF(\mu_k)-\cF(\mu_{k+1})
&\geq \frac{\lambda}{2}\|\T_{k+1}-\id\|_{L^2(\mu_k)}^2\\
&\geq \frac{\lambda}{2L^2}\|\gW\cF(\mu_k)\|_{L^2(\mu_k)}^2\\
&\geq \frac{\lambda m}{L^2}\cF(\mu_k).
\end{aligned}
\]
By \eqref{eq:iterates_difference_gap} and $\inf\cF=\cF(\nu)=0$, this is the DC-PL estimate \eqref{eq:wdc_pl} along the iterates with $\eta_1=\lambda m/L^2$ and $\eta_2=0$. The iteration argument in \Cref{prop:wasserstein_DC_PL} thus yields $\cF(\mu_k)\leq q^k E_0$. Applying local quadratic growth gives the stated bound on $\W_2^2(\mu_k,\nu)$ and convergence to $\nu$.



\end{document}